\documentclass[twoside,11pt]{article}
\usepackage{jmlr2e}

\usepackage{mathtools}
\usepackage{float}
\usepackage[most]{tcolorbox}
\usepackage[margin=2.5cm]{geometry}

\DeclareMathOperator{\Law}{Law}
\DeclareMathOperator{\Tr}{Tr}
\DeclareMathOperator{\Hess}{Hess}

\DeclareMathOperator*{\arginf}{arg\,inf}
\DeclareMathOperator{\supp}{supp}
\DeclareMathOperator{\Ran}{Ran}

\newtheorem{assumption}{Assumption}

\usepackage{lastpage}
\jmlrheading{24}{2023}{1-\pageref{LastPage}}{6/20; Revised
8/22}{1/23}{20-602}{Andrew Duncan, Nikolas N{\"u}sken and Lukasz Szpruch}
\ShortHeadings{On the geometry of Stein variational gradient descent}{Duncan, N{\"u}sken and Szpruch}

\begin{document}

\title{On the geometry of Stein variational gradient descent}

\author{\name A. Duncan \email a.duncan@imperial.ac.uk\\ \addr Imperial College London, Department of Mathematics\\ London SW7 2AZ and The Alan Turing Institute, UK
\AND 
\name N. N\"usken\thanks{corresponding author} \email nikolas.nusken@kcl.ac.uk \\  \addr Kings's College London, Department of Mathematics \\ 
London WC2R 2LS, UK
\AND 
\name L. Szpruch \email l.szpruch@ed.ac.uk \\ \addr University of Edinburgh \\ School of Mathematics \\ Edinburgh, EH9 3JZ, UK and The Alan Turing Institute, London, UK
}

\editor{}

\editor{Kenji Fukumizu}
\maketitle

\begin{abstract}
Bayesian inference problems require sampling or approximating high-dimensional probability distributions. The focus of this paper is on the recently introduced Stein variational gradient descent methodology, a class of algorithms that rely on iterated steepest descent steps with respect to a reproducing kernel Hilbert space norm. This construction leads to interacting particle systems, the mean-field limit of which is a gradient flow on the space of probability distributions equipped with a certain geometrical structure. We leverage this viewpoint to shed some light on the convergence properties of the algorithm, in particular addressing the problem of choosing a suitable positive definite kernel function. Our analysis leads us to considering certain nondifferentiable kernels with adjusted tails. We demonstrate significant performance gains of these in various numerical experiments.  
\end{abstract}

\par\bigskip


\section{Introduction}
\label{sec:introduction}

Sampling and Variational Inference (VI) are the most common paradigms for extracting information from posterior distributions arising from Bayesian inference problems.   This is a particularly challenging problem in high dimensions, where the posterior distribution will only be known up to a constant of normalisation.    Markov Chain Monte Carlo (MCMC) methods based on the Metropolis-Hastings algorithm provide a generic approach to sampling from such distributions.  However, in high dimensions these methods suffer from poor scalability due to correlation between successive samples.    
Variational techniques reformulate inference as an optimisation problem; seeking a distribution from a family of simple probability distributions which best approximates the target posterior distribution.  VI typically permits faster inference, albeit at the cost of losing asymptotic exactness. 
\\\\
Recently there has been interest in \emph{particle optimisation techniques} which combine aspects of both approaches.   Here, an ensemble of particles are collectively evolved forward, seeking to approximate the posterior distribution.    One such approach, known as Stein Variational Gradient Descent (SVGD), was introduced in \citet{liu2016stein}.  In this method, an ensemble of $N$ particles in $\mathbb{R}^d$ defining an empirical measure $\rho^N$ is moved forward in a series of discrete steps via the map
$$
	x \mapsto  T(x) = x + \varepsilon \psi(x),
$$
where $\varepsilon$ is the step size and $\psi$ is a vector field, which is chosen such that the pushforward measure $T_{\sharp}\rho^N$ has minimal KL divergence with respect to the target posterior $\pi \propto \exp(-V)$.   Choosing $\psi$ from within the unit ball of a vector valued RKHS $\mathcal{H}_k^d$ with positive definite kernel $k:\mathbb{R}^d\times \mathbb{R}^d \rightarrow \mathbb{R}$ results in discrete dynamics of the form
$$
	X^i_{n+1} = X^i_{n} - \frac{\varepsilon}{N} \left( \sum_{j=1}^N \nabla k(X^i_{n}, X^j_{n})  +  \sum_{j=1}^N k(X^i_{n},X^j_{n}) \nabla V(X^j_{n})\right), 
$$
where $\nabla k$ denotes the gradient with respect to the first variable. 
In the continuous time limit, as $\varepsilon \rightarrow 0$, this results in the  following system of ordinary differential equations (ODEs) describing the evolution of the particles $X^1, \ldots, X^N$,
\begin{equation} 
\label{eq:Stein ode}
\frac{\mathrm{d}X_t^i}{\mathrm{d}t} = - \frac{1}{N} \sum_{j=1}^N \nabla k(X_t^i, X_t^j)  - \frac{1}{N} \sum_{j=1}^N k(X_t^i,X_t^j) \nabla V(X_t^j), \quad i=1,\ldots,N.
\end{equation}

It was observed in \citet{liu2017stein} that the scaling limit of \eqref{eq:Stein ode} as $N\rightarrow \infty$ is given by the mean-field equation
\begin{equation}
\label{eq:Stein pde}
\partial_t \rho_t(x) = \nabla\cdot\left( \rho_t(x) \int_{\mathbb{R}^d}   k(x,y) \left[ \nabla \rho_t(y) + \rho_t(y) \nabla V(y) \right] \mathrm{d}y \right),
\end{equation}
where $\rho$ denotes the limiting density of the particles as $N$ tends to infinity. The convergence of $\rho^N$ to $\rho$ was  proved rigorously in \citet{lu2019scaling} together with existence and uniqueness for \eqref{eq:Stein pde}, as well as convergence  to equilibrium, albeit without quantitative rates.  In \citet{liu2017stein} it was observed that the evolution equation \eqref{eq:Stein pde} can be viewed as a gradient flow on the space of probability densities, equipped with a certain distance that depends on the kernel $k$. Remarkably, this observation places SVGD in direct correspondence with the more conventional (overdamped) Langevin dynamics \citep{pavliotis2014stochastic}, see Appendix \ref{app:Langevin}. Our main focus in this paper is to follow the thread of this parallel and leverage the gradient flow perspective for the study of contraction and equilibration properties of \eqref{eq:Stein pde}. To wit, we develop a second order calculus and study the convexity properties of the $\mathrm{KL}$-divergence with respect to an appropriately constructed geometry on the space of probability densities, henceforth called \emph{Stein geometry}, and identify conditions in the form of functional inequalities which are necessary for exponential convergence of $\rho_t$ to the equilibrium $\pi$.   Building on this analysis, we are able to derive principled guidelines for making a suitable choice of the kernel function $k$. In particular, we explore analytically and numerically the use of singular kernel functions, i.e. those that are not continuously differentiable. In our experiments we demonstrate significant performance gains in a variety of inference tasks.  

\subsection{Previous work}
\label{sec:previous}

The SVGD method has attracted a lot of interest since it was  introduced in \citet{liu2016stein}.  Indeed, numerous variants have been proposed which improve scalability by exploiting additional information such as the conditional dependency structure \citep{zhuo2017message} or the underlying geometry of the posterior \citep{chen2019projected,detommaso2018stein,liu2018riemannian,wang2019stein}.  Stochastic variants which introduce noise into the dynamics in order to aid exploration and efficiency of SVGD have also been proposed \citep{gallego2018stochastic,li2019stochastic,zhang2018stochastic,zhang2018towards}. 	Other methods in the spirit of particle optimisation have been proposed, such as \citet{ambrogioni2018wasserstein,bigoni2019greedy,chen2018stein,liu2019understanding,mroueh2017sobolev,mroueh2019sobolev}. The potential of SVGD has also been explored in the context of sequentially updated Bayesian posteriors \citep{detommaso2019stein,pulido2018kernel}.
\\\\
Gradient flows provide a natural formalism in which to analyse the long-term behaviour of certain classes of nonlinear, nonlocal partial differential equations with dissipative behaviour.  This includes many PDEs arising as the mean-field equations of ensembles of interacting stochastic particle systems.  
\\\\
The space of densities equipped with the  quadratic Wasserstein metric formally defines a Riemannian structure over which gradient flows can be defined.  It is well known that solutions to the Fokker-Plank equation associated with the overdamped Langevin dynamics can be formulated as gradient flows of the $\mathrm{KL}$-divergence (or relative entropy) with respect to the Wasserstein metric.  Analysis of the geodesic convexity of the $\mathrm{KL}$-divergence yields conditions under which exponential convergence to equilibrium can be established.   This differential-geometric perspective was put forward by F. Otto and coworkers (see for example \citet{jordan1998variational,otto2001geometry,otto2005eulerian} or \citet[Chapter 15]{V2009} and \citet[Chapter 9]{V2003}). Of particular importance for the development in Section \ref{sec:2nd order} is the discussion in \citet[Section 3]{otto2000generalization}.   Extensions to systems of overdamped Langevin particles with various forms of interactions and their relationships to ensemble Kalman filters and inverse problems \citep{iglesias2013ensemble} have also been considered \citep{garbuno2019affine,garbuno2019gradient,nusken2019note}, see also the extension to $\gamma$-drift diffusions studied in \cite{li2019diffusion}.  In \citet{lu2019accelerating}, the Langevin dynamics are augmented with interactions giving rise to a nonlocal birth-death term in the mean-field equations.   By reformulating the system as a gradient flow of the $\mathrm{KL}$-divergence with respect to the Wasserstein-Fisher-Rao
metric, sufficient conditions for expontential convergence to equilibrium are obtained with quantitative rates. The dynamics put forward in \citet{pathiraja2019discrete,reich2019fokker} are based on approximations of the particle-density within a suitably chosen RKHS; this approach should be contrasted with SVGD which relies on a driving vector field with minimal RKHS-norm. We would also like to refer the reader to \citet{wang2020information}, where Newton gradient flows have been developed, holding the promise of accelerating convergence in the face of ill-conditioning.
\\\\
In the context of machine learning a number of recent works have proposed gradient flow formulations of methods for sampling and variational inference, see for example \citet{arbel2019maximum,li2018natural,lu2019accelerating,wang2019function,yuan2019deep,li2020ricci}.   In particular, a number of approaches which unify Langevin dynamics and SVGD via the common framework of Wasserstein gradient flows have also appeared \citep{chen2017particle,chen2018unified}.

\subsection{Our contribution}
The contributions in this paper are:
\begin{itemize}
	\item Following \citet{liu2017stein} we formulate the mean-field limit of SVGD as a gradient flow of the $\mathrm{KL}$-divergence in the so-called \emph{Stein geometry}. We define appropriate tangent spaces and study foundational properties of the structure thus obtained.   
	\item We derive expressions for the geodesics in this geometry and based on these, explore second order properties of the gradient flow dynamics. The latter are intimately related to a qualitative and quantitative understanding of the convergence to equilibrium, as has been widely recognised in the literature on Wasserstein gradient flows (see \citet{V2009} and references therein). By way of counterexample, we show that, within this framework and using only entropy as the driving force, it is in general impossible to obtain bounds on the Stein-Hessian operator that would allow us to conclude exponential convergence as in the Wasserstein case. 
	\item Moreover, we study the curvature of the $\mathrm{KL}$-divergence  around equilibrium, and identify conditions in the form of functional inequalities which are equivalent to exponential decay when near equilibrium. In certain scenarios we show that there is a direct correspondence with functional-analytic properties of the reproducing kernel Hilbert space (RKHS) associated to the kernel function $k$.
	\item Based on this we derive a series of guidelines for making a suitable choice of kernel function $k$, especially placing emphasis on regularity and tail properties.
\end{itemize}
We would like to point out that differential-geometric tools at this point mainly serve for intuition, and that a rigorous formulation in the framework of metric length spaces has been carried out in \citet{ambrosio2008gradient} for the Wasserstein case. Adapting those techniques to the Stein geometry is an interesting direction for future work. 
\\\\
The remainder of the paper will be as follows. In Section \ref{sec:formulation} we shall introduce basic notation and a number of preliminary assumptions.  In Section \ref{sec:noisy} we discuss a stochastic variant of the SVGD dynamics (originally proposed in \citet{gallego2018stochastic}) and show that the resulting mean-field PDE coincides with \eqref{eq:Stein pde}. In Section \ref{sec:gradient flow} we
recall and extend the Stein geometry introduced in \citet{liu2017stein}, in particular characterising the solution of the mean-field equation \eqref{eq:Stein pde} as a gradient flow of the $\mathrm{KL}$-divergence with respect to this geometry.  In Section \ref{sec:2nd order} we study the geodesic equations under the Stein metric and investigate the geodesic convexity of the $\mathrm{KL}$-divergence. In Section \ref{sec:curvature} we focus on the long-time behaviour when close to equilibrium, and in particular identify conditions in the form of functional inequalities for exponential return. In Section \ref{sec:polynomial} we give a brief outlook at applications of the developed theory for polynomial kernels.  In Section \ref{sec:numerics} a number of numerical experiments are presented to confirm and complement the theory.   Comments and conclusions are deferred to Section \ref{sec:conclusions}. In Appendix \ref{app:Langevin} we draw parallels between SVGD and the Stein geometry on the one hand, and Langevin dynamics and the Wasserstein geometry on the other hand.

\section{Assumptions and Preliminaries}
\label{sec:formulation}

\subsection{Notation and preliminaries}
\label{sec:preliminaries}
We first briefly define the function spaces which will be used throughout this paper. The space $C_c^\infty(\mathbb{R}^d)$ consists of smooth functions with compact support, and $\mathcal{D}'(\mathbb{R}^d)$ refers to its topological dual, the space of distributions.
Given a probability measure $\rho$ on $\mathbb{R}^d$ we define $L^2(\rho)$ to be the Hilbert space of square-integrable functions with respect to $\rho$ with inner product $\langle \phi, \psi \rangle_{L^2(\rho)} = \int_{\mathbb{R}^d} \phi \psi \, \mathrm{d}\rho$.  The subspace $L_0^2(\rho)$ consists of centered functions in $L^2(\rho)$, that is,
\begin{equation}
\label{eq:L20}
L_0^2(\rho) = \left\{ \phi \in L^2(\rho) : \quad \int_{\mathbb{R}^d} \phi \, \mathrm{d}\rho =  0 \right\}.
\end{equation} 
We define the (weighted) Sobolev space $H^1(\rho)$ to be the subspace of $L^2(\rho)$ functions having derivatives also in $L^2(\rho)$, i.e.
\begin{equation}
\nonumber
H^1(\rho) = \left\{\phi \in L^2(\rho): \quad \Vert \nabla \phi \Vert_{L^2(\rho)} < \infty \right\}.
\end{equation}
The following assumption on $k$ is fundamental:
\begin{assumption}[Assumptions on $k$]\label{ass:pdef} The kernel $k:\mathbb{R}^d \times \mathbb{R}^d \rightarrow \mathbb{R}$ is continuous, symmetric and positive definite, i.e. 
\begin{equation}
\nonumber
\sum_{i,j = 1}^n \alpha_i \alpha_j k(x_i, x_j) \ge 0,
\end{equation}	
	for all $n \in \mathbb{N}$, $\alpha_1, \ldots \alpha_n \in \mathbb{R}$ and $x_1, \ldots, x_n \in \mathbb{R}^d$. 
\end{assumption}

Canonical examples of kernels satisfying Assumption \ref{ass:pdef} include the Gaussian kernel $k(x,y) = \exp\left(-\frac{\vert x - y \vert^2}{\sigma^2}\right)$, and  Laplace kernel $k(x,y) = \exp \left( -\frac{\vert x- y \vert }{\sigma}\right)$. More generally, we will consider the kernels $k_{p,\sigma}:\mathbb{R}^d \times \mathbb{R}^d \rightarrow \mathbb{R}$, defined via
\begin{equation}
	\label{eq:p kernel}
	k_{p,\sigma}(x,y) = \exp\left(-\frac{\vert x - y \vert^p}{\sigma^p}\right),
\end{equation}
where  $p \in (0,2]$ is a smoothness parameter, and $\sigma > 0$ is called the kernel width.

 Let $(\mathcal{H}_k, \langle \cdot, \cdot \rangle_{\mathcal{H}_k})$ be the reproducing kernel Hilbert space (RKHS) associated to the kernel $k$, \citep[Sec 4.2]{steinwart2008support}, that is, $\mathcal{H}_k$ is the Hilbert space of all functions on $\mathbb{R}^d$ such that, for $x \in \mathbb{R}^d$, $k(x, \cdot) \in \mathcal{H}_k$ and $f(x) = \langle f, k(x, \cdot)\rangle_{\mathcal{H}_k}$.  We let $\Vert \cdot \Vert_{\mathcal{H}_k}$ be the norm induced by the inner product on $\mathcal{H}_k$. The $d$-fold Cartesian product 
\begin{equation}
\label{eq:vector rkhs}
\mathcal{H}^d_k = \underbrace{\mathcal{H}_k \times \ldots \times \mathcal{H}_k}_{d \, \text{times}}
\end{equation}
is a Hilbert space of vector fields $v = (v_1, \ldots, v_d) :\mathbb{R}^d \rightarrow \mathbb{R}^d$, equipped with the norm
\begin{equation}
\nonumber
\Vert v \Vert_{\mathcal{H}_k^d}^2 = \sum_{i=1}^d \Vert v_i \Vert^2_{\mathcal{H}_k}.
\end{equation}
\begin{remark}[Vector-valued RKHS] More generally one can consider matrix-valued kernels of the form $\bar{k}:\mathbb{R}^d \times \mathbb{R}^d \rightarrow \mathbb{R}^{d \times d}$, \citep{carmeli2006vector,micchelli2005learning}, as has recently been done in \citet{wang2019stein}. The associated RKHS $\mathcal{H}_{\bar{k}}$ then consists of vector-valued functions. We leave the analysis of SVGD algorithms based on matrix-valued kernels for future work.	
\end{remark}
The following is a nondegeneracy assumption on $k$, instrumental in guaranteeing convergence of solutions to \eqref{eq:Stein pde} towards the target $\pi$.
\begin{assumption}
	\label{ass:ispd}\citep{fukumizu2009kernel,sriperumbudur2010hilbert}
	The kernel $k$ is \emph{integrally strictly positive definite} (ISPD), i.e. 
	\begin{equation}
 \nonumber
	\int_{\mathbb{R}^d} \int_{\mathbb{R}^d} k(x,y) \,\mathrm{d}\rho(x) \mathrm{d}\rho(y) > 0 
	\end{equation}
	holds for all finite nonzero signed Borel measures $\rho$. 
\end{assumption}
From \citet[Theorem 7]{sriperumbudur2010hilbert}, ISPD kernels are characteristic, i.e. the kernel mean embedding $\rho \mapsto \int k(\cdot, y) \, \mathrm{d}\rho(y)$ is injective.  
We note that the kernels defined in \eqref{eq:p kernel} (in particular, the Gauss and Laplace kernels) are ISDP, see Lemma \ref{lem:exp p kernel} below.
\\\\
Throughout this article, we will denote by $\mathcal{P}(\mathbb{R}^d)$ the space of probability measures on $\mathbb{R}^d$. Abusing the notation, we will use the same letter for their Lebesgue densities in case they exist.
Given a kernel $k$, we define the following subset of  $\mathcal{P}(\mathbb{R}^d)$,
\begin{subequations}
\begin{align*}
\mathcal{P}_k(\mathbb{R}^d) = \Bigg\{ \rho \in \mathcal{P}(\mathbb{R}^d) :\quad & \rho \,\,\text{admits a smooth Lebesgue density,} \quad \supp \rho = \mathbb{R}^d, 
\\
&\quad \int_{\mathbb{R}^d} k(x,x) \, \mathrm{d}\rho(x) < \infty \Bigg\},
\end{align*}
\end{subequations}
and, for $\rho \in \mathcal{P}_k(\mathbb{R}^d)$, the linear operator $\mathcal{T}_{k, \rho}:L^2(\rho)\rightarrow \mathcal{H}_k$ via
\begin{equation}
\label{eq:T k rho}
\mathcal{T}_{k,\rho} \phi = \int_{\mathbb{R}^d} k(\cdot,y) \phi(y) \, \mathrm{d}\rho(y), \quad \phi \in L^2(\rho).
\end{equation}
For $\rho \in \mathcal{P}_k(\mathbb{R}^d)$, $\mathcal{T}_{k,\rho}$ is compact, self-adjoint and positive semi-definite. Furthermore,  by \citet[Theorem 4.26]{steinwart2008support} the associated RKHS $\mathcal{H}_k$ will consist of $L^2(\rho)$-functions.   By Assumption \ref{ass:ispd} and the fact that $\supp \rho = \mathbb{R}^d$, $\mathcal{T}_{k,\rho}$ is injective, and consequently, the embedding $\mathcal{H}_k \subset L^2(\rho)$ is dense. For a normed vector space $V$ (such as $L^2(\rho)$, $H^1(\rho)$ or $\mathcal{H}_k$ above) and a subset $A\subset V$, we denote by $\overline{A}^{V} \subset V$ the closure in the corresponding norm. That is, $\overline{A}^V$ is the smallest set containing $A$ that is closed with respect to $\Vert \cdot \Vert_V$.
\\\\
Finally, our objective will be to generate samples from the target density $\pi \propto e^{-V}$ on $\mathbb{R}^d$.  We shall make the following basic assumptions on $\pi$ and $V$:

\begin{assumption}
\label{ass:V and pi}
The potential $V: \mathbb{R}^d \rightarrow \mathbb{R}$ is continuously differentiable, with $e^{-V} \in L^1(\mathbb{R}^d)$. The target density is given by
\begin{equation}
\label{eq:target}
\pi = \frac{1}{Z} e^{-V},
\end{equation}	
where $Z = \int_{\mathbb{R}^d} e^{-V} \mathrm{d}x$ is the normalising constant. Furthermore, $\pi \in \mathcal{P}_k(\mathbb{R}^d)$.
\end{assumption}

\section{Stochastic SVGD and its Mean Field Limit}
\label{sec:noisy}
Before turning our focus towards the main topic of this paper in Section \ref{sec:gradient flow}, we comment on a stochastic variant of \eqref{eq:Stein ode}, providing another link to the overdamped Langevin dynamics. This section can be skipped (or read independently from the rest of the paper). The follow-up work \citet{nusken2021stein} connects the deterministic dynamics \eqref{eq:Stein ode} to its stochastic augmentation \eqref{eq:noisy Stein} discussed below using the theory of large deviations and the geometric framework developed in this paper. 

In \citet{gallego2018stochastic}, the following modification of \eqref{eq:Stein ode} was introduced,
\begin{equation}
\label{eq:noisy Stein}
\mathrm{d}\bar{X}_t = \left(-\mathcal{K}(\bar{X}_t) \nabla \bar{V}(\bar{X}_t)  + \nabla \cdot \mathcal{K}(\bar{X}_t)\right) \mathrm{d}t + \sqrt{2 \mathcal{K}(\bar{X}_t)}\, \mathrm{d}W_t,
\end{equation}
where $\bar{X} = (X^1,\ldots,X^N) \in \mathbb{R}^{Nd}$ comprises the collection of particles, $(W_t)_{t \ge 0}$ denotes an $Nd$-dimensional standard Brownian motion,
\begin{equation}
\nonumber
\bar{V}(x_1,\ldots,x_N) = \sum_{i=1}^N V(x_i)
\end{equation}
is the extended potential, 
 and the state-dependent \emph{mass matrix} $\mathcal{K}:\mathbb{R}^{Nd} \rightarrow \mathbb{R}^{Nd \times Nd}$ can be decomposed into $N^2$ blocks of size $d\times d$ as follows,
\begin{equation}
\nonumber
\mathcal{K}(\bar{x}) = 
\begin{pmatrix}
{K}_{11}(\bar{x}) & \dots & K_{1N}(\bar{x}) \\
\vdots & \ddots & \vdots \\
K_{N1}(\bar{x}) & \ldots & K_{NN}(\bar{x}) 
\end{pmatrix},
\end{equation}
where 
\begin{equation}
\nonumber
K_{ij} (\bar{x}) = \frac{1}{N} k(x_i,x_j) I_{d\times d}.
\end{equation}
Furthermore, $\sqrt{\mathcal{K}(\bar{x})}$ denotes a square root of the nonnegative matrix $\mathcal{K}(\bar{x})$.  By definition,
\begin{equation}
\nonumber
(\nabla \cdot \mathcal{K})_i = \sum_{j=1}^{Nd} \frac{\partial \mathcal{K}_{ij}}{\partial \bar{x}_j}, \qquad i = 1,\ldots,Nd,
\end{equation}
so we see that
the $i^{th}$ coordinate $X_t^i$ satisfies the SDE
\begin{equation}
\label{eq:noisy_stein_coordinate}
\begin{aligned}
\mathrm{d}X_t^i = \frac{1}{N}\sum_{j=1}^N \left[ - k(X_t^i, X_t^j)\nabla V(X_t^j) + \nabla_{X_t^j}k(X_t^i, X_t^j)\right] \mathrm{d}t + \sum_{j=1}^N \sqrt{2\mathcal{K}(\bar{X}_t)}_{ij}\, \mathrm{d}W_t^j, 
\end{aligned}
\end{equation}
coinciding with \eqref{eq:Stein ode} up to the noise term $\sqrt{2\mathcal{K}(\bar{X}_t)}\, \mathrm{d}W_t$. Indeed, this perturbation becomes vanishingly small in the limit as $N \rightarrow \infty$, and the mean-field limits of \eqref{eq:Stein ode} and \eqref{eq:noisy Stein} agree:\footnote{While a rigorous  convergence proof is beyond the scope of this work, we can formally identify the mean-field limit.}
\begin{proposition}[Formal identification of the mean-field limit]
	\label{prop:mean field limit}
	As $N\rightarrow \infty$, the empirical measure $\rho_t^N = \frac{1}{N}\sum_{i=1}^N \delta_{X_t^i}$ associated with \eqref{eq:noisy_stein_coordinate} converges to the solution $\rho_t$ of \eqref{eq:Stein pde}.
\end{proposition}
\begin{proof}
	See Appendix \ref{app:noisy Stein}.
\end{proof}
It is straightforward to check that 
\begin{equation}
\label{eq:marginals}
\bar{\pi}(x_1, \ldots, x_N) := \prod_{i=1}^N \pi(x_i) = \frac{1}{Z^N} \exp \left(- \sum_{i=1}^N V(x_i) \right)
\end{equation}
is an invariant probability density for \eqref{eq:noisy Stein}, with marginals\footnote{We use the notation  $\mathrm{d}x_1 \ldots \widehat{\mathrm{d}x_i} \ldots \mathrm{d}x_N$ to indicate that integration is meant to be performed over all variables except for $x_i$.}
\begin{equation}
\nonumber
\int_{\mathbb{R}^{(N-1)d}} \bar{\pi}(x_1,\ldots, x_N) \, \mathrm{d}x_1 \ldots \widehat{\mathrm{d}x_i} \ldots \mathrm{d}x_N = \pi(x_i). 
\end{equation}
Below, we will show that under mild conditions, the dynamics \eqref{eq:noisy Stein} is in fact ergodic with respect to $\bar{\pi}$, so that in particular
\begin{equation}
\label{eq:ergodicity}
\frac{1}{T} \int_0^T \frac{1}{N} \sum_{i=1}^N\phi(X_t^i) \, \mathrm{d}t \xrightarrow{T \rightarrow \infty } \int_{\mathbb{R}^{d}} \phi \, \mathrm{d}\pi, \quad \text{a.s.},
\end{equation} 
for any test function $\phi\in C_b(\mathbb{R}^{d})$.
Suitable discretisations of \eqref{eq:noisy Stein} therefore lead to MCMC-type algorithms on an extended state space in the framework of \citet{ma2015complete}, as already noticed in \citet{gallego2018stochastic}. See also \citet[Section 2.2]{DNP2017} and \citet{nusken2019constructing} for related discussions. 
\\\\
For our ergodicity result we need the following set of assumptions:
\begin{assumption}
	\label{ass:ergodicity}
	The following hold:
	\begin{enumerate}
		\item The SDE  \eqref{eq:noisy Stein} admits a global strong solution.
		\item
		\label{it:potential control}
		 We have $\mathbb{E}\int_0^t |\nabla V(X^i_s)|\,\mathrm{d}s < \infty$ for all $i = 1, \ldots, N$ and all $t > 0$.
		\item 
		\label{it:kernel assumption}
		The kernel $k$ is translation-invariant, i.e. 
		$$
		k(x,y) = h \left(x-y\right),\quad x, y \in \mathbb{R}^d,
		$$
		where $h \in C(\mathbb{R}^d) \cap C^1(\mathbb{R}^d\setminus \{0\})$ is Lipschitz continuous, and its gradient satisfies the one-sided Lipschitz condition
		\begin{equation}
		\label{eq:one-sided Lipschitz}
		\left(\nabla h(x) - \nabla h(y) \right) \cdot (x-y) \le C \vert x-y \vert^2,
		\end{equation}
		for some constant $C$ and all $x,y \neq 0$. 
	\end{enumerate}
\end{assumption}
\begin{proposition}[Ergodicity of stochastic SVGD]
\label{prop:ergodicity}
Let Assumption \ref{ass:ergodicity} be satisfied, $d \ge 2$, and assume that the initial condition for \eqref{eq:noisy_stein_coordinate} is distinct, i.e. $X_0^i \neq X_0^j$ for $i \neq j$. Then $X_t^i \neq X_t^j$ for $i \neq j$ for all $t > 0$, almost surely. Moreover, the process $(\bar{X}_t)_{t \ge 0}$ is ergodic with respect to the product measure \eqref{eq:marginals}.  
\end{proposition}
\begin{proof}
See Appendix \ref{app:noisy Stein}.
\end{proof}
\begin{remark}
Assumption \ref{ass:ergodicity}.\ref{it:potential control} holds under suitable (mild) conditions on the growth of $V$ at infinity.  
Any bounded translation-invariant kernel of regularity $C^2$ satisfies Assumption \ref{ass:ergodicity}, (\ref{it:kernel assumption}). Specifically, the kernels \eqref{eq:p kernel} satisfy Assumption \ref{ass:ergodicity}.\ref{it:kernel assumption} if $p \in [1,2]$. In the case when $p<1$ these kernels are not Lipschitz continuous. We leave an extension of Proposition \ref{prop:ergodicity} to this regime for future work. Note that the assumption of translation-invariance can easily be weakened, but we choose to impose it for ease of presentation.
\end{remark}


\section{SVGD as a gradient flow}
\label{sec:gradient flow}

In \citet{liu2017stein} it was observed that the evolution equation \eqref{eq:Stein pde} can be interpreted as gradient flow dynamics of the $\mathrm{KL}$-divergence on the space of probability measures equipped with a novel distance $d_k$ that depends on the chosen kernel. Formally, $d_k$ is furthermore the geodesic distance induced by a suitably chosen Riemannian metric. Here we review this perspective and identify the relevant tangent spaces, preparing the ground for our calculations in the later sections. Let us remark that in order to understand the results of the later sections Corollary \ref{cor:gradient flow} suffices; the remainder of this section may thus be skipped at first reading.

In what follows we set up a formal Riemannian calculus on $\mathcal{P}_k(\mathbb{R}^d)$, acting as though $\mathcal{P}_k(\mathbb{R}^d)$ was a smooth manifold. To reinforce this heuristic viewpoint, and for notational convenience, we will use the shorthand $M := \mathcal{P}_k(\mathbb{R}^d)$.   This perspective (nowadays known as \emph{Otto calculus}) has been put forward for the case of the quadratic Wasserstein distance in the seminal works \citet{jordan1998variational,otto1998dynamics,otto2001geometry,otto2000generalization,otto2005eulerian} and was further developed in \citet{ambrosio2008gradient,gigli2012second} and \citet{daneri2008eulerian}. For textbook accounts we refer to \citet[Chapter 8]{V2003}, \citet[Chapter 15]{V2009} and \citet[Chapter 3]{ambrosio2013user}.

To facilitate intuition, we begin with an informal discussion. Speaking in broad terms, many particle-based methods in general (see Section \ref{sec:previous}), and SVGD in particular, postulate dynamical schemes of the form
\begin{equation}
\label{eq:ODE}
\frac{\mathrm{d}X_t}{\mathrm{d}t} = v_t(X_t), \qquad X_0 \sim \rho_0.
\end{equation}
Those are based on a family of vector fields $v_t$, inducing a flow of probability measures $\rho_t = \Law X_t$. Under mild growth and regularity assumptions on $v_t$, the evolution of $\rho_t$ is governed by the \emph{continuity equation}
\begin{equation}
\label{eq:continuity}
\partial_t \rho_t + \nabla \cdot (\rho_t v_t)  = 0,    
\end{equation}
see, for instance, \citet[Section 4.1.2]{ambrosio2013user}. 
On the other hand, given a flow of probability measures $\rho_t$, we may reverse this logic and ask for a family of vector fields $v_t$ that reproduces  $\rho_t$, in the sense of 
\eqref{eq:continuity}, or, equivalently, \eqref{eq:ODE}. Notice that $v_t$ will not be unique, since for any sufficiently regular density $\rho_t$ there exist infinitely many vector fields $u_t$ that satisfy $\nabla \cdot (\rho_t u_t) = 0$; those $u_t$ can be added to any $v_t$ without affecting the validity of \eqref{eq:continuity}. To enforce uniqueness\footnote{Apart from uniqueness, the subsequent minimal norm requirement holds the promise of making numerical schemes associated to \eqref{eq:ODE} particularly stable by reducing the stiffness of the dynamics.}, it is reasonable to either select $v_t$ so as to minimise a certain norm or to constrain it to lie in a specified subspace (while at the same time satisfying \eqref{eq:continuity}). The following result shows that requiring $v_t$ to have minimal $\mathcal{H}_k^d$-norm is equivalent to $v_t \in \overline{\mathcal{T}_{k,\rho_t}\nabla C_c^\infty(\mathbb{R}^d)}^{\mathcal{H}_k^d}$, that is, up to taking limits in $\mathcal{H}_k^d$, $v_t$ is a gradient field, convolved using the operator $\mathcal{T}_{k,\rho}$ defined in \eqref{eq:T k rho}. In other words, the \emph{SVGD construction principle} originally put forward in \cite{liu2016stein} (namely to construct movement schemes that are minimal in $\mathcal{H}_k^d$-sense) implies that $v_t \in \overline{\mathcal{T}_{k,\rho_t}\nabla C_c^\infty(\mathbb{R}^d)}^{\mathcal{H}_k^d}$ for dynamics of the form \eqref{eq:ODE}.
\begin{proposition}[Selection principle]
\label{prop:selection}
Let the pair
$(\rho,v):(0,1) \rightarrow \mathcal{P}_k(\mathbb{R}^d) \times \mathcal{H}_k^d$  satisfy the continuity equation \eqref{eq:continuity}. Furthermore, assume that $v_t \in \overline{\mathcal{T}_{k,\rho_t} \nabla C_c^\infty(\mathbb{R}^d)}^{\mathcal{H}_k^d}$, for all $t \in (0,1)$. Then the following hold:
\begin{enumerate}
\item \emph{Given $\rho$, the vector field $v$ is the unique solution to \eqref{eq:continuity} in $\overline{\mathcal{T}_{k,\rho_t} \nabla C_c^\infty(\mathbb{R}^d)}^{\mathcal{H}_k^d}$:} If $(\rho,w) :(0,1) \rightarrow  \mathcal{P}_k(\mathbb{R}^d) \times \mathcal{H}_k^d$ also satisfies \eqref{eq:continuity} as well as $w_t \in \overline{\mathcal{T}_{k,\rho_t} \nabla C_c^\infty(\mathbb{R}^d)}^{\mathcal{H}_k^d}$ for all $t \in (0,1)$, then $v = w$.
    \item \emph{The vector field $v$ minimises the $\mathcal{H}_k^d$-norm among solutions to \eqref{eq:continuity}:} Let $w:(0,1) \rightarrow \mathcal{H}_k^d$ be any other vector field that together with $\rho$ satisfies \eqref{eq:continuity}. Then
    \begin{equation}
    \nonumber
    \Vert v_t \Vert_{\mathcal{H}_k^d} \le \Vert w_t \Vert_{\mathcal{H}_k^d}, \qquad \qquad \text{for all } t \in (0,1).
    \end{equation}
\end{enumerate}
\end{proposition}
The following proposition (proven in Appendix \ref{app:geometry proofs}) provides the basis for Proposition \ref{prop:selection}, as well as for many of the other constructions in this section. It should be compared to the usual $L^2(\rho)$-orthogonal decomposition of vectors fields into gradients and (weighted) divergence-free vector fields, see, for instance, \citet{figalli2021invitation}.  
\begin{proposition}[Helmholtz decomposition for RKHS]
\label{prop:helmholtz}
Let $\rho \in M$ and define the space of (weighted) divergence-free vector fields
\begin{equation}
\nonumber
L^2_{\mathrm{div}}(\rho) = \left\{ v \in (L^2(\rho))^d: \quad \langle v, \nabla \phi \rangle_{(L^2(\rho))^d} = 0, \quad \text{for all } \phi \in C_c^{\infty}(\mathbb{R}^d)\right\}.
\end{equation}
Then $\mathcal{H}_k^d$ admits the following $\langle \cdot, \cdot \rangle_{\mathcal{H}^d_k}$-orthogonal decomposition,
\begin{equation}
\nonumber
    \mathcal{H}_k^d = \left( L^2_{\mathrm{div}}(\rho) \cap \mathcal{H}_k^d \right) \oplus \overline{\mathcal{T}_{k,\rho} \nabla C_c^\infty(\mathbb{R}^d)}^{\mathcal{H}_k^d}.
\end{equation}
\end{proposition}
\begin{proof}\textbf{of Proposition \ref{prop:selection}}
For the first claim, notice that $\nabla \cdot(\rho_t (v_t - w_t))) = 0$, for all $t \in (0,1)$. Since also $v_t - w_t \in \overline{\mathcal{T}_{k,\rho_t}\nabla C_c^\infty(\mathbb{R}^d)}^{\mathcal{H}_k^d}$, the statement follows directly from the Helmholtz decomposition for $\mathcal{H}_k^d$ in Proposition \ref{prop:helmholtz}.
For the second claim, notice that we can decompose $w_t = v_t + u_t$, where $\nabla \cdot (\rho_t u_t) = 0$. From the orthogonality in \eqref{eq:Helmholtz decomp} it then follows that
\begin{equation}
\label{eq:Helmholtz decomp}
\Vert w_t \Vert^2_{\mathcal{H}_k^d} = 
\Vert v_t \Vert^2_{\mathcal{H}_k^d} + \underbrace{2 \langle v_t,u_t \rangle_{\mathcal{H}_k^d}}_{=0} + \Vert u_t \Vert^2_{\mathcal{H}_k^d} \ge \Vert v_t \Vert^2_{\mathcal{H}_k^d},      
\end{equation}
as required.
\end{proof}

After this intuitive introduction, we proceed by introducing a suitable notion of tangent spaces equipped with positive-definite quadratic forms, playing the role of Riemannian metrics. This construction is motivated by the special role played by the spaces $\overline{\mathcal{T}_{k,\rho} \nabla C_c^\infty(\mathbb{R}^d)}^{\mathcal{H}_k^d}$ according to Proposition \ref{prop:selection} and justified by Corollary \ref{cor:gradient flow} (see below). We follow \citet[Section 4.2]{mielke2014relation} in style of exposition. 

\begin{definition}[Tangent spaces and Riemannian metric] 
	\label{def:Riemann geometry}
	For $\rho \in M$, we define the \emph{tangent space}
\begin{subequations}
	\label{eq:tangent spaces}
	\begin{align}
	T_{\rho} M = \Bigg\{ \xi \in \mathcal{D}'(\mathbb{R}^d):\quad  &  \text{there exists } v \in \overline{\mathcal{T}_{k,\rho} \nabla C_c^\infty(\mathbb{R}^d)}^{\mathcal{H}_k^d} \,\, \text{such that}  \\ 
	& \xi + \nabla \cdot (\rho v) = 0 \quad \text{in the sense of distributions}
	\Bigg\}
	\end{align}
\end{subequations}
and the \emph{Riemannian metric} $g_{\rho}: T_\rho M \times T_\rho M \rightarrow \mathbb{R}$ by
\begin{equation} 
\label{eq:Stein metric}
g_\rho(\xi, \chi) = \langle u, v \rangle_{\mathcal{H}_k^d}, 
\end{equation}
where $\xi + \nabla \cdot(\rho u ) = 0$ and $\chi + \nabla \cdot(\rho v ) = 0$.
\end{definition}

\begin{remark}
As usual, we say that $\xi + \nabla \cdot (\rho v) = 0$ holds in the sense of distributions if 
\begin{equation}
\nonumber
\langle \xi, \phi \rangle  - \int_{\mathbb{R}^d}  \nabla \phi \cdot v \, \mathrm{d}\rho = 0,   
\end{equation}
for all $\phi \in C_c^\infty(\mathbb{R}^d)$, where $\langle \cdot, \cdot \rangle$ denotes the duality pairing between $\mathcal{D}'(\mathbb{R}^d)$ and $C_c^\infty(\mathbb{R}^d)$. Moreover, $\overline{\mathcal{T}_{k,\rho} \nabla C_c^\infty(\mathbb{R}^d)}^{\mathcal{H}_k^d}$ refers to the closure of the set $\mathcal{T}_{k,\rho} \nabla C_c^{\infty}(\mathbb{R}^d) = \left\{\mathcal{T}_{k,\rho} \nabla \phi: \,\, \phi \in C_c^{\infty} (\mathbb{R}^d) \right\}$ with respect to the norm $\Vert\cdot \Vert_{\mathcal{H}_k^d}$.
\end{remark}
We have the following result, in particular justifying the definition of $g_\rho$ in \eqref{eq:Stein metric}:
\begin{lemma}
	[Properties of $T_\rho M$ and $g_\rho$]  
	\label{lem:Riemann properties}
	For every $\rho \in M$, the following hold:
	\begin{enumerate}
		\item $(T_\rho M, g_\rho)$ is a Hilbert space. 
		\item For every $\xi \in T_\rho M$ there exists a unique $v \in \overline{\mathcal{T}_{k,\rho} \nabla C_c^\infty(\mathbb{R}^d)}^{\mathcal{H}_k^d}$ such that $\xi + \nabla \cdot (\rho v) = 0$ in the sense of distributions, in particular $g_\rho$ is well-defined. The map $v \mapsto -\nabla \cdot (\rho v)$ is a Hilbert space isomorphism between $(\overline{\mathcal{T}_{k,\rho} \nabla C_c^\infty(\mathbb{R}^d)}^{\mathcal{H}_k^d},\langle \cdot, \cdot \rangle_{\mathcal{H}_k^d})$ and $(T_\rho M, g_\rho)$.
	\end{enumerate}
\end{lemma}
\begin{proof}
	See Appendix \ref{app:geometry proofs}.
\end{proof}
\begin{remark}
The second statement of Lemma \ref{lem:Riemann properties} shows that the tangent spaces $(T_\rho M,g_\rho)$ could equivalently be defined as $(\overline{\mathcal{T}_{k,\rho} \nabla C_c^\infty(\mathbb{R}^d)}^{\mathcal{H}_k^d},\langle \cdot, \cdot \rangle_{\mathcal{H}_k^d})$. In the case of the quadratic Wasserstein distance this is the route taken in \citet[Section 1.4]{gigli2012second} and \citet[Section 2.3.2]{ambrosio2013user}. The space $(\overline{\mathcal{T}_{k,\rho} \nabla C_c^\infty(\mathbb{R}^d)}^{\mathcal{H}_k^d}$ has an appealing intuitive interpretation: It consists exactly of those vector fields that might arise from particle movement schemes when those are constrained by an RKHS-norm (see the intuitive introduction to this section), as proposed in the original paper \citet{liu2016stein}. We note in passing that our definition of the tangent spaces differs from the one put forward in \citet{liu2017stein} by the constraint $v \in \overline{\mathcal{T}_{k,\rho} \nabla C_c^\infty(\mathbb{R}^d)}^{\mathcal{H}_k^d}$. The latter is crucial for the isomorphic properties obtained in Lemma \ref{lem:Riemann properties} and for the calculations in Section \ref{sec:2nd order}.
\end{remark}
In preparation for the following lemma, let us recall that the $L^2(\mathbb{R}^d)$-functional derivative of a suitable functional $\mathcal{F}:M \rightarrow \mathbb{R}$ is defined via
\begin{equation}
\label{eq:functional derivative}
\int_{\mathbb{R}^d} \frac{\delta \mathcal{F}}{\delta{\rho}}(\rho) \phi \, \mathrm{d}x = \frac{\mathrm{d}}{\mathrm{d}\varepsilon} \Big\vert_{\varepsilon = 0} \mathcal{F}(\rho + \varepsilon \phi),
\end{equation}
for $\phi \in C_c^\infty(\mathbb{R}^d)$ with $\int_{\mathbb{R}^d} \phi \, \mathrm{d}x = 0$, see for instance \citet[Section 3.4.1]{peletier2014variational}. We remark that a more rigorous treatment can be given in terms of Fr{\'e}chet derivatives (see \citet[Section 5.4]{carmona2018probabilistic} for a related discussion). The heuristic Riemannian structure introduced in Definition \ref{def:Riemann geometry} induces a gradient operator which we can formally identify as follows:
\begin{lemma}[Stein gradient]
\label{lem:gradient}
Let $\rho \in M$ and $\mathcal{F}: M \rightarrow \mathbb{R}$ be such that the functional derivative $\frac{\delta \mathcal{F}}{\delta \rho}(\rho)$ is well-defined and continuously differentiable. Moreover assume that $\mathcal{T}_{k,\rho} \nabla \frac{\delta \mathcal{F}}{\delta \rho}(\rho) \in \overline{\mathcal{T}_{k,\rho} \nabla C_c^\infty(\mathbb{R}^d)}^{\mathcal{H}_k^d}$.  Then the Riemannian gradient associated to $(T_\rho M, g_\rho)$ is given by
\begin{equation}
\label{eq:gradient}
(\mathrm{grad}_k \mathcal{F})(\rho) = - \nabla \cdot \left( \rho \, \mathcal{T}_{k,\rho} \nabla \frac{\delta \mathcal{F}}{\delta \rho} (\rho) \right).
\end{equation}
\end{lemma}
\begin{proof}
	See Appendix \ref{app:geometry proofs}.
\end{proof}
\begin{remark}[Onsager operators]
	\label{rem:Onsager}
	The operators $\mathbb{K}_\rho: \phi  \mapsto - \nabla \cdot \left(\rho  \mathcal{T}_{k,\rho}\nabla \phi \right)$ should be thought of as mappings from the topological dual $T_\rho^* M$ into $T_\rho M$. As such, they correspond to the musical isomorphisms between tangent and cotangent bundles in Riemannian geometry \citep{lee2006riemannian}, or, in the language of physics, to the raising and lowering of indices. Following this analogy, the functional (Fr{\'e}chet) derivative $\frac{\delta \mathcal{F}}{\delta \rho}(\rho)$ lies in the space $T_\rho^* M$, at least formally. In the theory of gradient flows, the operators $\mathbb{K}_\rho$ are often referred to as \emph{Onsager operators} \citep{arnrich2012passing,liero2013gradient,machlup1953fluctuations,mielke2011gradient,mielke2013thermomechanical,mielke2016generalization,ottinger2005beyond}. 
\end{remark}
We recall the definition of the $\mathrm{KL}$-divergence with respect to the target measure $\pi$, 
\begin{equation}
\label{eq:KL}
\mathrm{KL}(\rho \vert \pi)=  \int_{\mathbb{R}^d} \log \left( \frac{\rho}{\pi}\right) \, \mathrm{d} \rho = \underbrace{\int_{\mathbb{R}^d} \rho \log \rho \, \mathrm{d}x}_{=:\mathrm{Reg}(\rho)} + \underbrace{\int_{\mathbb{R}^d} V \, \mathrm{d} \rho}_{=:\mathrm{Cost}(\rho \vert \pi)} + Z,
\end{equation}
noting the decomposition into a data term ${\mathrm{Cost}(\rho \vert \pi)}$ and an entropic regularisation $\mathrm{Reg}(\rho)$ that aids intuition in a statistical context \citep{csimcsekli2018sliced}. The following result forms the linchpin for the work subsequently presented in this paper (see also \citet[Theorem 3.5]{liu2017stein}).
\begin{corollary}
\label{cor:gradient flow}
The gradient flow dynamics of the $\mathrm{KL}$-divergence with respect to the Stein geometry is given by the Stein PDE \eqref{eq:Stein pde}.
\end{corollary}
\begin{proof}
	This follows from Lemma \ref{lem:gradient} together with 
	\begin{equation}
 \nonumber
	\frac{\delta \mathrm{KL}}{\delta \rho} (\rho) = \log \rho + 1 + V,
	\end{equation}
	which can be obtained by standard computations from \eqref{eq:functional derivative}, see for instance \citet[Chapter 15]{V2009}.
\end{proof}
The gradient flow perspective immediately implies the decay of the $\mathrm{KL}$-divergence along the flow. Our aim in Section \ref{sec:2nd order} will be to make the following statement more quantitative.
\begin{corollary}[Decay of the $\mathrm{KL}$-divergence]
	\label{cor:KL decay}
	For solutions $(\rho_t)_{t \ge 0}$ to the Stein PDE \eqref{eq:Stein pde} it holds that
	\begin{equation}
 \nonumber
	\frac{\mathrm{d}}{\mathrm{d}t} \mathrm{KL}(\rho_t \vert \pi) \le 0.
	\end{equation}
\end{corollary}
The Riemannian structure introduced in Definition \ref{def:Riemann geometry} formally induces a Riemannian distance \citep[Chapter 6]{lee2006riemannian} on $M$ as follows:
\begin{definition}[Stein distance]
	\label{def:distance}
	For $\mu,\nu \in M$ we define the \emph{Stein distance}
\begin{equation}
\label{eq:Stein constrained}
d_k^2(\mu,\nu) = \inf_{(\rho,v) \in \mathcal{A}(\mu,\nu)} \left\{ \int_0^1 \Vert v_t \Vert_{\mathcal{H}^d_k}^2 \, \mathrm{d}t, \quad v_t \in \overline{\mathcal{T}_{k,\rho_t}\nabla C_c^\infty(\mathbb{R}^d)}^{\mathcal{H}_k^d}\right\} ,
\end{equation}
where the \emph{set of connecting curves} is given by
\begin{subequations}
	\label{eq:rho v set Stein}
	\begin{align}
	\mathcal{A}(\mu,\nu) = \Bigg\{ & (\rho , v) : [0,1] \rightarrow \mathcal{P}_k(\mathbb{R}^d) \times \mathcal{H}_k^d, \quad \rho_0 = \mu, \rho_1=\nu, \nonumber\\
	\label{eq:weak continuity equation}
	& \partial_t \rho + \nabla \cdot (\rho v) = 0 \quad \text{in the sense of distributions}
	\Bigg\}.
	\nonumber
	\tag{\ref*{eq:rho v set Stein}}
	\end{align}
\end{subequations}
\end{definition}
\begin{remark}
	The distance $d_k$ is constructed in such a way that, formally,
	\begin{equation}
 \nonumber
	d_k^2(\mu,\nu) = \inf_{\rho} \left\{ \int_0^1 g_{\rho_t}( \partial_t \rho_t, \partial_t \rho_t) \, \mathrm{d}t : \quad \rho_0 = \mu, \,\rho_1 = \nu \right\},
	\end{equation}
	however sidestepping the issue of defining the appropriate notion of differentiation for $\partial_t \rho$.
\end{remark}

\begin{lemma}
	\label{lem:properties distance}
	The following hold:
	\begin{enumerate}
		\item 
	The Stein distance $d_k$ is an \emph{extended metric}\footnote{An extended metric satisfies the usual axioms (see the proof in Appendix \ref{app:geometry proofs}), but $d(\mu_1, \mu_2) = + \infty$ for some $\mu_1,\mu_2 \in M$ is possible.} on $M$.
  \item If $k$ is continuous and bounded, then there exists a constant $C>0$ such that
  \begin{equation}
  \nonumber
  \mathcal{W}_2(\mu,\nu) \le C d_k(\mu,\nu), \quad \mu,\nu \in M,
  \end{equation}
  denoting by $\mathcal{W}_2$ the quadratic Wasserstein distance.
  In particular, the topology induced by $d_k$ is stronger than the topology of weak convergence.
  \item
  \label{it:unconstrained}
  The constraint $\quad v_t \in \overline{\mathcal{T}_{k,\rho_t}\nabla C_c^\infty(\mathbb{R}^d)}^{\mathcal{H}_k^d}$ in \eqref{eq:Stein constrained} can be dropped, i.e. we have
	\begin{equation}
	\label{eq:Stein unconstrained}
	d_k^2(\mu,\nu) = \inf_{(\rho,v) \in \mathcal{A}(\mu,\nu)} \left\{ \int_0^1 \Vert v_t \Vert_{\mathcal{H}^d_k}^2 \, \mathrm{d}t \right\}.
	\end{equation}
	\end{enumerate}
\end{lemma}
\begin{proof}
	See Appendix \ref{app:geometry proofs}.
\end{proof}
\begin{remark}
	\label{rem:d_W}
	With Lemma \ref{lem:properties distance}.\ref{it:unconstrained} in conjuction with Corollary \ref{cor:gradient flow} we recover the main result from \citet{liu2017stein}. The additional constraint $\quad v_t \in \overline{\mathcal{T}_{k,\rho_t}\nabla C_c^\infty(\mathbb{R}^d)}^{\mathcal{H}_k^d}$ in \eqref{eq:Stein constrained} allows us to reduce the optimisation problem to a subset of $\mathcal{A}(\mu,\nu)$ and to place the analysis in a formal Riemannian framework, in particular allowing the calculations in Section \ref{sec:2nd order}. 
	
	It is instructive to note the similarity of \eqref{eq:Stein unconstrained} with the Benamou-Brenier formula for the quadratic Wasserstein distance $\mathcal{W}_2$, see  \citet{benamou2000computational}, \citet[Theorem 8.1]{V2003}, \citet[Theorem 5.53]{carmona2018probabilistic}, as well as Appendix \ref{app:Langevin}. In particular, $d_k$ can be obtained form $\mathcal{W}_2$ by merely adapting the notion of kinetic energy, i.e. by exchanging the $L^2(\rho)$-norm for the $\mathcal{H}^d_k$-norm. We would like to advertise the works \citet{buttazzo2009optimization,carrillo2010nonlinear,dolbeault2009new,li2019diffusion}
	for a rigorous analysis of similarly modified transport-based distances, as well as the overview article \citet{brasco2012survey} for an in-depth discussion. 
\end{remark}
\begin{remark}
[Kernels that depend on $\rho$]
Although the framework in this section has been set up for a fixed kernel $k$, it is straightforward to extend it to the case when $k$ varies with $\rho$, allowing for adaptive choices as the algorithm progresses. In particular, the gradient flow perspective is still valid. Indeed, it is sufficient to replace $k$ by $k(\rho)$ in the equations \eqref{eq:tangent spaces}, \eqref{eq:Stein metric}, \eqref{eq:gradient}, \eqref{eq:Stein constrained} and \eqref{eq:rho v set Stein}. Note, however, that in this case the results in the following Section \ref{sec:2nd order} would require nontrivial adaptations, in particular to Proposition \ref{prop:geodesic equations}. Those might be an interesting avenue for future research, and in this regard we would like to point the reader to \citet[Section 4]{li2021hessian} for a recently discovered connection between mean-field kernels and differential geometric structures induced by (positive-definite) Hessians.
\end{remark}



\section{Second order calculus for SVGD}
\label{sec:2nd order}
In this section, we study the constant-speed geodesics associated to the Riemannian geometry developed in the previous section. As is well-known, convexity properties of the $\mathrm{KL}$-divergence along those curves correspond to the contraction behaviour of the associated gradient flow (see Theorem \ref{thm:convexity contraction} below). 
Constant-speed geodesics $(\rho_t)_{0 \le t \le 1}$ are characterised by
\begin{equation}
\nonumber
d_k(\rho_s,\rho_t) = \vert t - s \vert d_k(\rho_0,\rho_1) \quad s,t \in [0,1],
\end{equation}
and can be obtained as critical points for the variational problem \eqref{eq:Stein constrained}, or, equivalently, \eqref{eq:Stein unconstrained}, allowing arbitrary starting and end points $\mu,\nu \in M$. As it turns out, constant-speed geodesics can formally be described by a coupled system of PDEs:
\begin{proposition}[Geodesic equations] 
	\label{prop:geodesic equations}
    Let $(\rho_t,v_t)_{0 \le t \le 1}$ be a critical point of \eqref{eq:Stein constrained}. Then
	\begin{subequations}
	\label{eq:Stein geodesics}
	\begin{align}
	\partial_t \rho + \nabla \cdot \left( \rho \mathcal{T}_{k,\rho}  \nabla \Psi  \right) & = 0, 
	\\
	\partial_t \Psi +  \nabla \Psi \cdot \mathcal{T}_{k,\rho} \nabla \Psi & = 0,
	\end{align}
\end{subequations}
for some function $\Psi: [0,1] \times \mathbb{R}^d \rightarrow \mathbb{R}$, and $v_t = \mathcal{T}_{k,\rho_t} \nabla \Psi_t$. 
\end{proposition} 
\begin{proof}\textbf{(Informal)} The proof (to be found in Appendix \ref{app:proofs geodesics}) relies on formal computations, inspired by the heuristics in \citet[Section 3]{otto2000generalization}. It proceeds by identifying \eqref{eq:Stein geodesics} as the formal optimality conditions for \eqref{eq:Stein constrained}; in particular, $\Psi$ acts as a Lagrange multiplier enforcing the constraints. A rigorous formulation (involving well-posedness of \eqref{eq:Stein geodesics}) is the subject of ongoing work. In the Wasserstein case, rigorous formulation of the associated geodesic equations have been given imposing additional regularity assumption, see \citet[Proposition 4]{lott2008some} or using the machinery of geodesic length spaces \citep[Proposition 3.10 and Remark 3.11]{gigli2012second}.
\end{proof}
In the sequel, we will refer to smooth solutions $(\rho_t, \Psi_t)_{0 \le t \le 1}$ of the system \eqref{eq:Stein geodesics} as \emph{Stein geodesics}.
\begin{remark}
	\label{rem:geodesics}
	It is interesting to compare \eqref{eq:Stein geodesics} to the geodesic equations for the quadratic Wasserstein distance $\mathcal{W}_2$, 
	\begin{subequations}
		\label{eq:Wasserstein geodesics}
		\begin{align}
		\label{eq:continuity geodesic}
		\partial_t \rho + \nabla \cdot (\rho \nabla \Psi) & = 0, \\
		\label{eq:HJB}
		\partial_t \Psi + \frac{1}{2} \vert \nabla \Psi \vert^2 & = 0,
		\end{align}
	\end{subequations}
see \citet{lott2008some}, \citet[Chapter 5]{V2003} and \citet{otto2000generalization}. In contrast to \eqref{eq:Stein geodesics}, the second equation \eqref{eq:HJB} decouples from the first one, \eqref{eq:continuity geodesic}. The fact that the distance $d_k$ induces a system of coupled equations for its geodesics can naturally be linked to the interpretation of \eqref{eq:Stein pde} as the mean-field limit of an interacting particle system. See also Appendix \ref{app:Langevin}.
\end{remark}

In what follows, our objective is to take some steps towards a more quantitative understanding of the $\mathrm{KL}$-decay in Corollary  \ref{cor:KL decay}. As is well-known, decay estimates can be obtained from convexity properties along geodesics. We refer to
\citet[Section 9.1]{villani2003optimal}, in particular to Formal Corollary 9.3, restated here as follows: 
\begin{theorem}[Informal]
	\label{thm:convexity contraction}
	Assume that there exists $\lambda > 0$ such that
	\begin{equation}
	\label{eq:KL 2der}
	\frac{\mathrm{d}^2}{\mathrm{d}t^2} {\rm KL} (\rho_t \vert \pi) \big\vert_{t = 0}  > \lambda,
	\end{equation}
	for all \emph{unit-speed geodesics} $(\rho_t)_{t \in (-\varepsilon, \varepsilon)}$. Then
	\begin{equation}
	\label{eq:KL exp decay}
	{\rm KL}(\rho_t \vert \pi) \le e^{-2\lambda t} {\rm KL}(\rho_0 \vert \pi).
	\end{equation}
	along solutions $(\rho_t)_{t \ge 0}$ of \eqref{eq:Stein pde}.
\end{theorem}
\begin{remark}[Beyond the $\mathrm{KL}$-divergence] Using Lemma \ref{lem:gradient}, it is possible to derive alternative dynamical schemes that seek to minimise arbitrary functionals of sufficient regularity. In Theorem \ref{thm:convexity contraction}, it would then be sufficient to replace the $\mathrm{KL}$-divergence by the functional of interest, and the calculations that follow in this section (in particular those leading to Lemma \ref{lem:Hessian}) could be carried out in a similar fashion. We would like to point the reader towards \citet{arbel2019maximum}, where the gradient flow of the maximum mean discrepancy in the Wasserstein geometry has been investigated using similar ideas.
\end{remark}
\begin{remark}
\label{rem:unit speed}
Unit-speed geodesics are solutions $(\rho_t,\Psi_t)_{t \in (-\varepsilon, \varepsilon)}$  to \eqref{eq:Stein geodesics}	satisfying $g_{\rho_t}(\partial_t \rho,\partial_t \rho) = 1$ for $t \in (-\varepsilon,\varepsilon)$. By the definition of $g_\rho$ (see \eqref{eq:Stein metric}) the latter statement is equivalent to
\begin{equation}
\nonumber
\left \langle  \mathcal{T}_{k,\rho_t} \nabla \Psi_t,\mathcal{T}_{k,\rho_t} \nabla \Psi_t  \right \rangle_{\mathcal{H}_k^d} = 1,
\end{equation}
and, by using \citet[Theorem 4.26]{steinwart2008support}, to
\begin{equation}
\label{eq:unit tangent}
\int_{\mathbb{R}^d} \int_{\mathbb{R}^d} \nabla \Psi_t(y) k(y,z) \nabla \Psi_t(z) \mathrm{d}\rho_t(y) \mathrm{d}\rho_t(z) = 1.
\end{equation}
\end{remark}
Motivated by Theorem \ref{thm:convexity contraction} we compute the left-hand side of \eqref{eq:KL 2der}:
\begin{lemma}[Computing the Hessian]
	\label{lem:Hessian}
	Let $(\rho_t,\Psi_t)_{t \in (-\varepsilon,\varepsilon)}$ be a Stein geodesic, i.e. a smooth solution to \eqref{eq:Stein geodesics}, and $\rho_0 \equiv \rho$, $\Psi_0 \equiv \Psi$.  Then
	\begin{equation}
	\frac{\mathrm{d}^2}{\mathrm{d}t^2} \mathrm{KL}(\rho_t \vert \pi) \big\vert_{t = 0}  = \Hess_\rho(\Psi, \Psi),
	\end{equation}
	where
	\begin{equation}
	\label{eq:Hess op}
	\Hess_\rho(\Phi,\Psi) = \sum_{i,j = 1}^d \int_{\mathbb{R}^d} \int_{\mathbb{R}^d} \partial_i \Phi(y) q_{ij} [\rho] (y,z) \partial_j \Psi(z) \, \mathrm{d}\rho(y) \mathrm{d}\rho(z),
	\end{equation}
	and
	\begin{subequations}
		\label{eq:Hessian}
		\begin{align}
		\label{eq:nonequilibrium1}
		q_{ij}[\rho](y,z)  = & \delta_{ij} \sum_{l=1}^d \int_{\mathbb{R}^d} \partial_{x_l} \left( e^{-V(x)} k(x,y) \right) e^{V(x)} \, \mathrm{d}\rho(x) \, \partial_{y_l} k(y,z)
		\\
		\label{eq:nonequilibrium2}
		& -  \int_{\mathbb{R}^d} \partial_{y_j} \partial_{x_i} \left(e^{-V(x)} k(x,y) \right) e^{V(x)} \, \mathrm{d}\rho(x) k(y,z)
		\\
		\label{eq:equilibrium}
		& - \int_{\mathbb{R}^d} \partial_{x_j} \left(e^{V(x)} \partial_{x_i} \left( e^{-V(x)} k(x,y) \right) \right) k(x,z) \, \mathrm{d}\rho(x).
		\end{align}
	\end{subequations}
\end{lemma}
\begin{proof}
	See Appendix \ref{app:proof Hessian}.
\end{proof}
\begin{remark}
	For notational convenience, our definition of $\mathrm{Hess}_\rho$ slightly differs from the definition of Hessian operators commonly encountered in the literature on Wasserstein gradient flows (see for instance \citet[Section 3.1]{otto2005eulerian}). 
\end{remark}
\begin{remark}
	\label{rem:rough k}
Although \eqref{eq:Hessian} is written in a form requiring suitable differentiability properties of $k$, we would like to emphasise that an examination of the proof shows that the result also holds for kernels that are merely continuous (provided that $\rho$ and $\Psi$ are smooth enough), either by interpreting \eqref{eq:Hessian} in a distributional way, or by performing integration by parts in   \eqref{eq:Hess op}.
\end{remark}
Combining Theorem \ref{thm:convexity contraction} with \eqref{eq:unit tangent} we obtain the following informal lemma, relating a functional inequality to exponential decay of the $\mathrm{KL}$-divergence: 
\begin{lemma}[Informal]
	\label{lem:functional inequality}
	Assume that there exists $\lambda > 0$ such that
	\begin{equation}
	\label{eq:hessian metric bound}
	\mathrm{Hess}_\rho(\Psi,\Psi) \ge \lambda \int_{\mathbb{R}^d} \int_{\mathbb{R}^d} \nabla \Psi(y) \cdot k(y,z) \nabla \Psi(z) \,\mathrm{d}\rho(y) \mathrm{d}\rho(z)
	\end{equation}
	for all $\rho \in M$ and $\Psi$ such that the right-hand side of \eqref{eq:hessian metric bound} is well-defined. Then 
	the exponential decay estimate \eqref{eq:KL exp decay} holds. 
\end{lemma}
\begin{remark}
	In more geometrical terms, \eqref{eq:hessian metric bound} can be written as
	\begin{equation}
 \nonumber
	\mathrm{Hess}_\rho(\Psi,\Psi) \ge \lambda g_\rho(v,v),
	\end{equation}
	with $v = \mathcal{T}_{k,\rho}\nabla \Psi$. 
\end{remark}

The Hessian can be split according to the decomposition of the $\mathrm{KL}$-divergence in \eqref{eq:KL},
\begin{equation}
\nonumber
\mathrm{Hess}_\rho(\Phi,\Psi) = \mathrm{Hess}^\mathrm{Reg}_\rho(\Phi,\Psi) + \mathrm{Hess}^\mathrm{Cost}_\rho(\Phi,\Psi),
\end{equation}  
for explicit expressions see Lemmas \ref{lem:Hess Reg} and \ref{lem:Hess cost} in Appendix \ref{app:proof Hessian}.
 Since the work of McCann \citep{mccann1997convexity}, it is well-known that $\mathrm{Reg}(\rho)$ is displacement-convex in the sense of Theorem \ref{thm:convexity contraction} along unit-speed Wasserstein geodesics. The analogous statement is false for the Stein geodesics considered in this paper:
\begin{lemma}
	\label{lem:hessian negative}
	Let $\Psi: \mathbb{R}^d \rightarrow \mathbb{R}$ be a linear function, i.e. $\Psi(x)= a \cdot x$ for some $a \in \mathbb{R}^d$, $a \neq 0$. Then $\mathrm{Hess}^{\mathrm{Reg}}_{\rho}( \Psi, \Psi) < 0$ for all $\rho \in M$ and all translation-invariant kernels $k$.
\end{lemma}
\begin{proof}
	See Appendix \ref{app:proofs geodesics}.
\end{proof}

Lemma \ref{lem:hessian negative} shows that the entropic term $\mathrm{Reg}(\rho)$ by itself is not sufficient to explain contraction properties of the Stein PDE \eqref{eq:Stein pde}, contrary to the case of the Fokker-Planck equation associated to overdamped Langevin dynamics (see also Appendix \ref{app:Langevin}). As a consequence, we have not been able to obtain bounds for the Stein-Hessian operator within this framework, which would have allowed us to obtain the analogue of a logarithmic Sobolev inequality.  More specifically, we expect that more stringent assumptions (in comparison to standard settings in the theory of the Fokker-Planck equation) would have to be imposed on $V$ in order to obtain functional inequalities of the form \eqref{eq:hessian metric bound}. A possible route towards Stein logarithmic Sobolev inequalities (under such more stringent assumptions) might be via `systematic integration by parts', developed in \citet[Chapter 3]{jungel2016entropy}. 

\begin{remark}[Different scalings for SVGD and overdamped Langevin] 
	\label{rem:scalings}
	It is important to note that comparing the convergence properties for the Stein PDE \eqref{eq:Stein pde} and the Fokker-Planck equation does not straightforwardly lead to any conclusions about the associated algorithms, as the Fokker-Planck equation arises from a different scaling. Indeed, consider $N$ independent particles moving according to
	\begin{equation}
	\label{eq:overdamped system}
	\mathrm{d}X^i_t = -\nabla V(X_t^i) \, \mathrm{d}t + \sqrt{2} \, \mathrm{d}W_t^i, \quad i = 1,\ldots,N,
	\end{equation}  
	where $(W_t^i)_{t \ge 0}$ denote independent standard Brownian motions.
	By arguments similar to those used in the proof of Proposition \ref{prop:mean field limit} it is possible to show that the associated empirical measure $\rho_t^N = \frac{1}{N} \sum_{i=1}^N \delta_{X_t^i}$ converges towards the solution of the Fokker-Planck equation
	\begin{equation}
	\label{eq:FKP}
	\partial_t \rho = \nabla \cdot( \rho \nabla V + \nabla \rho).
	\end{equation}
	Notice that the interacting system \eqref{eq:noisy_stein_coordinate} contains an additional factor of $\frac{1}{N}$ in comparison with \eqref{eq:overdamped system}. Since this corresponds to a time rescaling of the form $t \mapsto t/N$, the Stein mean-field limit describes the evolution on a fast time scale, in comparison with \eqref{eq:FKP}. Direct comparisons between Langevin sampling (based on the formulation in \eqref{eq:overdamped system}) and SVGD (either using \eqref{eq:Stein ode} or its stochastic variant \eqref{eq:noisy_stein_coordinate}) are hence very challenging, both from a practical and a theoretical point of view.
    Intuitively, it seems reasonable to expect that the interacting nature of SVGD type schemes (in particular, the repulsion term including $\nabla k$) might be advantageous when $V$ is multimodal and non-interacting Langevin samplers struggle to explore the whole state space.
	We leave an in-depth study of this important problem for future work.
\end{remark}

\section{Curvature at equilibrium}
\label{sec:curvature}

\label{sec:equilibrium}
In this section we study the properties of the bilinear form $\Hess_\pi$, i.e. the curvature at equilibrium. By a continuity argument and according to Section \ref{sec:2nd order} (see in particular Theorem \ref{thm:convexity contraction} and Lemma \ref{lem:functional inequality}), we expect rapid convergence of solutions started close to equilibrium if and only if $\Hess_\pi$ is bounded from below in the following sense\footnote{A similar reasoning has been employed in \citep{otto1998dynamics} in the context of pattern formation in magnetic fluids.}:
\begin{definition}[Exponential decay near equilibrium]
	We say that \emph{exponential decay near equilbrium} holds if there exists $\lambda > 0$ such that  
	\begin{equation}
	\label{eq:local BE}
	\Hess_{\pi}(\Psi,\Psi) \ge \lambda \int_{\mathbb{R}^d} \int_{\mathbb{R}^d} \nabla \Psi(y) \cdot k(y,z) \nabla \Psi(z) \mathrm{d}\pi(y) \mathrm{d}\pi(z) 
	\end{equation}
	holds for all $\Psi \in C_c^\infty(\mathbb{R}^d)$. In this case we call the largest possible choice of $\lambda$ the \emph{local convergence rate}.   
\end{definition}
For algorithmic performance, it is clearly desirable that exponential decay near equilibrium holds and that $\lambda$ can be chosen as large as possible. The following notion will turn out to be useful for a finer comparison between different kernels:
\begin{definition}[Rayleigh coefficients]
	\label{def:Rayleigh}
	For $\Psi \in C_c^{\infty}(\mathbb{R}^d) \setminus \{0\}$, the associated \emph{Rayleigh coefficient} is defined by
	\begin{equation}
 \nonumber
	\lambda_{\Psi}^k := \frac{\mathrm{Hess}_{\pi}(\Psi,\Psi)}{\int_{\mathbb{R}^d} \int_{\mathbb{R}^d} \nabla \Psi(y) \cdot k(y,z) \nabla \Psi(z)\, \mathrm{d}\pi(y) \mathrm{d}\pi(z)}.
	\end{equation}
If $k_1$ and $k_2$ are positive definite kernels, we say that \emph{$k_1$ locally dominates $k_2$} if
\begin{equation}
\nonumber
\lambda^{k_1}_{\Psi} \ge \lambda^{k_2}_{\Psi}
\end{equation}
for all $\Psi \in C_c^{\infty}(\mathbb{R}^d) \setminus \{0\}$.
\end{definition}
\begin{remark}
	From Remark \ref{rem:unit speed} we have 
	\begin{equation}
 \nonumber
	\lambda_\Psi^k = \frac{\mathrm{d}^2}{\mathrm{d}t^2} \mathrm{KL}(\rho_t \vert \pi) \big\vert_{t = 0},
	\end{equation}
	where $(\rho_t)_{t \in (-\varepsilon, \varepsilon)} \subset M$ is a curve with $\rho_0 = \pi$ and $\partial_t \rho + \nabla \cdot (\rho \mathcal{T}_{k,\rho} \nabla \Psi) = 0$. Intuitively, $k_1$ locally dominates $k_2$ precisely when, in the geometry associated to $k_1$, the $\mathrm{KL}$-divergence `appears to be more curved at $\pi$' than in the geometry associated with $k_2$, `in all directions'.
\end{remark}
In what follows, we will start with the analysis of the functional inequality \eqref{eq:local BE}, in particular identifying guidelines for a judicious choice of $k$. 
\\\\
Integration by parts in $x$ shows that the expressions \eqref{eq:nonequilibrium1} and \eqref{eq:nonequilibrium2} vanish for $\rho = \frac{1}{Z} \exp(-V)$, so that   
\begin{subequations}
\label{eq:q equilibrium}
\begin{align}
q_{ij}[\pi](y,z)  & = - \int_{\mathbb{R}^d} \partial_{x_j} \left(e^{V(x)} \partial_{x_i} \left( e^{-V(x)} k(x,y) \right) \right) k(x,z) \, \mathrm{d}\pi(x)
\\
\label{eq:BEpi}
& = \int_{\mathbb{R}^d} \partial_i \partial_j V(x) k(x,y) k(x,z) \, \mathrm{d}\pi (x)
+ \int_{\mathbb{R}^d} \partial_{x_i} k(x,z) \partial_{x_j} k(x,y) \, \mathrm{d} \pi(x).
\end{align}
\end{subequations}
It is thus appropriate to associate the contributions \eqref{eq:nonequilibrium1} and \eqref{eq:nonequilibrium2} to the behaviour of SVGD for distributions far from equilibrium. 
The expression \eqref{eq:BEpi} relates the curvature properties of the $\mathrm{KL}$-divergence at $\pi$ to those of $V$ through its Hessian. Instructively, the same is true for the Wasserstein-Hessian, leading to the celebrated Bakry-{\'E}mery criterion (see Appendix \ref{app:Langevin}).
We will see that the functional inequality \eqref{eq:local BE} can be conveniently expressed in terms of the linear operator
\begin{equation}
\label{eq:local generator}
\mathcal{L} \phi = - \sum_{i=1}^d e^V \partial_i \left( e^{-V} \mathcal{T}_{k,\pi} \partial_i \phi \right), \quad \phi \in C_c^\infty(\mathbb{R}^d).
\end{equation}
Integration by parts shows that $\mathcal{L}$ is symmetric and positive semi-definite on $L^2(\pi)$. By slight abuse of notation, we will denote its self-adjoint (Friedrichs-)extension by the same symbol, and its domain of definition by $\mathcal{D}(\mathcal{L})$. We would like to stress that the expression \eqref{eq:local generator} is well-defined even though the kernel $k$ might not be differentiable. Indeed, $T_{k,\pi}\partial_i \phi$ is smooth without regularity assumptions on $k$, provided that $\pi$ and $\phi$ are regular enough. Note also that under Assumption \ref{ass:ispd} on the kernel $k$, the null space of $\mathcal{L}$ coincides with the constant functions (for a proof we refer to the proof of Lemma \ref{lem:funct inequalities} in Appendix \ref{app:equilibrium}).

The role of $\mathcal{L}$ becomes clear from the following lemma. Recall the definition of $L_0^2(\pi)$ from \eqref{eq:L20}.

\begin{lemma}
	\label{lem:funct inequalities}
	Let $k$ satisfy Assumption \ref{ass:ispd}. For $\lambda \ge  0$, the following are equivalent:
	\begin{enumerate}
		\item 
		\label{it:abstract BE}
		The inequality \eqref{eq:local BE} holds for all $\Psi \in C_c^\infty(\mathbb{R}^d)$.
		\item 
		\label{it:Poincare}
		The \emph{`Stein-Poincar{\'e} inequality'} 
		\begin{equation}
		\label{eq:Poincare}
		\langle \phi, \mathcal{L} \phi \rangle_{L^2(\pi)} \ge  \lambda  \langle \phi, \phi \rangle_{L^2(\pi)} 
		\end{equation}
		holds for all $\phi \in L_0^2(\pi) \cap C_c^\infty(\mathbb{R}^d)$. 
	\end{enumerate}
\end{lemma}
\begin{proof}
	See Appendix \ref{app:equilibrium}.
\end{proof}
\begin{remark}
	Let $\lambda \ge 0$ be the smallest nonnegative real number such that one (equivalently, both) of the inequalities \eqref{eq:local BE} and \eqref{eq:Poincare} hold(s) for all $\Psi \in C_c^\infty(\mathbb{R}^d)$. Then
	\begin{equation}
 \nonumber
	\lambda = \inf ( \sigma(\mathcal{L}) \setminus \{0\}),
	\end{equation}
	where $\sigma(\mathcal{L})$ denotes the spectrum of $\mathcal{L}$. Inequalities of the form \eqref{eq:Poincare} are therefore often termed \emph{spectral gap inequalities}. In the theory of the Fokker-Planck equation, \eqref{eq:Poincare} has a direct analogue, the role of $-\mathcal{L}$ is taken by the generator of overdamped Langevin dynamics \citep[Chapter 4]{bakry2013analysis}, see also Appendix \ref{app:Langevin}. 
\end{remark}
\begin{remark}
For clarity, we emphasised the fact that $k$ is assumed to be ISPD (see Assumption \ref{ass:ispd}) in the statement of Lemma \ref{lem:funct inequalities}, as the result will fail to hold otherwise. As we believe that non-ISPD kernels are of algorithmic interest, an extension of this result to this setting is subject of ongoing work.   
\end{remark}
\begin{remark}[Linearisation around $\pi$]
	The following represents an alternative way of deriving the Stein-Poincar{\'e} inequality \eqref{eq:Poincare}. Assuming that $(\rho_t)_{t \ge 0}$ solves the Stein PDE \eqref{eq:Stein pde}, a simple calculation yields
	\begin{equation}
	\label{eq:Fisher}
	\partial_t \mathrm{KL} (\rho_t \vert \pi) = \int_{\mathbb{R}^d} \int_{\mathbb{R}^d} \nabla \left(\frac{\rho_t(y)}{\pi(y)} \right) \cdot  k(y,z) \nabla \left(\frac{\rho_t(z)}{\pi(z)} \right) \mathrm{d}\pi(y) \mathrm{d}\pi(z) =:I_{\mathrm{Stein}}(\rho_t \vert \pi),
	\end{equation}
	where we have defined the \emph{`Stein-Fisher information'} $I_{\mathrm{Stein}}$. Assuming a \emph{`Stein-log-Sobolev inequality'} of the form
	\begin{equation}
	\label{eq:log Sobolev}
	\mathrm{KL}(\rho \vert \pi) \le \frac{1}{2 \lambda} I_{\mathrm{Stein}} (\rho \vert \pi),
	\end{equation}
	the exponential decay estimate \eqref{eq:KL exp decay} would follow by a standard Gronwall argument (see, for instance, \citet[Theorem 5.2.1]{bakry2013analysis} in the context of the usual log-Sobolev inequality). We can now analyse \eqref{eq:log Sobolev} for small perturbations of the target $\pi$. Setting $\rho = (1 + \varepsilon \phi) \pi$ with $\int_{\mathbb{R}^d} \phi \, \mathrm{d}x = 0$ and $\varepsilon \ll 1$, we obtain
	\begin{equation}
 \nonumber
	\mathrm{KL}(\rho \vert \pi)  \simeq \frac{1}{2}\varepsilon^2  \Vert \phi \Vert^2_{L^2(\pi)} \quad \text{and} \quad I_{\mathrm{Stein}} (\rho \vert \pi) \simeq \varepsilon^2 \langle \phi, \mathcal{L} \phi \rangle_{L^2(\pi)}  
	\end{equation}   
	to leading order, recovering \eqref{eq:Poincare} from \eqref{eq:log Sobolev} in the limit as $\varepsilon \rightarrow 0$. This argument is well-known in the case of the usual log-Sobolev and Poincar{\'e} inequalities (see \citet[Proposition 5.1.3]{bakry2013analysis}). Finally, we refer the reader to \cite{li2019diffusion} for a study of related functional inequalities in the context of modified transport geometries.  
\end{remark}
The next lemma shows that exponential convergence to equilibrium does not hold if $k$ is too regular:
\begin{lemma}
	\label{lem:k smooth}
	Let $k \in C^{1,1}(\mathbb{R}^d \times \mathbb{R}^d)$, and assume the integrability condition
	\begin{equation}
	\label{eq:integrability}
	\sum_{i=1}^d \left[ (\partial_i V(x))^2 k(x,x) - \partial_i V(x) \left(\partial_i^1 k(x,x) + \partial_i^2k(x,x)\right) + \partial_i^1 \partial_i^2 k(x,x) \right] \mathrm{d}\pi(x) < \infty, 
	\end{equation}
	where $\partial_i^1$ and $\partial_i^2$ denote derivatives with respect to the first and second argument of $k$, respectively.
	Then the inequalities \eqref{eq:local BE} and \eqref{eq:Poincare} only hold for $\lambda = 0$, i.e. exponential convergence to equilibrium does not hold.  
\end{lemma}

\begin{proof}
	See Appendix \ref{app:equilibrium}.
\end{proof}

\begin{remark}
	The integrability condition \eqref{eq:integrability} is very mild; it holds for instance in the case whenever $\pi$ has exponential tails and the derivatives of $k$ and $V$ grow at most at a polynomial rate. 
\end{remark}

\subsection{The one-dimensional case}
\label{sec:1d}
In this subsection we discuss the functional inequality \eqref{eq:local BE} in the case $d=1$, when it simplifies considerably (see Lemma \ref{lem:V k separate} below). The higher-dimensional case appears to be significantly more involved and will be considered in forthcoming work. 
\begin{lemma}
	\label{lem:V k separate}
	Assume that $d=1$, $\mathcal{H}_k \subseteq H^1(\pi)$ with dense embedding, $V \in C^2(\mathbb{R})$ with bounded second derivative and $\lambda > 0$. Then \eqref{eq:local BE} holds for all $\Psi \in C_c^\infty(\mathbb{R}^d)$ if and only if 
	\begin{equation}
	\label{eq:BE 1d}
	\int_{\mathbb{R}} V'' \phi^2 \, \mathrm{d}\pi + \int_{\mathbb{R}} (\phi')^2 \, \mathrm{d}\pi \ge \lambda \Vert \phi\Vert^2_{\mathcal{H}_k}
	\end{equation}
	for all $\phi \in \mathcal{H}_k$. 	
\end{lemma}
\begin{proof}
	See Appendix \ref{app:equilibrium}.
\end{proof}
The utility of the formulation \eqref{eq:BE 1d} resides in the fact that $V$ and $\pi$ only appear on the left-hand side while $k$ only appears on the right-hand side. Hence, in the one-dimensional case and when the conditions of Lemma \ref{lem:V k separate} are satisfied, optimal kernel choice and the influence of the target measure can be discussed separately. We have the following corollary on translation-invariant kernels:
\begin{corollary}
	\label{cor:translation invariant}
	Assume the conditions from Lemma \ref{lem:V k separate} and furthermore that $k$ is translation-invariant, i.e. that there exists $h\in C(\mathbb{R}) \cap C^1(\mathbb{R} \setminus \{0\})$ with $h$ absolutely continuous such that $k(x,y) = h(x-y)$. If moreover $h(x) \rightarrow 0$ and $h'(x) \rightarrow 0$ as $x \rightarrow \pm \infty$, 
	then exponential convergence near equilibrium does not hold.
\end{corollary}
\begin{proof}
	See Appendix \ref{app:equilibrium}. 
\end{proof}
The following example shows that the main assumptions of Lemma \ref{lem:k smooth} (differentiability of the kernel) and Corollary \ref{cor:translation invariant} (translation-invariance of the kernel) cannot be dropped. In other words, rapid convergence close to equilibrium can be achieved by choosing a nondifferentiable kernel that is adapted to the tails of the target:
\begin{example}
	\label{ex:exp decay}
	In the case $d=1$, consider the `weighted Mat{\'e}rn kernel' 
	\begin{equation}
	\label{eq:weighted matern}
	k(x,y) = \pi(x)^{-1/2} e^{-\vert x - y \vert}\pi(y)^{-1/2},
	\end{equation}
	and assume that there exists a constant $\tilde{\lambda} > 0$ such that 
	\begin{equation}
	\label{eq:BK matern}
	V''(x) + \frac{(V')^2(x)}{2} \ge \tilde{\lambda},
	\end{equation}
	for all $x \in \mathbb{R}$. Then exponential convergence near equilibrium holds, with the explicit constant
	\begin{equation}
	\label{eq:Matern lambda}
	\lambda = \min \left(1,\tilde{\lambda}/2 \right).
	\end{equation} 
	We present the calculation justifying this statement in Appendix \ref{app:equilibrium}.
\end{example}
In the case when \eqref{eq:BE 1d} is valid, we can characterise the local dominance of kernels (in the sense of Definition \ref{def:Rayleigh}) in terms of the unit-balls in $\mathcal{H}_{k_1}$ and $\mathcal{H}_{k_2}$:
\begin{lemma}
	\label{lem:dominance balls}
	Let $k_1$ and $k_2$ be two positive definite kernels, and assume that the conditions in Lemma \ref{lem:V k separate} are satisfied for both. Then $k_1$ dominates $k_2$ if and only if $\mathcal{H}_{k_2} \subset \mathcal{H}_{k_1}$ and 
	\begin{equation}
 \nonumber
	\Vert \phi \Vert_{\mathcal{H}_{k_1}} \le \Vert \phi \Vert_{\mathcal{H}_{k_2}},
	\end{equation}
	for all $\phi \in \mathcal{H}_{k_2}$.
\end{lemma}
\begin{proof}
	See Appendix \ref{app:equilibrium}.
\end{proof}
To exemplify the statement of Lemma \ref{lem:dominance balls}, let us consider the kernels $k_{p,\sigma}:\mathbb{R}^d \times \mathbb{R}^d \rightarrow \mathbb{R}$ defined in \eqref{eq:p kernel}. We recall that $p \in (0,2]$ is a smoothness parameter and $\sigma > 0$ denotes the kernel width. The relation between the corresponding RKHSs is as follows:
\begin{lemma}
	\label{lem:exp p kernel}
	The following hold:
	\begin{enumerate}
		\item
		$k_{p,\sigma}$ is a strictly integrally positive definite kernel, for all $p \in (0,2]$ and $\sigma > 0$.
		\item 
		If $p > q$ then $\mathcal{H}_{k_{p,\sigma_p} }\subset \mathcal{H}_{k_{q,\sigma_q}}$, for all $\sigma > 0$. The inclusion is strict. 
		\item
		If $p > q$ then there exist $\sigma_{p},\sigma_q > 0$ such that
\begin{equation}
\nonumber
\Vert \phi \Vert_{\mathcal{H}_{k_{q,\sigma_q}}} \le \Vert \phi \Vert_{\mathcal{H}_{k_{p,\sigma_p}}},
\end{equation}
for all $\phi \in \mathcal{H}_{k_{p,\sigma_p}}$.
	\end{enumerate}
\end{lemma}
\begin{proof}
See Appendix \ref{app:equilibrium}.
\end{proof}
The preceding result in conjunction with Lemma \ref{lem:dominance balls} suggests that choosing a smaller value of $p \in (0,2]$ and adjusting $\sigma$ accordingly when simulating SVGD dynamics with a kernel of the form \eqref{eq:p kernel} might lead to improved algorithmic performance. Note, however, that there is a computational cost associated to kernels with small $p$, as the equations \eqref{eq:Stein ode} become stiff. 
In Section \ref{sec:numerics} we investigate these issues in numerical experiments. 

\section{Outlook: polynomial kernels}
\label{sec:polynomial}

In \citet{liu2018stein} the authors suggest using polynomial kernels of the form $k(x,y) = x \cdot y + 1$ when the target measure is approximately Gaussian. Here we would like to point out that the formulas obtained in Lemma \ref{lem:Hessian} are convenient for the analysis of this case since all the derivatives can be computed explicitly and have simple forms. An in-depth analysis of the implications for the use of polynomial kernels would be beyond the scope of this work, but we present the following result:
\begin{lemma}
	\label{lem:xy kernel}
	Let $d = 1$, $V(x) = \frac{\alpha}{2}x^2$, $\alpha > 0$ and $k(x,y) =xy$. Then 
	\begin{equation}
	\label{eq:polynomial curvature}
	q[\rho](y,z) = 2 \alpha k(y,z) \int_{\mathbb{R}}  x^2 \, \mathrm{d}\rho(x),
 	\end{equation}
	and hence \eqref{eq:hessian metric bound} holds with
	\begin{equation}
	\label{eq:lambda polynomial}
	\lambda = 2 \alpha \int_{\mathbb{R}} x^2 \, \mathrm{d}\rho(x).
	\end{equation}
\end{lemma}
\begin{proof}
The identity \eqref{eq:polynomial curvature} follows by straightforward calculation from \eqref{eq:Hessian}. 
\end{proof}

Lemma \ref{lem:xy kernel} is an encouraging result since $\lambda > 0$ whenever $\rho \neq \delta_0$. Furthermore, the rate of contraction is naturally linked to the second moment of the measure $\rho$. A more detailed study of polynomial kernels in the multidimensional setting and for non-Gaussian targets is the subject of ongoing work. 

\begin{remark}
	Since polynomial kernels are not ISPD (and hence violate Assumption \ref{ass:ispd}), convergence to the target $\pi$ is not guaranteed. However, we note that $\widetilde{k} = k + \varepsilon k_0$ is ISPD whenever $k_0$ is (and where $\varepsilon > 0$, $k$ being \emph{any} kernel). Polynomial kernels are thus admissible in our framework (and Lemma \ref{lem:xy kernel} might be indicative) when used in conjuntion with a small perturbation, for instance by a kernel of the form \eqref{eq:p kernel}.  
\end{remark}

\section{Numerical Experiments}
\label{sec:numerics}
In this section, using numerical experiments, we demonstrate that some of the results (see in particular Example \ref{ex:exp decay} as well as the discussion following Lemmas \ref{lem:dominance balls} and \ref{lem:exp p kernel}) arising from the mean-field analysis of Section \ref{sec:curvature} carry through to the associated finite-particle model.  In particular, we demonstrate that the smoothness of the kernel plays a significant role on the performance of the SVGD dynamics as a sampling algorithm.  We  study two simple Gaussian mixture model tests.  In the first example we consider the one dimensional target distribution $\pi = \frac{1}{4}\mathcal{N}(2, 1)+\frac{1}{4}\mathcal{N}(-2, 1)+\frac{1}{4}\mathcal{N}(6, 1)+\frac{1}{4}\mathcal{N}(-6, 1)$ on $\mathbb{R}$. The standard SGVD dynamics \eqref{eq:Stein ode} are simulated for $N=500$ particles up to time $T = 5000$. The resulting ODE was integrated using an implicit variable order BDF scheme \citep{byrne1975polyalgorithm}, for which we keep track of the number of gradient evaluations throughout the entire simulation. We investigate kernels of the form  \eqref{eq:p kernel} for different values of $p \leq 2$.  We first consider such kernels with fixed $p$ taking values $2, 1, \frac{1}{2},\ldots$. The behaviour of the scheme is strongly dependent on the choice of the bandwidth $\sigma$.  Following \citet{liu2016stein} and all subsequent works we choose $\sigma$ according to the median heuristic.   In Figure \ref{fig:one_dim_histogram} the histograms of the empirical distributions is plotted at the final time.  We observe a significant improvement in accuracy between $p=1$ and $p=2$, with the $500$ particles packing far more efficiently as $p$ is decreased from $2$.  However, moving beyond $p=1$ the approximation starts to suffer close to the tails of the distribution, suggesting that more particles would be needed as $p$ is taken to $0$.  The temporal behaviour is shown in Figure \ref{fig:one_dim_evolution} which plots the Wasserstein-1 distance between the target density and the empirical SVGD distribution over time. The Wasserstein distance was computed using the Python Optimal Transport Library \citep{flamary2017pot} based on an exact sample of size $10^7$.  For $p=1$ we observe that the finite-particle bias in the stationary distribution is far lower.  However, decreasing $p$ further down to $\frac{1}{2}$ we do not see this improvement being sustained.  In the right-hand side figure, the Wasserstein error is plotted as a function of the number of gradient evaluations, which characterises computational cost.  We observe that, after an initial transient period, the simulation for $p=1$ is far more accurate per unit cost, whilst maintaining this accuracy becomes more expensive as $p$ decreases. The latter is in line with the fact that the derivatives of the kernels \eqref{eq:p kernel} become unbounded for $p \in (0,1)$, and so the system \eqref{eq:Stein ode} becomes numerically significantly stiffer in that regime.   Simulating SVGD for $p$ smaller than $\frac{1}{2}$ the accuracy degrades very strongly. These observations suggest that a kernel with a time-evolving value of $p$ might achieve the `best of both worlds'.  To this end we consider a form of \emph{annealing} where we take $\log p(t) = (1-t/T)\log p_0 + t/T \log p_1$, for $t \in [0,T]$ and where $T=1000$ is the final simulation time.  We choose $p_0 = 2$ and $p_1 = \frac{1}{2}$.  The convergence results for the annealing strategy are shown in  \ref{fig:one_dim_histogram} and \ref{fig:one_dim_evolution}.  We observe that the annealed version attains the lowest Wasserstein error overall, substantially lower than the fixed $p$-kernels at time $t=500$.  However, this advantage quickly diminishes as $T$ increases to $1000$, suggesting that this is likely a finite-particle effect.   We observe similar behaviour when plotting the Wasserstein error against the number of gradient evaluations.  While it is clear that there is potential for performance increases through annealing, it is evident that this is very sensitive to the particular annealing `schedule', and we leave a detailed study for further work.
\\\\
As a second example, a two-dimensional Gaussian mixture model is considered defined by $\pi = \frac{1}{6}\sum_{i=1}^6 \mathcal{N}(\mu_i, \Sigma_i)$, where $\mu_1 = (-5, -1)^\top$, $\mu_2 = (-5, 1)^\top$, $\mu_3 = (5, -1)^\top$, $\mu_4 = (5, 1)^\top$, $\mu_5 = (0, 1)^\top$, $\mu_6 = (0, -1)^\top$ and 
$$
\Sigma_1 = \Sigma_2 = \Sigma_3 = \Sigma_4 = \frac{1}{5}I_{2\times 2}, \mbox{ and } \Sigma_5 = \Sigma_6 = \left(\begin{matrix} 10 & 0 \\ 0 & \frac{1}{2}\end{matrix}\right),
$$
see Figure \ref{fig:two_dim_truth}.  Standard SVGD dynamics are simulated with 500 particles up to time $T = 5000$ using a kernel of the form \eqref{eq:p kernel} with $p=2, 1, 0.5, 0.1$, etc.  We see from Figure \ref{fig:two_dim_evolution} that the lowest error (in terms of Wasserstein-1 distance) is attained when $p=0.5$, after which the performance degrades very rapidly.   From the right-hand side plot, we also observe that $p=0.5$ provides the lowest error per unit computational cost, after an interim transient period.
\\\\
Both the above examples suggest that $p$ needs to be tuned to the target distribution, and that it could be updated adaptively.  We leave investigations of such adaptation strategies for future work. Finally, we remark that Corollary \ref{cor:translation invariant} suggests using non-translation-invariant kernels with adapted tails as in Example \ref{ex:exp decay}. In our numerical studies we find, however, that doing so incurs an additional computational cost that often outweighs the favourable properties of the associated mean-field dynamics. Still, developing SVGD schemes relying on kernels with appropriately adapted tails might be an interesting direction for further research.

\begin{figure}[t!]
\begin{center}
\includegraphics[width=0.24\textwidth]{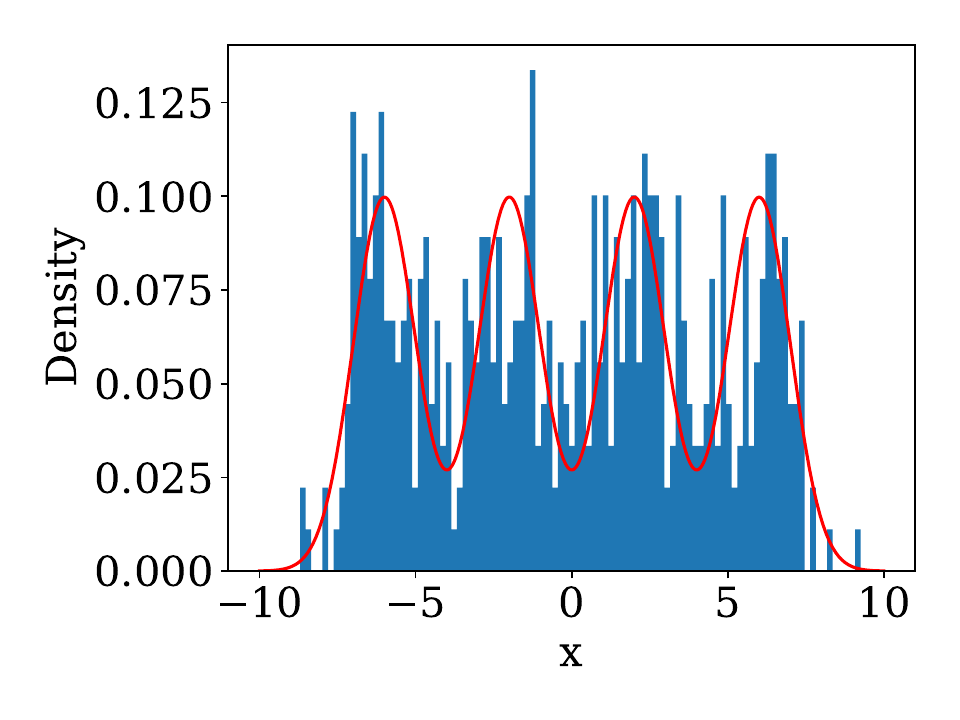}
\includegraphics[width=0.24\textwidth]{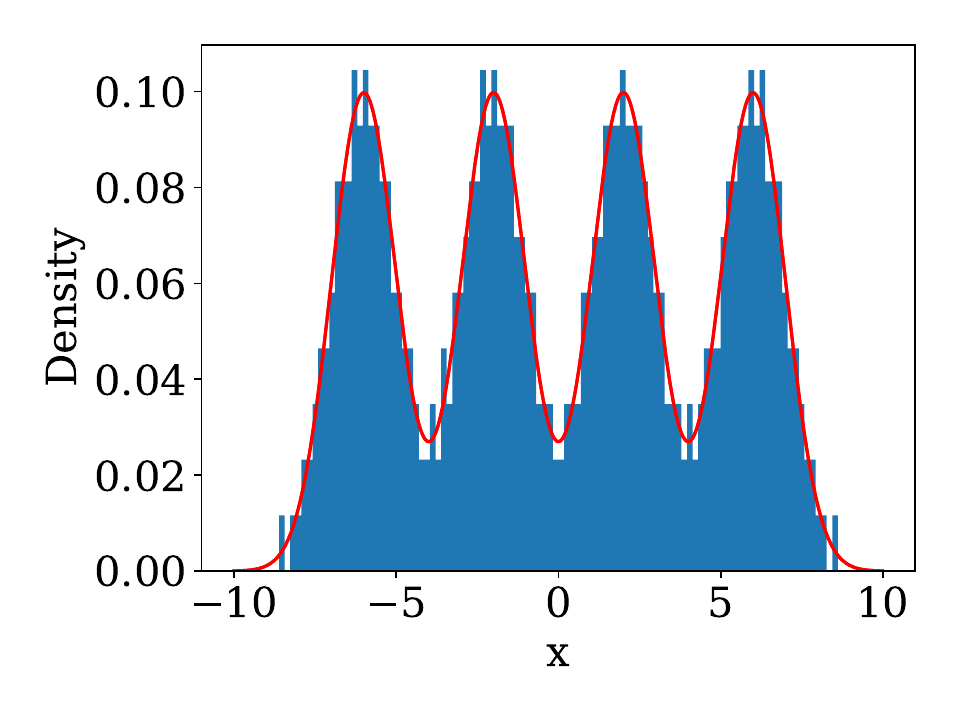}
\includegraphics[width=0.24\textwidth]{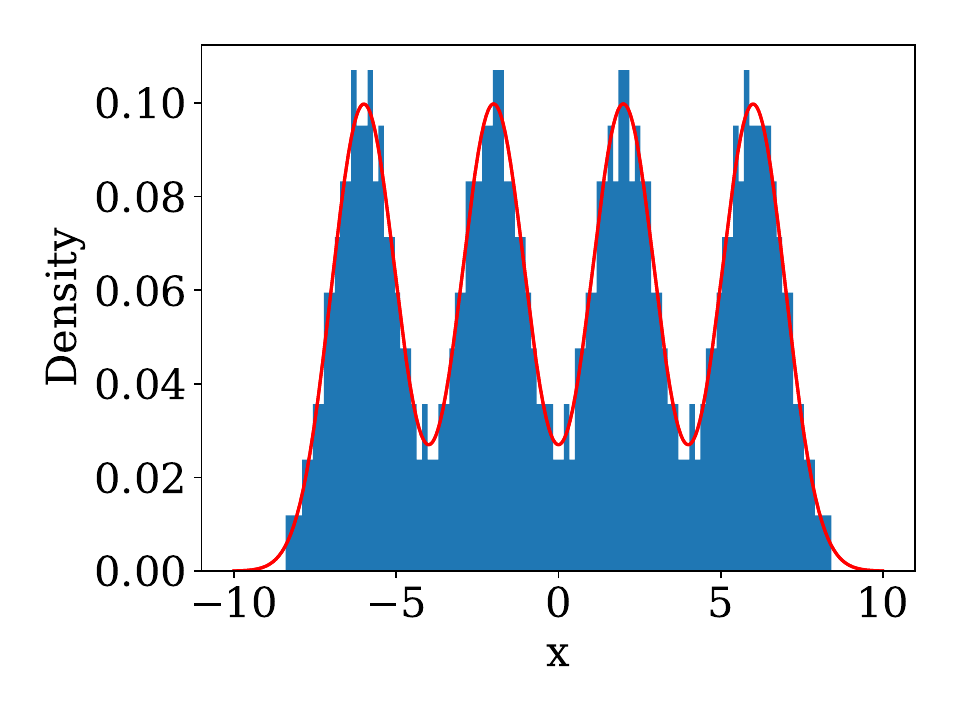}
\includegraphics[width=0.24\textwidth]{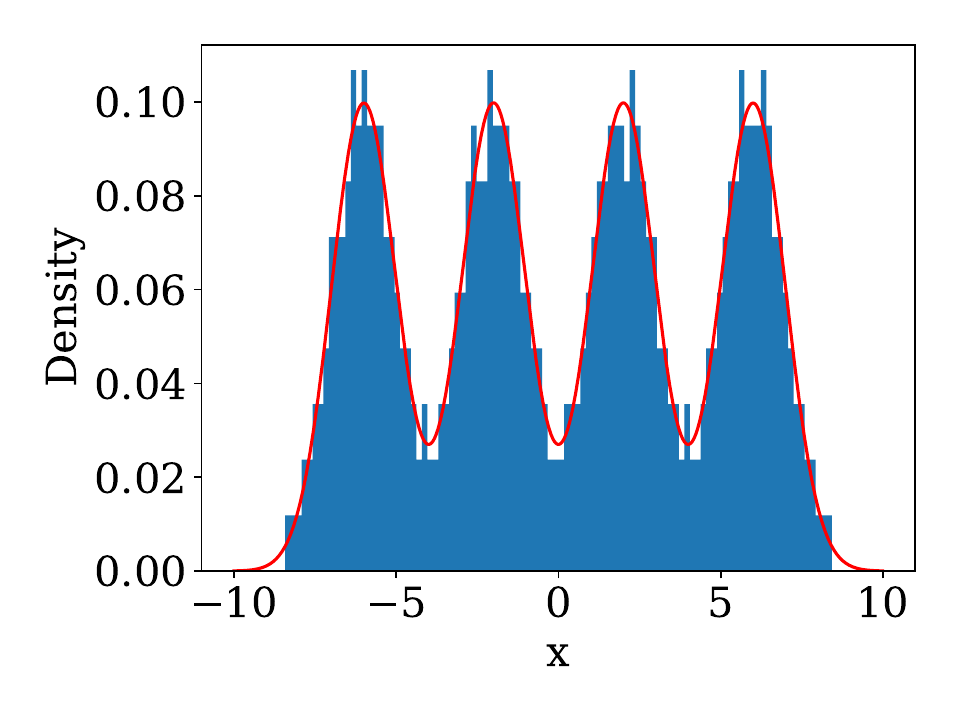}

\caption{Histogram of the empirical distribution of $N=500$ particles at final time, simulated according to standard SVGD dynamics for $T=500$ time units. The first three are using a kernel of the form  \eqref{eq:p kernel} for $p=2,1, \frac{1}{2}$, respectively. The last histogram employs a time-evolving kernel where the value of $p$ is evolved from $p_0=2$ to $p_1=\frac{1}{2}$.  The red line denotes the target density.\label{fig:one_dim_histogram}}
\end{center} 
\end{figure}

\begin{figure}[t!]
\begin{center}
\includegraphics[width=0.45\textwidth]{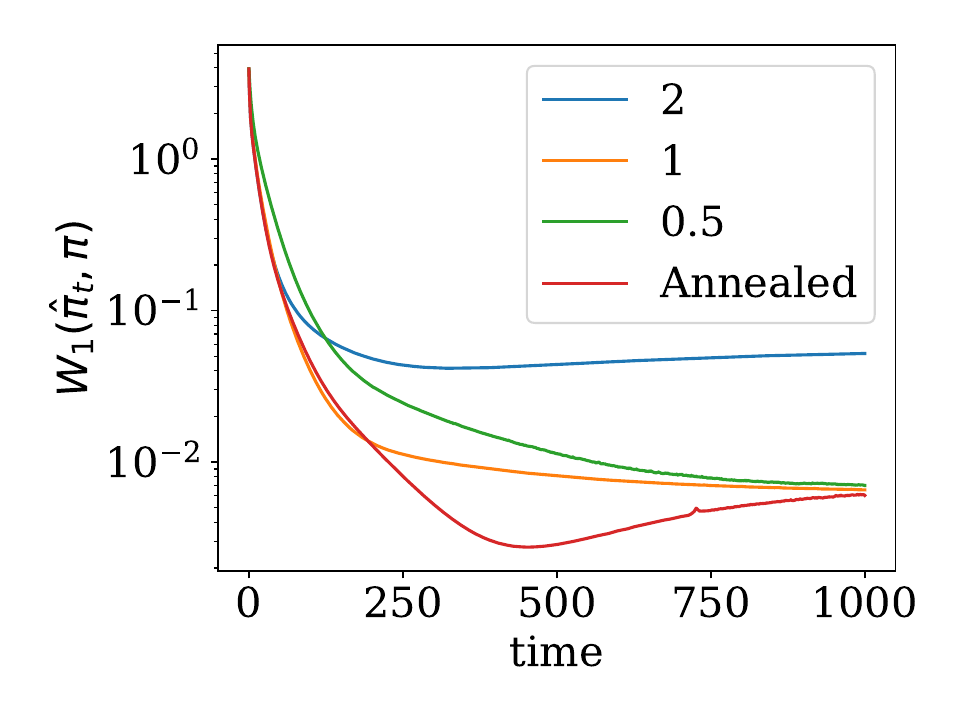}
\includegraphics[width=0.45\textwidth]{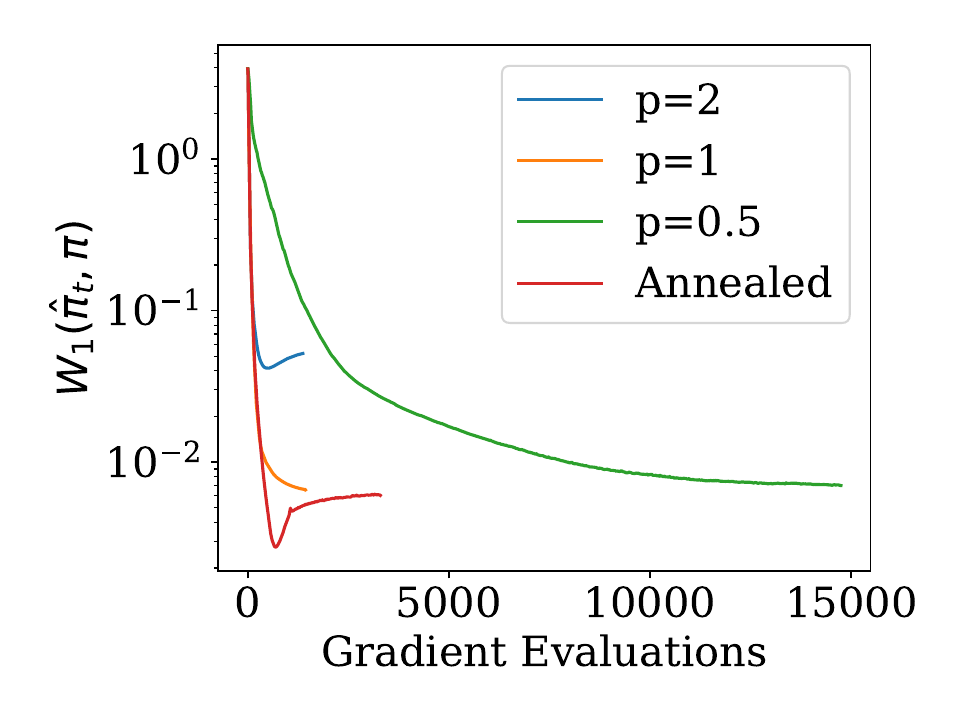}
\caption{Time evolution of the Wasserstein-1 error between the target and empirical distributions arising from simulating SVGD from $0$ to $T$ and different values of $p$.  In the left plot, the evolution is shown as a function of time.  In the right plot, it is shown as a function of the number of gradient evaluations, reflecting the true computational cost.\label{fig:one_dim_evolution}}
\end{center} 
\end{figure}

\begin{figure}[t!]
\begin{center}
\includegraphics[width=0.5\textwidth]{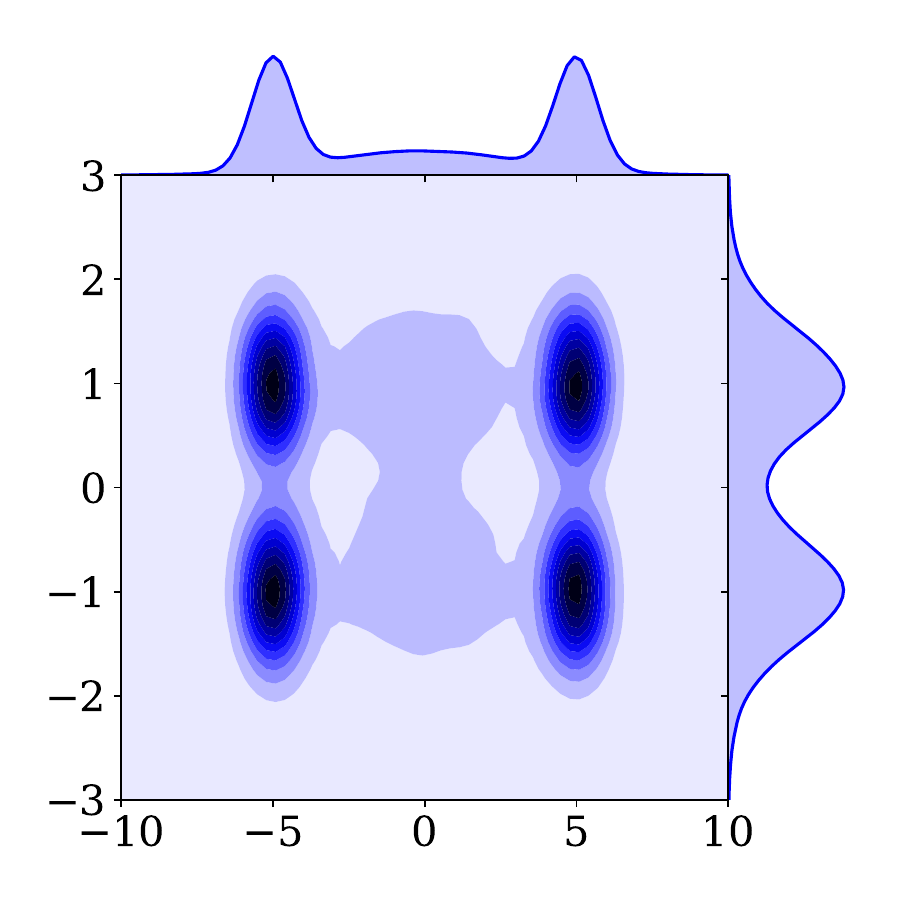}
\caption{Target distribution for the two-dimensional Gaussian mixture model example.\label{fig:two_dim_truth}}
\end{center} 
\end{figure} 

\begin{figure}[t!]
\begin{center}
\includegraphics[width=0.45\textwidth]{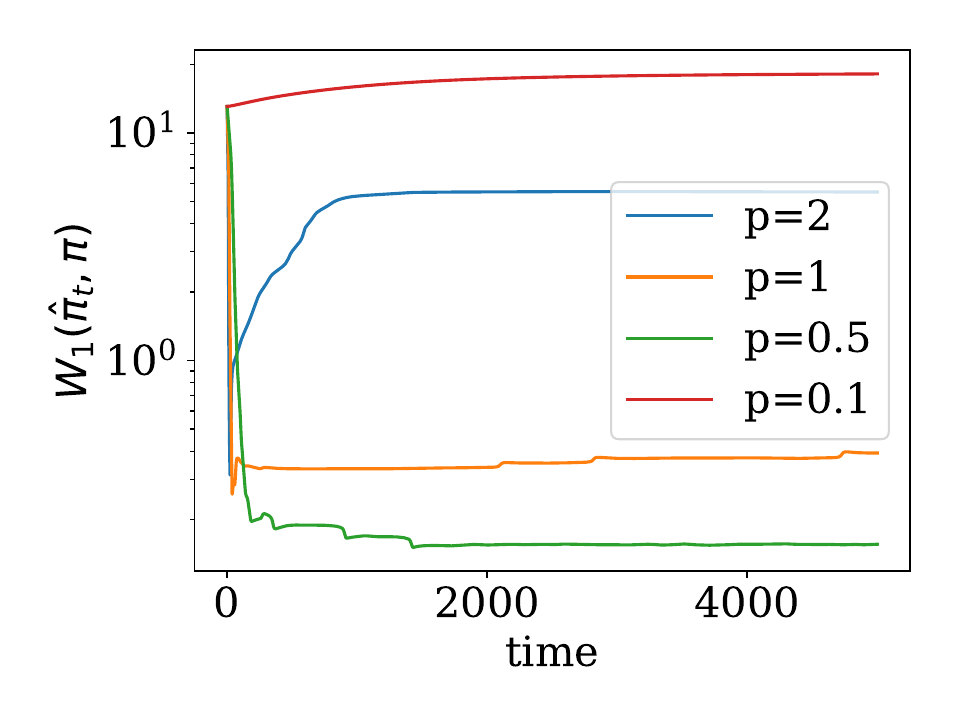}
\includegraphics[width=0.45\textwidth]{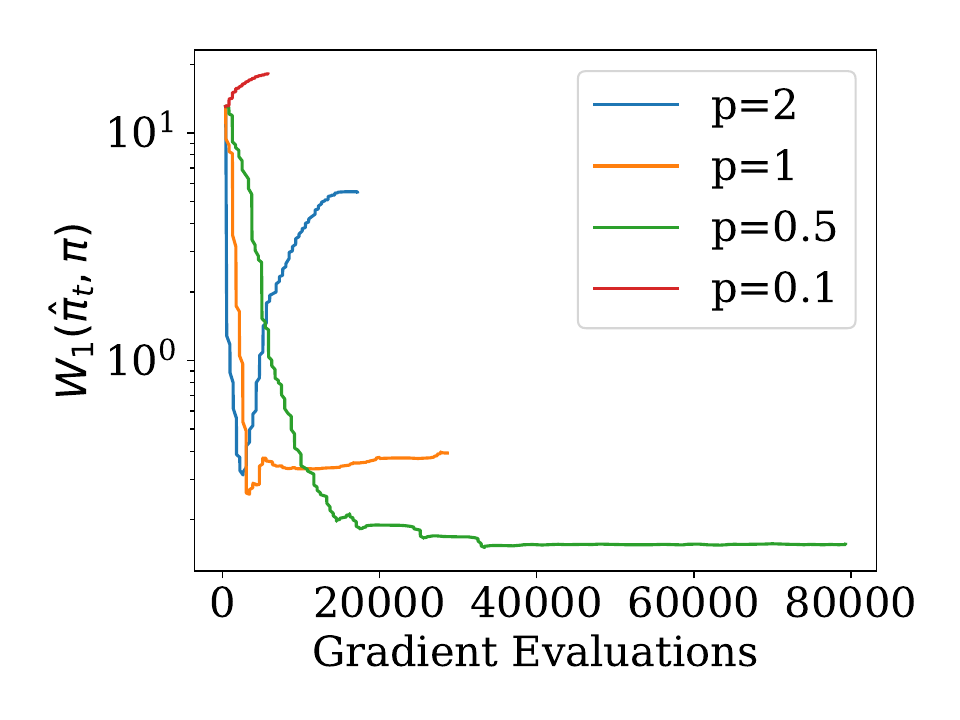}
\caption{Time evolution of the Wasserstein-1 error between the target and empirical distributions arising from simulating SVGD from $0$ to $T$ and different values of $p$ for the two-dimensional mixture model example.  In the left plot, the evolution is shown as a function of time.  In the right plot, it is shown as a function of the number of gradient evaluations, reflecting the true computational cost.\label{fig:two_dim_evolution}}
\end{center} 
\end{figure}

\section{Conclusions}
\label{sec:conclusions}
 
In this paper we have analysed the geometric properties of SVGD related to its gradient flow interpretation. In particular, we have extended the framework put forward in \citet{liu2017stein}, obtained the associated geodesic equations and used those results to derive functional inequalities connected to exponential convergence of SVGD dynamics close to equilibrium. We have leveraged the latter to develop principled guidelines for an appropriate choice of the kernel $k$ and verified those in numerical experiments. In particular, our theoretical considerations have led us to investigating singular kernels with adjusted tails. 

There are various avenues for further research. First, it would be interesting to place the geometric calculations in the framework of metric spaces developed in \citet{ambrosio2008gradient}, relaxing the regularity assumptions and placing in particular Proposition \ref{prop:geodesic equations} on a more rigorous foundation. It will be of key importance to extend the results obtained in Section \ref{sec:1d} to the multidimensional case. The numerical experiments have indicated that such an extension might be possible and yield further insights. Quantifying the speed of convergence for initial distributions far from equilibrium remains an open and challenging problem. As noted in Section \ref{sec:polynomial}, this might be possible (and first encouraging results are available) for polynomial kernels. Last but not least, we believe that understanding the properties of the finite-particle systems \eqref{eq:Stein ode} and \eqref{eq:noisy Stein} (as opposed to the mean field limit \eqref{eq:Stein pde}) will be important for further algorithmic advances. All of the preceding points are currently under investigation. 
\\\\

\noindent
{\bf Acknowledgement.} This research has been partially funded by 
Deutsche Forschungsgemeinschaft (DFG) through grant CRC 1114 \lq Scaling Cascades\rq  \,(project A02). NN would like to thank Alexander Mielke, Felix Otto and Sebastian Reich for stimulating discussions.  AD was supported by the Lloyds Register Foundation Programme on Data Centric Engineering and by The Alan Turing Institute under the EPSRC grant [EP/N510129/1]. We would like to thank the anonymous referees whose very careful reading and thoughtful remarks have helped to substantially improve the paper.

\appendix


\section{Analogies between Langevin dynamics and SVGD}
\label{app:Langevin}

In this appendix we will trace the similarities between overdamped Langevin dynamics and SVGD according to the gradient flow perspective. We note that a similar comparison has been made in \citet[Section 3.5]{liu2017stein}. Here our aim is to extend this discussion and place our results in this context.

\subsection{Overdamped Langevin dynamics, the Fokker-Planck equation and optimal transport}
\label{sec:Langevin}

To start with, let us consider the \emph{overdamped Langevin dynamics} \citep[Section 4.5]{pavliotis2014stochastic}  

\begin{equation}
\label{eq:overdamped Langevin}
\mathrm{d}X_t = -\nabla V(X_t) \, \mathrm{d}t + \sqrt{2}\, \mathrm{d}B_t, \quad X_0 = x_0. 
\end{equation}
It is well-known that under mild conditions on $V$ this SDE admits a unique strong solution $(X_t)_{t \ge 0}$ that is ergodic with respect to $\pi \propto e^{-V}$, see, for instance, \citet{roberts1996exponential}. This fact motivates using a suitable discretisation of \eqref{eq:overdamped Langevin} as a sampling scheme, laying the foundation for a number of (approximate) MCMC algorithms such as MALA and ULA \citep[Section 6.5.2]{robert2013monte}.  The law of $X_t$, denoted by $\rho_t := \Law(X_t)$, solves the \emph{Fokker-Planck equation} 
\begin{subequations}
	\label{eq:Fokker-Planck}
	\begin{align}
	\partial_t \rho & = \nabla \cdot(\rho \nabla V + \nabla \rho) \\
	\label{eq:FKP log grad}
	& = \nabla \cdot \left( \rho (\nabla V + \nabla \log \rho )\right).
	\end{align}
\end{subequations}
The value of the reformulation \eqref{eq:FKP log grad} becomes apparent when we notice that the Stein PDE \eqref{eq:Stein pde} can be written in the form
\begin{equation}
\nonumber
\partial_t \rho = \nabla \cdot \left( \rho \mathcal{T}_{k,\rho} (\nabla V + \nabla \log \rho )\right),
\end{equation} 
see Lemma \ref{lem:gradient} and Corollary \ref{cor:gradient flow}. In particular, the Fokker-Planck Onsager operator \citep{machlup1953fluctuations,mielke2016generalization,ottinger2005beyond}
$$
\mathbb{K}_\rho^{FP} : \phi \mapsto -\nabla \cdot(\rho \nabla \phi),
$$
should be compared to the Stein Onsager operator from Remark \ref{rem:Onsager}. 
As first observed in the seminal paper \citet{jordan1998variational}, the PDE \eqref{eq:Fokker-Planck} can be interpreted as the gradient flow of the $\mathrm{KL}$-divergence \eqref{eq:KL} with respect to the quadratic Wasserstein distance $\mathcal{W}_2$ using the Benamou-Brenier formula \citep{benamou2000computational}
\begin{equation}
\label{eq:Wasserstein}
\mathcal{W}_2^2(\mu,\nu) = \inf_{(\rho,v)} \left\{ \int_0^1 \Vert v_t \Vert_{L^2(\rho_t)}^2 \, \mathrm{d}t: \quad \partial_t \rho + \nabla \cdot(\rho v) = 0, \quad \rho_0 = \mu, \, \rho_1 =\nu \right\}. 
\end{equation}
As already noticed in Remark \ref{rem:d_W}, the Stein distance $d_k$ essentially differs from $\mathcal{W}_2$ only by exchanging the $L^2(\rho)$-norm for the $\mathcal{H}_k^d$-norm. The infimum in \eqref{eq:Wasserstein} remains the same if optimisation is carried out over gradient fields $v = \nabla \phi$, see for instance \citet[Section 1.4]{gigli2012second}. This is completely analogous to the optional constraint $ v_t \in \overline{\mathcal{T}_{k,\rho_t}\nabla C_c^\infty(\mathbb{R}^d)}^{\mathcal{H}_k^d}$ in Definition \ref{def:distance}, see \eqref{eq:Stein unconstrained}.
The geodesics associated to the distances $d_k$ and $\mathcal{W}_2$ are described by the systems of equations \eqref{eq:Stein geodesics} and \eqref{eq:Wasserstein geodesics}. As already observed in Remark \ref{rem:geodesics}, the equations pertaining to the Stein geometry are coupled, reflecting the fact that SVGD is based on an evolution of interacting particles.
In \citet[Section 3]{otto2000generalization}, the Hessian of the $\mathrm{KL}$-divergence in the Wasserstein geometry was computed; this expression should naturally be compared to the Hessian in the Stein geometry, see Lemma \ref{lem:Hessian}. Notably, the Wasserstein-Hessian can be related to the Ricci-curvature of the underlying manifold, an observation that has sparked numerous developments within the intersection between geometry and probability (see for instance \citet[Part III]{V2009}). As of now we are not able to spot a similar connection in \eqref{eq:Hessian}. We believe that a more intuitive (possibly geometric) understanding of \eqref{eq:Hessian} might lead to further algorithmic improvements of SVGD. Finally, the Wasserstein-Hessian has been leveraged in \citet{otto2001geometry} for the analysis of certain functional inequalities central to the understanding of exponential convergence to equilibrium of the overdamped Langevin dynamics \eqref{eq:overdamped Langevin}. We mention in particular the Poincar{\'e} inequality taking the same form as \eqref{eq:Poincare}, but with $\mathcal{L}$ given by
\begin{equation}
\label{eq:overdamped generator}
\mathcal{L} \phi = - \sum_{i=1}^d e^V \partial_i \left(e^{-V} \partial_i \phi\right),
\end{equation}  
i.e. only differing by the operator $T_{k,\pi}$. The viewpoint of \citet{otto2000generalization} led to a geometric understanding of the celebrated Bakry-{\'E}mery criterion \citep[Theorem 2]{otto2000generalization}; we note that our condition \eqref{eq:BK matern} has a similar flavour (albeit in a simplified context). Despite all those similarities, we would like to stress that the Fokker-Planck equation \eqref{eq:Fokker-Planck} governs the law of \eqref{eq:overdamped Langevin} while the Stein PDE \eqref{eq:Stein pde} arises as the mean-field limit for \eqref{eq:Stein ode} and \eqref{eq:noisy Stein}.  This fact makes a direct theoretical comparison between the corresponding algorithms difficult, see Remark \ref{rem:scalings}.
 
\section{Proofs for Section \ref{sec:noisy}}
\label{app:noisy Stein}

\begin{proof}\textbf{of Proposition \ref{prop:mean field limit}} \;\;
	Let $\phi \in C_c^\infty([0,\infty)\times\mathbb{R}^{d})$ be a smooth test function with compact support and define $\Phi \in C_c^\infty([0,\infty) \times \mathbb{R}^{Nd})$ by $\Phi(t,x):=\frac{1}{N}\sum_{i=1}^{N}\phi(t,x_i)$. Using the notation
	\begin{equation*}
	b(x,\rho)  :=\int_{\mathbb{R}^d}\left[-k(x,y)\nabla V(y)+\nabla_{y}k(x,y)\right]\mathrm{d}\rho(y),
	\end{equation*}
	It{\^o}'s formula implies
	\begin{align*}
	\mathrm{d}\Phi(t,\bar{X}_{t}) & =\frac{1}{N}\sum_{i=1}^{N} \left(\partial_{t}\phi(t,X_{t}^i)+\nabla\phi(t,X_{t}^i)\cdot b(X_{t}^i,\rho_{t}^{N})\right) \mathrm{d}t+\Tr \left(\mathcal{K}(\bar{X}_{t})\mathrm{Hess} \, \Phi(\bar{X}_{t})\right) \mathrm{d}t\\
	& \qquad+\frac{\sqrt{2}}{N}\sum_{i,j=1}^{N}\nabla\phi(X_{t}^i) \cdot \sqrt{\mathcal{K}(\bar{X}_{t})}_{ij} \,\mathrm{d}W_{t}^{j}.
	\end{align*}
	The Hessian $\mathrm{Hess} \, \Phi \in \mathbb{R}^{Nd \times Nd}$ consists of $N^2$ blocks of size $d \times d$ with 
	\[
	\left[\mathrm{Hess} \,\Phi(x)\right]_{ij}= \begin{cases}
	\frac{1}{N}\mathrm{Hess} \, \phi (x_i) & \mbox{if }i=j\\
	0 & \mbox{otherwise}
	\end{cases}, \quad (i,j) \in \{1,\ldots,N\}^2,
	\]
	so that it is a block diagonal matrix, with each diagonal block containing
	the Hessian of $\phi$.\\
	\\
	A simple calculation yields that
	\[
	\Tr (\mathcal{K}(x)\mathrm{Hess}\,\Phi(x))=\frac{1}{N^{2}}\sum_{i=1}^{N}k(x_{i},x_{i})\Tr (\mathrm{Hess}\,\phi(x_{i}))=\frac{1}{N^{2}}\sum_{i=1}^{N}k(x_{i},x_{i})\Delta \phi(x_{i}),
	\]
	so that 
	\[
	\Tr \left(\mathcal{K}(\bar{X}_{t})\mathrm{Hess} \, \Phi(\bar{X}_{t})\right)=\frac{1}{N}\int_{\mathbb{R}^d} k(x,x)\Delta\phi(x)\, \mathrm{d}\rho_{t}^{N}(x).
	\]
	It follows that 
	\begin{align*}
	\langle\phi(t,\cdot),\rho_{t}^{N}\rangle-\langle\phi(0,\cdot),\rho_{0}^{N}\rangle & =\int_{0}^{t}\langle\partial_{s}\phi(s,\cdot),\rho_{s}^{N}\rangle\,\mathrm{d}s+\int_{0}^{t}\langle\nabla\phi(s,\cdot) \cdot b(\cdot,\rho_{s}^{N}),\rho_{s}^{N}\rangle\,\mathrm{d}s\\
	& \qquad+\frac{1}{N}\int_{0}^{t}\langle k(\cdot,\cdot)\Delta\phi(\cdot),\rho_{s}^{N}\rangle\,\mathrm{d}s+\mathcal{N}_{t},
	\end{align*}
	where the brackets denote the duality pairing between test functions and measures. 
	The term $\mathcal{N}_{t}$ represents a local martingale with quadratic variation
	\begin{align*}
	[\mathcal{N}_{\cdot},\mathcal{N}_{\cdot}]_{t} & =\frac{2}{N^{2}}\int_{0}^{t}\sum_{i,j=1}^{N}\nabla\phi(X_{s}^i)\cdot\mathcal{K}(\bar{X}_{s})_{ij}\nabla\phi(X_{s}^j)\,\mathrm{d}s\\
	& =\frac{2}{N^{3}}\int_0^t\sum_{i,j=1}^{N}\nabla\phi(X_{s}^i)\cdot\nabla\phi(X_{s}^j)\, k(X_{s}^i,X_{s}^j)\,\mathrm{d}s\\
	& =\frac{2}{N} \int_0^t \int_{\mathbb{R}^d}\int_{\mathbb{R}^d} \nabla\phi(y)\cdot\nabla\phi(z)k(y,z)\, \mathrm{d}\rho_{s}^{N}(y) \, \mathrm{d}\rho_{s}^{N}(z) \, \mathrm{d}s.
	\end{align*}
	In particular, assuming that the family $\lbrace\rho_{\cdot}^{N}\::\:N\in\mathbb{N}\rbrace$
	possesses a limit point in $\mathcal{P}(C[0,T])$, it follows that $[\mathcal{N}_{\cdot},\mathcal{N}_{\cdot}]_{t}\sim O(N^{-1})$
	as $N\rightarrow\infty.$ Let $\rho_{\cdot}$ be a limit point of
	the family $\lbrace\rho_{\cdot}^{N}\::\:N\in\mathbb{N}\rbrace$, then
	formally as $N\rightarrow\infty$ we obtain the following relationship
	for the limiting distribution:
	\[
	\langle\phi(t,\cdot),\rho_{t}\rangle-\langle\phi(0,\cdot),\rho_{0}\rangle=\int_{0}^{t}\langle\partial_{s}\phi(s,\cdot),\rho_{s}\rangle\,\mathrm{d}s+\int_{0}^{t}\langle\nabla\phi(s,\cdot) \cdot b(\cdot,\rho_{s}),\rho_{s}\rangle\,\mathrm{d}s,
	\]
	so that the limit $\rho_{t}=\lim_{N\rightarrow\infty}\rho_{t}^{N}$
	satisfies the nonlinear transport equation
	\[
	\partial_{t}\rho_{t}(t,x)=-\nabla\cdot(b(x,\rho_{t})\rho_{t}),
	\]
	as required.
\end{proof}
\begin{proof}\textbf{of Proposition \ref{prop:ergodicity}}
\;\;For a textbook account of similar proof strategies we refer to \citet{khasminskii2011stochastic}, see also the proof of Theorem 3.1 in \citet{meyn1993stability}.
Let us define the set
\begin{equation}
\nonumber
D := (\mathbb{R}^d)^N \setminus \bigcup_{i \neq j} \left\{ (x_1, \ldots, x_n) \in (\mathbb{R}^d)^N: \quad x_i = x_j \right\}
\end{equation}
and the Lyapunov function
\begin{equation}
\label{eq:Lyapunov}
F(\bar{x}) =  \sum_{\substack{l,m =1 \\ l \neq m}}^N F_{lm}(x_l,x_m), \quad \bar{x} = (x_1\ldots,x_N) \in D,
\end{equation}
with
\begin{equation}
\nonumber
F_{lm}(x_l,x_m) = - \frac{1}{2} \chi(\vert x_l - x_m \vert^2) \log \vert  x_l - x_m  \vert^2.
\end{equation}
Here $\chi \in C_c^\infty(\mathbb{R})$ is assumed to be a fixed nonnegative cutoff function with $ \chi \equiv 1$ on $[0,1]$. We now argue that there exist constants $C_1, C_2 \in \mathbb{R}$ such that 
\begin{equation}
\label{eq:Lyapunov bound}
(\bar{L} F)(\bar{x}) \le C_1 \sum_{i=1}^N \vert \nabla V(x_i) \vert  + C_2, \quad \bar{x} = (x_1, \ldots x_N) \in D, 
\end{equation}
where $\bar{L}$ is the infinitesimal generator\footnote{Here we use the notation $\mathcal{K}: \nabla \nabla \phi = \sum_{ij} \mathcal{K}_{ij} \partial_i \partial_j \phi$.} of \eqref{eq:noisy Stein},
\begin{equation}
\label{eq:generator}
\bar{L} \phi = - \nabla \bar{V} \cdot \mathcal{K} \nabla \phi + (\nabla \cdot  \mathcal{K})  \cdot \nabla \phi + \mathcal{K} : \nabla \nabla \phi, \quad \phi \in \mathcal{D}(\bar{L}).
\end{equation}
For $l \neq m$, we see that
\begin{subequations}
	\label{eq:potential calculation}
	\begin{align}
	\label{eq:Lipschitz1}
	\left(-\nabla \bar{V} \cdot \mathcal{K} \nabla F_{lm} \right)(\bar{x}) = &   -\chi\left(\vert x_l - x_m \vert^2\right) \sum_{i=1}^N  \nabla V(x_i)  \cdot \frac{x_l - x_m}{(x_l - x_m)^2 }\left( h(x_i - x_m) - h(x_i - x_l) \right) \\
	\label{eq:remainder1}
	& + \frac{1}{2} \log \vert x_l - x_m \vert^2 \sum_{i,j=1}^N \nabla V(x_i) \cdot h(x_i - x_j)\nabla_{x_j} \chi \left( \vert x_l - x_m \vert^2 \right) \\
	\label{eq:estimate1}
	& \le \tilde{C}_1 \sum_{i=1}^N \vert \nabla V(x_i) \vert  + \tilde{C}_2,
	\end{align}
\end{subequations}
where here and in what follows $\tilde{C}_1$ and $\tilde {C}_2$ denote generic constants, the value of which can change from line to line. The estimate \eqref{eq:estimate1} follows from the fact that \eqref{eq:remainder1} is bounded (with compact support) by the construction of $\chi$, and by using Lipschitz continuity of $h$ in \eqref{eq:Lipschitz1}.
Similarly, we have that 
\begin{subequations}
	\label{eq:drift calculation}
	\begin{align}
	\left((\nabla \cdot \mathcal{K}) \cdot \nabla F_{lm} \right)(\bar{x})& =  \sum_{i,j=1}^N \nabla_{x_i} h(x_i - x_j) \cdot \nabla_{x_j} F_{lm}(x_l,x_m) 
	\\
	& =  \chi(\vert x_l - x_m \vert^2) \sum_{i=1}^N \left( \nabla_{x_i}(h (x_i - x_m) - h(x_i - x_l) \right) \cdot \frac{x_l - x_m}{(x_l - x_m)^2}
	\\
	& - \log \vert x_l - x_m \vert^2 \chi'(\vert x_l - x_m \vert^2) \sum_{i = 1}^N \nabla_{x_i} \left(  h (x_i - x_l) - h(x_i - x_m)\right) \cdot (x_l - x_m)
 	\end{align}
\end{subequations}
is bounded due to the one-sided Lipschitz bound \eqref{eq:one-sided Lipschitz}. Lastly,
\begin{subequations}
\label{eq:Laplace calculation}
\begin{align}
\left(\mathcal{K} : \nabla \nabla F_{lm} \right)(\bar{x})& = \sum_{i,j=1}^N h(x_i - x_j) \nabla_{x_i} \cdot \nabla_{x_j} F_{lm} 
 = - \sum_{i=1}^N \left( h(x_i - x_l) - h(x_i - x_m) \right) \nabla_{x_i} 
\cdot 
\\
& \cdot 
\left(\log \vert x_l - x_m \vert^2 \chi'\left(\vert x_l - x_m \vert^2 \right)  (x_l - x_m)  + \chi\left(\vert x_l - x_m \vert^2 \right) \frac{x_l - x_m}{(x_l - x_m)^2}\right)
\\
\label{eq:Laplace bound}
& \le \tilde{C} - 2 \chi\left(\vert x_l -x_m \vert^2 \right) \frac{d-2}{(x_l - x_m)^2} \left( h(0) - h(x_m - x_l) \right),
\end{align}
\end{subequations}
where we have again subsumed terms that are bounded by the construction of $\chi$ in the constant $\tilde{C}$. Note that the second term in \eqref{eq:Laplace bound} (including the minus sign) is nonpositive since $h$ is a positive definite function (see, for instance, \citet[Theorem 3.1(4)]{fasshauer2007meshfree}). Collecting \eqref{eq:potential calculation}, \eqref{eq:drift calculation} and \eqref{eq:Laplace calculation}, we indeed arrive at the Lyapunov bound \eqref{eq:Lyapunov bound}.

Now note that $F$ is bounded from below, and so we can choose a constant $c$ such that $\tilde{F} := F + c$ is nonnegative.
By the assumption that the initial condition is distinct, there exists $q_0 \in \mathbb{N}$ such that $\tilde{F}(\bar{X}_0) < q_0$. For $q > q_0$ let us define the stopping times
\begin{equation}
\nonumber
\tau_q = \inf\{ t \ge 0: \quad \tilde{F}(\bar{X}_t) = q  \}.
\end{equation}
By Dynkin's formula in combination with the bound \eqref{eq:Lyapunov bound} and Assumption \ref{ass:ergodicity}, we see that 
\begin{equation}
\label{eq:q independent bound}
\mathbb{E} [\tilde{F} (\bar{X}_{\tau_q \wedge t})] < C_t,
\end{equation}
for all $q > q_0$ and a constant $C_t$ that depends on $t$, but not on $q$. On the other hand, 
\begin{equation}
\label{eq:Chebyshev}
\mathbb{E}[\tilde{F}(\bar{X}_{\tau_q \wedge t})] = \mathbb{E} [\tilde{F}(\bar{X_t})\mathbf{1}_{\{t < \tau_q\}} ] + q \mathbb{P} [t \ge \tau_q] \ge q \mathbb{P} [t \ge \tau_q],
\end{equation}
where we have used the fact that $\tilde{F}$ is nonnegative. This, together with \eqref{eq:q independent bound},  immediately implies $\mathbb{P}[t \ge \xi] = 0$ for all $t \ge 0$, where $\xi := \lim_{q \rightarrow \infty} \tau_q$. Monotone convergence then shows that $\mathbb{P}[\xi = \infty] = 1$.

In other words, we have shown that $\bar{X}_t \in D$ almost surely, for all $t \ge 0$. Since $\mathcal{K}$ is strictly positive definite on $D$, there is an invariant measure with strictly positive Lebesgue density (see \eqref{eq:marginals}) and $D$ is path-connected \citep[Lemma 3.1]{bolley2018dynamics}, it follows that the process is irreducible and hence ergodic with unique invariant measure \eqref{eq:marginals}, see \citet{kliemann1987recurrence}.    
\end{proof}

\section{Proofs for Section \ref{sec:gradient flow}}
\label{app:geometry proofs}
Let us begin with the following auxiliary lemma:
\begin{lemma}
	\label{lem:orthogonality}
	Let $\rho \in \mathcal{P}_k(\mathbb{R}^d)$.
	Then $\overline{\mathcal{T}_{k,\rho}\nabla C_c^\infty(\mathbb{R}^d)}^{\mathcal{H}_k^d}$ is the orthogonal complement of $L^2_{\mathrm{div}}(\rho) \cap \mathcal{H}_k^d$ in $\mathcal{H}^d_k$, where $L^2_{\mathrm{div}}(\rho)$ is the space of weighted divergence-free vector fields, i.e.  
	\begin{equation}
	\label{eq:L2div}
	L^2_{\mathrm{div}}(\rho)  = \left\{ v \in (L^2(\rho))^d:   \langle v, \nabla \phi \rangle_{L^2(\rho)} = 0 \quad \text{for all } \phi \in C_c^\infty(\mathbb{R}^d) \right\}.
	\end{equation}
	Moreover, $L^2_{\mathrm{div}}(\rho) \cap \mathcal{H}_k^d$ is closed in $\mathcal{H}_k^d$.
\end{lemma}
\begin{proof}\textbf{of Proposition \ref{prop:helmholtz}}
We begin by showing that $\overline{\mathcal{T}_{k,\rho}\nabla C_c^\infty(\mathbb{R}^d)}^{\mathcal{H}_k^d}$ is the orthogonal complement of $L^2_{\mathrm{div}}(\rho) \cap \mathcal{H}_k^d$ in $\mathcal{H}^d_k$. Indeed,
	using the relation $(U^\perp)^\perp = \overline{U}$ valid for arbitrary linear subspaces of Hilbert spaces, it is enough to show that 
	\begin{equation}
	\label{eq:perp}
	\left( \mathcal{T}_{k,\rho} \nabla C_c^\infty(\mathbb{R}^d)  \right)^{\perp_{\mathcal{H}_k^d}} = L^2_{\mathrm{div}}(\rho) \cap \mathcal{H}_k^d.
	\end{equation}
	By \citet[Theorem 4.26]{steinwart2008support}, we have that $\mathcal{T}_{k,\rho}$ is the adjoint of the inclusion $\mathcal{H}_k^d \hookrightarrow (L^2(\rho))^d$, implying
	\begin{equation}
	\label{eq:L2 H iso}
	\langle v, \nabla \phi \rangle_{L^2(\rho)} = \langle v, \mathcal{T}_{k,\rho} \nabla \phi \rangle_{\mathcal{H}^d_k},
	\end{equation}
	for all $v \in \mathcal{H}_k^d$ and $\phi \in C_c^\infty(\mathbb{R}^d)$. This proves \eqref{eq:perp} and thus the orthogonality statement follows. 
	We next show that $L^2_{\mathrm{div}}(\rho) \cap \mathcal{H}_k^d$ is closed in $\mathcal{H}_k^d$.
	For that, let $(v_n) \subset L^2_{\mathrm{div}}(\rho) \cap \mathcal{H}_k^d$ with $v_n \rightarrow v$ in $\mathcal{H}_k^d$. Using \eqref{eq:L2 H iso} we see that $v \in L^2_{\mathrm{div}}(\rho)$, implying that $L^2_{\mathrm{div}}(\rho) \cap \mathcal{H}_k^d$ is closed. The statement of Proposition \ref{prop:helmholtz} now follows from Theorem II.3 in  \citet{reed1972methods}.
\end{proof}
We now turn to the proof of Lemma \ref{lem:Riemann properties}.
\begin{proof}\textbf{of Lemma \ref{lem:Riemann properties}}\;\;
We only prove the second claim, as it immediately implies the first one. Assume that for $\xi \in T_\rho M$ there exist $v,w \in \overline{\mathcal{T}_{k,\rho}\nabla C_c^\infty(\mathbb{R}^d)}^{\mathcal{H}_k^d}$ such that 
\begin{equation}
\nonumber
\xi + \nabla \cdot (\rho v) = \xi + \nabla \cdot( \rho w) = 0 
\end{equation}
in the sense of distributions. It follows immediately that 
\begin{equation}
\nonumber
\int_{\mathbb{R}^d} \nabla \phi \cdot (v - w) \, \mathrm{d}\rho = 0, 
\end{equation}
for all $\phi \in C_c^\infty(\mathbb{R}^d)$, i.e. $v-w \in L^2_{\mathrm{div}}(\rho)$. Since $ \overline{\mathcal{T}_{k,\rho}\nabla C_c^\infty(\mathbb{R}^d)}^{\mathcal{H}_k^d} \cap L^2_{\mathrm{div}(\rho)} = \{0\}$ by Lemma \ref{lem:orthogonality} and $v - w \in  \overline{\mathcal{T}_{k,\rho}\nabla C_c^\infty(\mathbb{R}^d)}^{\mathcal{H}_k^d}$, we conclude that $v=w$.
Consequently, the map $v \mapsto \nabla (\rho v)$ is a bijection. The fact that it is also an isometry follows directly from the definition of $g_\rho$. 
\end{proof}
\begin{proof}\textbf{of Lemma \ref{lem:gradient}}\;\;
	By definition, the Riemannian gradient $\mathrm{grad}_\rho \mathcal{F} \in T_\rho M$ is determined by the requirement that
	\begin{equation}
	\label{eq:def gradient}
	g_\rho \left( \mathrm{grad}_\rho \mathcal{F}, \partial_t \mu_t \Big \vert_{t=0} \right) = \frac{\mathrm{d}}{\mathrm{d}t} \mathcal{F}(\mu_t)\Big \vert_{t=0},
	\end{equation}
	for all sufficiently regular curves $(\mu_t)_{t \in (-\varepsilon,\varepsilon)} \subset M$ with $\mu_0 = \rho$ and $\partial_t \mu_t \Big \vert_{t=0}  \in T_\rho M$. Given such a curve and corresponding vector fields $(w_t)_{t \in (-\varepsilon,\varepsilon)}$ satisfying $\partial_t \mu + \nabla \cdot (\mu w) = 0$ in the sense of distributions, we compute the right-hand side of \eqref{eq:def gradient},
\begin{equation}
\nonumber
\frac{\mathrm{d}}{\mathrm{d}t} \mathcal{F}(\mu_t)\Big \vert_{t=0} = \int_{\mathbb{R}^d} \frac{\delta \mathcal{F}}{\delta \mu}(\mu_t) \, \partial_t \mu_t \, \mathrm{d}x \Big \vert_{t = 0}
= \int_{\mathbb{R}^d}  \nabla \frac{\delta \mathcal{F}}{\delta \rho}(\rho) \cdot w_0 \, \mathrm{d}\rho.
\end{equation}
From the definition of $T_\rho M$, we have that $\partial_t \mu_t \Big \vert_{t=0}  \in T_\rho M$ implies $w_0 \in \mathcal{H}_k^d$. Therefore, using \citet[Theorem 4.26]{steinwart2008support}, we can write
\begin{equation}
\label{eq:gradient LHS}
\int_{\mathbb{R}^d}  \nabla \frac{\delta \mathcal{F}}{\delta \rho}(\rho) \cdot w_0 \, \mathrm{d}\rho = \left\langle \mathcal{T}_{k,\rho} \nabla \frac{\delta \mathcal{F}}{\delta \rho} (\rho) ,w_0 \right\rangle_{\mathcal{H}_k^d}.
\end{equation}
From the definition of $g_\rho$, the left-hand side of \eqref{eq:def gradient} can be expressed as
\begin{equation}
\label{eq:gradient RHS}
g_\rho \left( \mathrm{grad}_\rho \mathcal{F}, \partial_t \mu_t \Big \vert_{t=0} \right) = \langle v, w_0 \rangle_{\mathcal{H}_k^d}, 
\end{equation} 
where $\mathrm{grad}_\rho \mathcal{F} + \nabla \cdot (\rho v) = 0$, $v \in \overline{\mathcal{T}_{k,\rho} \nabla C_c^\infty(\mathbb{R}^d)}$. Now imposing equality of \eqref{eq:gradient LHS} and \eqref{eq:gradient RHS} for all $w_0 \in \overline{\mathcal{T}_{k,\rho} \nabla C_c^\infty(\mathbb{R}^d)}$ leads to the desired result. 
\end{proof}

\begin{proof}\textbf{of Lemma \ref{lem:properties distance}}
	
1.) We recall that metrics by definition satisfy the axioms
\begin{subequations}
	\begin{align}
	\label{eq:nonnegativ}
	d_k(\mu_1,\mu_2) \ge 0, \quad \quad &\text{(nonnegativity)} \\
	\label{eq:symmetry}
	d_k(\mu_1,\mu_2) = d_k(\mu_2,\mu_1), \quad \quad  &\text{(symmetry)} \\
	\label{eq:nondegenerate}
	d_k(\mu_1,\mu_2) = 0 \quad \iff \quad \mu_1 = \mu_2,  \quad \quad &\text{(nondegeneracy)}\\
	\label{eq:triangle}
	d_k(\mu_1,\mu_2) + d_k(\mu_2,\mu_3) \le d_k(\mu_1,\mu_3), \quad \quad &\text{(triangle inequality)}
	\end{align}
\end{subequations}
for $\mu_1, \mu_2, \mu_3 \in M$.
The properties \eqref{eq:nonnegativ} and \eqref{eq:nondegenerate} follow directly from the definition of $d_k$. For \eqref{eq:symmetry} note that $(\rho_t,v_t)_{t \in [0,1]} \in \mathcal{A}(\mu,\nu)$ if and only if   $(\rho_{1-t},-v_{1-t})_{t \in [0,1] } \in \mathcal{A}(\nu,\mu)$ as well as $v \in \overline{\mathcal{T}_{k,\rho} \nabla C_c^\infty(\mathbb{R}^d)}$ if and only if $-v \in \overline{\mathcal{T}_{k,\rho} \nabla C_c^\infty(\mathbb{R}^d)}$. The triangle inequality \eqref{eq:triangle} follows from considering concatenated paths from $\mu_1$ to $\mu_3$ via $\mu_2$. 

2.) From \citet[Theorem 4.26]{steinwart2008support} we have that
\begin{equation}
\nonumber
\Vert v \Vert_{L^2(\rho)} \le \int_{\mathbb{R}^d} k(x,x) \, \mathrm{d}\rho(x) \, \Vert v \Vert_{\mathcal{H}_k^d}, \quad v \in \mathcal{H}_k^d.
\end{equation} 
The claim now follows directly from the Benamou-Brenier formula for the quadratic Wasserstein distance, see \citet{benamou2000computational}, together with Lemma \ref{lem:properties distance}.\ref{it:unconstrained}.

3.) For fixed $\mu, \nu \in M$, consider a connecting curve $(\rho,v) \in \mathcal{A}(\mu,\nu)$. According to Lemma \ref{lem:orthogonality} we have the $\mathcal{H}_k^d$-orthogonal decompositions 
\begin{equation}
\label{eq:H_k decomp}
\mathcal{H}_k^d = \overline{\mathcal{T}_{k,\rho_t} \nabla C_c^\infty(\mathbb{R}^d)}^{\mathcal{H}_k^d} \oplus \left( L^2_{\mathrm{div}}(\rho_t) \cap \mathcal{H}_k^d\right), \quad t \in [0,1],
\end{equation} 
i.e. we can write 
\begin{equation}
\nonumber
v_t = u_t + w_t, \quad u_t \in  \overline{\mathcal{T}_{k,\rho_t} \nabla C_c^\infty(\mathbb{R}^d)}^{\mathcal{H}_k^d} , \quad w_t \in L^2_{\mathrm{div}}(\rho_t) \cap \mathcal{H}_k^d, 
\end{equation}
with $(u_t)_{t \in [0,1]}$ and $(w_t)_{t \in [0,1]}$ being uniquely determined. Since $w_t \in L^2_{\mathrm{div}}(\rho_t)$ for all $t \in [0,1]$, we have that $v$ satisfies the continuity equation \eqref{eq:weak continuity equation} if and only if $u$ does. By $\mathcal{H}_k^d$-orthogonality in \eqref{eq:H_k decomp}, we moreover have
\begin{equation}
\label{eq:norm decomposition}
\Vert v_t \Vert^2_{\mathcal{H}_k^d}  = \Vert u_t \Vert_{\mathcal{H}_k^d}^2 + \Vert w_t \Vert_{\mathcal{H}_k^d}^2, \quad t \in [0,1]. 
\end{equation}
Because \eqref{eq:norm decomposition} is optimised for $w_t = 0$ while keeping the continuity equation unchanged, it is clear that the objective in \eqref{eq:Stein unconstrained} enforces $w_t = 0$, or, equivalently, $v_t \in  \overline{\mathcal{T}_{k,\rho_t} \nabla C_c^\infty(\mathbb{R}^d)}^{\mathcal{H}_k^d}$, for all $t \in [0,1]$.
\end{proof}
\section{Proofs for Section \ref{sec:2nd order}}
\label{app:proofs geodesics}
\begin{proof}\textbf{of Proposition \ref{prop:geodesic equations}}\;\;
	The arguments are formal and proceed along the lines of \citet[Section 3]{otto2000generalization}.
	In \eqref{eq:Stein constrained} let us substitute $v_t = \mathcal{T}_{k,\rho_t} \nabla \Phi_t$ with $\Phi_t \in C_c^\infty(\mathbb{R}^d)$, $t \in [0,1]$, to obtain
	\begin{equation}
	\label{eq:d_k reformulation}
	d_k^2(\mu,\nu) = \inf_{(\rho,\Phi)} \left\{ \int_0^1 \Vert \mathcal{T}_{k,\rho_t} \nabla \Phi_t \Vert_{\mathcal{H}_k^d}^2 \, \mathrm{d}t: \quad \partial_t \rho + \nabla \cdot(\rho \mathcal{T}_{k,\rho} \nabla \Phi) =0, \quad \rho_0 = \mu, \quad \rho_1 = \nu  \right\},
	\end{equation}
	where the continuity equation is as usual interpreted in a weak sense, i.e. the pair $(\rho,\Phi)$ satisfies the constraints in  \eqref{eq:d_k reformulation} if and only if
	\begin{equation}
	\label{eq:weak cont}
	- \int _0^1 \int_{\mathbb{R}^d} \partial_t \Psi \, \mathrm{d}\rho \,\mathrm{d}t - \int_0^1 \langle \nabla \Psi, \mathcal{T}_{k,\rho}\nabla \Phi \rangle_{L^2(\rho)} \, \mathrm{d}t + \int_{\mathbb{R}^d} \Psi_1 \, \mathrm{d}\nu - \int_{\mathbb{R}^d} \Psi_0 \, \mathrm{d}\mu = 0,
	\end{equation}
	for all test functions $\Psi \in C_c^\infty([0,1]\times \mathbb{R}^d)$.
	Let us now define the following functional on pairs $(\rho,\Phi)$,
	\begin{equation}
	\mathcal{E}(\rho,\Phi) := \sup_{\Psi} \left\{- \int _0^1 \int_{\mathbb{R}^d} \partial_t \Psi \, \mathrm{d}\rho \,\mathrm{d}t - \int_0^1 \langle \nabla \Psi, \mathcal{T}_{k,\rho}\nabla \Phi \rangle_{L^2(\rho)} \, \mathrm{d}t + \int_{\mathbb{R}^d} \Psi_1 \, \mathrm{d}\nu - \int_{\mathbb{R}^d} \Psi_0 \, \mathrm{d}\mu \right\},
	\nonumber
	\end{equation}
	where the supremum is taken over all $\Psi \in C_c^\infty([0,1]\times \mathbb{R}^d)$. 
	Since the expression inside the supremum is linear in $\Psi$, it follows that $\mathcal{E}$ characterises weak solutions in the sense of \eqref{eq:weak cont} in the following way,
	\begin{subequations}
		\nonumber
		\begin{align}
		\mathcal{E}(\rho,\Phi) = 
		\begin{cases}
		0 & \text{if } (\rho,\Phi)  \, \text{solves \eqref{eq:weak cont},} \\
		+\infty & \text{otherwise.}
		\end{cases}
		\end{align}
	\end{subequations}
	We can therefore write
	\begin{subequations}
		\begin{align}
		\frac{1}{2} d_k^2(\mu,\nu) & = \inf_{(\rho,\Phi)} \sup_{\Psi}  \left\{  \frac{1}{2}\int_0^1 \Vert \mathcal{T}_{k,\rho_t} \nabla \Phi_t \Vert_{\mathcal{H}_k^d}^2 \, \mathrm{d}t + \mathcal{E}(\rho, \Phi)\right\} \\
		\label{eq:saddle}
		& = \inf_{(\rho,\Phi)} \sup_{\Psi} \Bigg\{  \frac{1}{2}\int_0^1 \Vert \mathcal{T}_{k,\rho_t} \nabla \Phi_t \Vert_{\mathcal{H}_k^d}^2 \, \mathrm{d}t
		\\
		\label{eq:saddle2}
		& - \int _0^1 \int_{\mathbb{R}^d} \partial_t \Psi \, \mathrm{d}\rho \,\mathrm{d}t  - \int_0^1 \int_0^1 \langle \nabla \Psi, \mathcal{T}_{k,\rho}\nabla \Phi \rangle_{L^2(\rho)} \, \mathrm{d}t + \int_{\mathbb{R}^d} \Psi_1 \, \mathrm{d}\nu - \int_{\mathbb{R}^d} \Psi_0 \, \mathrm{d}\mu \Bigg\}. 
		\end{align}
	\end{subequations}
	The term in brackets in \eqref{eq:saddle}-\eqref{eq:saddle2} is convex in $\Phi$ and concave (in fact, linear) in $\Psi$. Hence, it is justified to exchange infimum and supremum (see \citet{rockafellar1970convex},\citet[Section 1.1.6]{V2003}) to obtain
	\begin{subequations}
		\label{eq:d_k reform}
		\begin{align}
		\frac{1}{2} d_k^2(\mu,\nu) &  = \inf_{\rho} \sup_{\Psi} \Bigg\{  - \int _0^1 \int_{\mathbb{R}^d} \partial_t \Psi \, \mathrm{d}\rho \,\mathrm{d}t
		+ \int_{\mathbb{R}^d} \Psi_1 \, \mathrm{d}\nu - \int_{\mathbb{R}^d} \Psi_0 \, \mathrm{d}\mu
		\\
		\label{eq:opt Phi}
		& + \inf_{\Phi}  \left\{\frac{1}{2}\int_0^1 \left(\Vert \mathcal{T}_{k,\rho_t} \nabla \Phi_t \Vert_{\mathcal{H}_k^d}^2 \, \mathrm{d}t - \langle \nabla \Psi, \mathcal{T}_{k,\rho_t}\nabla \Phi \rangle_{L^2(\rho_t)} \right) \mathrm{d}t  \right\}\Bigg\}. 
		\end{align}
	\end{subequations}
	Using the fact that $\mathcal{T}_{k,\rho}$ is self-adjoint in $L^2(\rho)$ and that $\mathcal{T}_{k,\rho}^{1/2}:L^2(\rho) \rightarrow \mathcal{H}_k$ is an isometry \citep[Section 4.3]{steinwart2008support}, we see that
	\begin{equation}
	\label{eq:L2 H}
	\langle \nabla \Psi, \mathcal{T}_{k,\rho_t}\nabla \Phi \rangle_{L^2(\rho_t)} = \langle \mathcal{T}_{k,\rho_t}^{1/2}\nabla \Psi, \mathcal{T}_{k,\rho_t}^{1/2}\nabla \Phi \rangle_{L^2(\rho_t)} = \langle \mathcal{T}_{k,\rho_t} \nabla \Psi, \mathcal{T}_{k,\rho_t} \nabla \Phi \rangle_{\mathcal{H}_k^d}.
	\end{equation}
	Substituting into \eqref{eq:opt Phi}, it follows that
	\begin{equation}
 \nonumber
	\arginf_{\Phi}  \left\{\frac{1}{2}\int_0^1 \left(\Vert \mathcal{T}_{k,\rho_t} \nabla \Phi_t \Vert_{\mathcal{H}_k^d}^2 \, \mathrm{d}t - \langle \nabla \Psi, \mathcal{T}_{k,\rho_t}\nabla \Phi \rangle_{L^2(\rho_t)} \mathrm{d}t \right)  \right\} = \Psi,
	\end{equation}
	up to an additive constant, i.e. 
	\begin{equation}
 \nonumber
\inf_{\Phi}  \left\{\frac{1}{2}\int_0^1 \left(\Vert \mathcal{T}_{k,\rho_t} \nabla \Phi_t \Vert_{\mathcal{H}_k^d}^2 \, \mathrm{d}t - \langle \nabla \Psi, \mathcal{T}_{k,\rho_t}\nabla \Phi \rangle_{L^2(\rho_t)} \right) \mathrm{d}t  \right\} = - \frac{1}{2}\int_0^1 \Vert \mathcal{T}_{k,\rho_t} \nabla \Psi_t \Vert_{\mathcal{H}_k^d}^2  \mathrm{d}t.
\end{equation}
Using \eqref{eq:L2 H}, we obtain the expression
\begin{equation}
\nonumber
\frac{1}{2}\Vert \mathcal{T}_{k,\rho} \nabla \Psi \Vert_{\mathcal{H}_k^d}^2 = \frac{1}{2}\int_{\mathbb{R}^d} \int_{\mathbb{R}^d} \nabla \Psi(x) k(x,y) \nabla \Psi(y) \, \mathrm{d}\rho(x) \mathrm{d}\rho(y).
\end{equation}
Therefore, formally, we can compute the functional derivatives (see \eqref{eq:functional derivative}),
	\begin{subequations}
		\label{eq:functional derivatives}
		\nonumber 
		\begin{align}
		\frac{\delta}{\delta \rho} \left(\frac{1}{2}\Vert \mathcal{T}_{k,\rho} \nabla \Psi \Vert_{\mathcal{H}_k^d}^2\right) (x) & = \int_{\mathbb{R}^d} k(x,y) \, \nabla \Psi(x)\cdot  \nabla \Psi(y) \,  \mathrm{d} \rho(y)  = \nabla \Psi(x) \cdot (\mathcal{T}_{k,\rho} \nabla \Psi)(x), \\
		\frac{\delta}{\delta \Psi}\left(\frac{1}{2} \Vert \mathcal{T}_{k,\rho} \nabla \Psi \Vert_{\mathcal{H}_k^d}^2\right)(x) & = \nabla_x \cdot \left( \rho(x) \int_{\mathbb{R}^d} k(x,y)  \nabla \Psi(y)  \, \mathrm{d} \rho(y) \right) = \nabla \cdot  \left( \rho \mathcal{T}_{k,\rho} \nabla \Psi \right)(x). 
		\end{align}
	\end{subequations}
	The formal optimality conditions for \eqref{eq:d_k reform} are therefore given by the system \eqref{eq:Stein geodesics}.
\end{proof}

\begin{proof}\textbf{of Lemma \ref{lem:hessian negative}}\;\;
	Dealing first with \eqref{eq:Reg nonequilirbrium2}-\eqref{eq:Reg equilibrium} and noting $\nabla \Psi = a$, we observe that 
\begin{equation}
\nonumber
\sum_{i,j=1}^d\int_{\mathbb{R}^d} \int_{\mathbb{R}^d} \int_{\mathbb{R}^d} a_i a_j {\left( \partial_{x_i} \partial_{x_j} k(x,y)  \right) \left( k(y,z) - k(x,z) \right)} \, \mathrm{d}\rho(x)\mathrm{d}\rho(y)\mathrm{d}\rho(z) = 0,
\end{equation}
since $\partial_{x_i} \partial_{x_j} k(x,y) =  \partial_{x_i} \partial_{x_j} k(y,x)$ and $\left( k(y,z) - k(x,z) \right) = -\left( k(x,z) - k(y,z) \right)$. We hence obtain
\begin{subequations}
\nonumber
\begin{align}
\mathrm{Hess}_\rho( \Psi, \Psi) & = \sum_{i,l=1}^d a_i^2 \int_{\mathbb{R}^d} \int_{\mathbb{R}^d} \int_{\mathbb{R}^d}  \partial_{x_l} k(x,y) \partial_{y_l} k(y,z) \, \mathrm{d}\rho(x) \mathrm{d}\rho(y) \mathrm{d}\rho(z) \\
& = -\sum_{i,l=1}^d a_i^2 \int_{\mathbb{R}^d} \left( \int_{\mathbb{R}^d} k(x,y) \partial_{x_l} \rho(x) \, \mathrm{d}x \right)^2 \mathrm{d}\rho(y) < 0.
\end{align}
\end{subequations}
The inequality is strict since $k$ is assumed to be integrally strictly positive definite, and the density $\rho$ cannot be constant.
\end{proof}

\section{Proof of Lemma \ref{lem:Hessian}}
\label{app:proof Hessian}
The proof proceeds by direct calculation, using the geodesic equations \eqref{eq:Stein geodesics}. For convenience, let us introduce the notation
\begin{equation}
\label{eq:w}
w = \mathcal{T}_{k,\rho} \nabla \Psi.
\end{equation}
The following lemma will come in handy.
\begin{lemma}
	Let $\rho$ and $\Psi$ be smooth solutions to \eqref{eq:Stein geodesics}. Then
	\begin{equation}
	\label{eq:wt}
	\partial_t w_i = - \sum_{j=1}^d\int_{\mathbb{R}^d} k(\cdot,y) \partial_j \Psi(y) \partial_i w_j(y) \, \mathrm{d}\rho(y) - \sum_{j=1}^d \int_{\mathbb{R}^d}  k(\cdot,y) \partial_j \left( \partial_i \Psi(y) w_j(y) \rho(y) \right) \mathrm{d}y,
\end{equation}
	for $i=1,\ldots,d$.
\end{lemma}
\begin{proof}
	By direct calculation, we obtain
	\begin{subequations}
		\begin{align}
		\partial_t w_i  & = \int_{\mathbb{R}^d} k(\cdot, y) \left[ \partial_i (\partial_t \Psi) \right](y)\,  \mathrm{d}\rho(y) + \int_{\mathbb{R}^d} k(\cdot, y) \left[ \partial_i \Psi \partial_t \rho \right](y) \, \mathrm{d}y 
		\\
		& = - \sum_{j=1}^d \int_{\mathbb{R}^d} k(\cdot, y) \left[ \partial_i ((\partial_j \Psi) w_j)\right](y) \, \mathrm{d}\rho(y) - \sum_{j=1}^d \int_{\mathbb{R}^d} k(\cdot, y) \left[ \partial_i \Psi(y) \partial_j (\rho w_j) \right](y)\, \mathrm{d}y  \\
		& =- \sum_{j=1}^d \int_{\mathbb{R}^d} k(\cdot,y) \partial_j \Psi(y) \partial_i w_j(y) \, \mathrm{d}\rho(y) - \sum_{j=1}^d \int_{\mathbb{R}^d}  k(\cdot,y) \partial_j \left( \partial_i \Psi(y) w_j(y) \rho(y) \right)  \mathrm{d}y.
		\label{eq:w calculation}
		\end{align}
	\end{subequations}
Note that in the last line we have used the fact that the term involving $\partial_i \partial_j \Psi$ cancels. 
\end{proof}
\noindent We will work under the assumption that $k$ is smooth. Note that we make this restriction for simplicity only such that all expressions can be written in compact form. The results extend without difficulty to the general case by either interpreting the relevant terms in the sense of distributions or by performing integration parts, shifting the derivatives to $\rho$ and $\Psi$ (asssumed to be smooth). See also Remark \ref{rem:rough k}.

Recall the decomposition \eqref{eq:KL}. In what follows, we compute the contributions from the terms $\mathrm{Reg}(\rho)$ and $\mathrm{Cost}(\rho \vert \pi)$ separately (see Lemmas \ref{lem:Hess Reg} and \ref{lem:Hess cost} below) and gather everything at the end of the section.

\begin{lemma}[Hessian of $\mathrm{Reg}(\rho)$]
	\label{lem:Hess Reg}
	Let $(\rho_t,\Psi_t)_{t \in (-\varepsilon,\varepsilon)}$ be a Stein geodesic, i.e. a smooth solution to \eqref{eq:Stein geodesics}, and $\rho_0 \equiv \rho$, $\Psi_0 \equiv \Psi$.   Then
	\begin{equation}
	\partial_t^2 \mathrm{Reg}(\rho_t) \big\vert_{t = 0}  = \Hess^{\mathrm{Reg}}_\rho(\Psi, \Psi),
	\end{equation}
	where
	\begin{equation}
	\label{eq:HessReg}
	\Hess^{\mathrm{Reg}}_\rho(\Phi,\Psi) = \sum_{i,j = 1}^d \int_{\mathbb{R}^d} \int_{\mathbb{R}^d} \partial_i \Phi(y) q^{\mathrm{Reg}}_{ij} [\rho] (y,z) \partial_j \Psi(z) \, \mathrm{d}\rho(y) \mathrm{d}\rho(z),
	\end{equation}
	and
	\begin{subequations}
		\label{eq:qReg}
		\begin{align}
		\label{eq:Reg nonequilibrium1}
		q^{\mathrm{Reg}}_{ij}[\rho](y,z) & =\delta_{ij} \sum_{l=1}^d \int_{\mathbb{R}^d}  \partial_{x_l} k(x,y) \mathrm{d}\rho(x)  \, \partial_{y_l} k(y,z)
		\\
		\label{eq:Reg nonequilirbrium2}
		& - k(y,z)  \int_{\mathbb{R}^d} \left( \partial_{x_i} \partial_{y_j} k(x,y) \right) \mathrm{d}\rho(x)
		\\
		\label{eq:Reg equilibrium}
		& -  \int_{\mathbb{R}^d} \left( \partial_{x_i} \partial_{x_j} k(x,y) \right) k(x,z) \,   \mathrm{d}\rho(x), \qquad \qquad i,j =1 ,\ldots, d.
		\end{align}
	\end{subequations}
\end{lemma}
\begin{proof}
	We have
	\begin{equation}
	\label{eq:R'}
	\partial_t \mathrm{Reg}(\rho) = \int_{\mathbb{R}^d} \partial_t \rho \log \rho \, \mathrm{d}x +  \underbrace{ \partial_t \int_{\mathbb{R}^d} \,\mathrm{d}\rho}_{=0},
	\end{equation}
	where the second term vanishes due to the conservation of total probability.
	Inserting \eqref{eq:Stein geodesics} into \eqref{eq:R'}, we arrive at
\begin{equation}
\nonumber
\partial_t \mathrm{Reg}(\rho) = - \sum_{i=1}^d \int_{\mathbb{R}^d} \partial_i  (\rho w_i) \log \rho \, \mathrm{d}x = \sum_{i=1}^d \int_{\mathbb{R}^d} w_i  \partial_i \rho\, \mathrm{d}x = -\sum_{i=1}^d\int_{\mathbb{R}^d} (\partial_i  w_i) \, \mathrm{d}\rho.
\end{equation}
For the second derivative we obtain
	\begin{subequations}
		\begin{align}
		\partial_t^2 \mathrm{Reg}(\rho) & = - \sum_{i=1}^d \int_{\mathbb{R}^d} \partial_i  (\partial_t w_i) \, \mathbb{d}\rho - \sum_{i=1}^d \int_{\mathrm{R}^d} (\partial_i  w_i) \partial_t \rho \, \mathrm{d}x
		\\
		& = - \sum_{i=1}^d \int_{\mathbb{R}^d} \partial_i  (\partial_t w_i) \, \mathrm{d}\rho + \sum_{i,j=1}^d \int_{\mathbb{R}^d} (\partial_i  w_i) \partial_j( \rho w_j) \, \mathrm{d}x
		\\
		\label{eq:H''}
		& =  - \sum_{i=1}^d \int_{\mathbb{R}^d} \partial_i (\partial_t w_i) \, \mathrm{d}\rho - \sum_{i,j = 1}^d \int_{\mathbb{R}^d} (\partial_i \partial_j w_i)   w_j \, \mathrm{d}\rho
		\end{align}
	\end{subequations}
	We now substitute \eqref{eq:w} and \eqref{eq:wt} into \eqref{eq:H''} to get
\begin{subequations}
\nonumber
\begin{align}
\partial_t^2 \mathrm{Reg}(\rho) & = \sum_{i,j=1}^d \int_{\mathbb{R}^d} \int_{\mathbb{R}^d}\int_{\mathbb{R}^d} \partial_{x_i} k(x,y) \partial_j \Psi(y) \partial_{y_i} k (y,z) \partial_j \Psi(z) \, \mathrm{d} \rho(x) \mathrm{d} \rho(y) \mathrm{d} \rho(z)
\\
& - \sum_{i,j=1}^d \int_{\mathbb{R}^d} \int_{\mathbb{R}^d}\int_{\mathbb{R}^d} \partial_{x_i} \partial_{y_j} k(x,y) \partial_i \Psi(y) k(y,z) \partial_j \Psi(z) \, \mathrm{d} \rho(x) \mathrm{d} \rho(y) \mathrm{d} \rho(z)
\\
	& - \sum_{i,j=1}^d \int_{\mathbb{R}^d} \int_{\mathbb{R}^d}\int_{\mathbb{R}^d} \partial_{x_i} \partial_{x_j} k(x,y) \partial_i \Psi(y) k(x,z) \partial_j \Psi(z) \, \mathrm{d} \rho(x) \mathrm{d} \rho(y) \mathrm{d} \rho(z),
		\end{align}
	\end{subequations}
which can be written in the form \eqref{eq:HessReg}-\eqref{eq:qReg}.
\end{proof}

\begin{lemma}[Hessian of $\mathrm{Cost}(\rho\vert \pi)$]
	\label{lem:Hess cost}
	Let $(\rho_t,\Psi_t)_{t \in (-\varepsilon,\varepsilon)}$ be a Stein geodesic, i.e. a smooth solution to \eqref{eq:Stein geodesics}, and $\rho_0 \equiv \rho$, $\Psi_0 \equiv \Psi$.  Then
\begin{equation}
\nonumber
\partial_t^2 \mathrm{Cost}(\rho_t \vert \pi) \big\vert_{t = 0}  = \Hess^{\mathrm{Cost}}_\rho(\Psi, \Psi),
\end{equation}
where
\begin{equation}
\nonumber
\Hess^{\mathrm{Cost}}_\rho(\Phi,\Psi) = \sum_{i,j = 1}^d \int_{\mathbb{R}^d} \int_{\mathbb{R}^d} \partial_i \Phi(y) q^{\mathrm{Cost}}_{ij} [\rho] (y,z) \partial_j \Psi(z) \, \mathrm{d}\rho(y) \mathrm{d}\rho(z),
\end{equation}
and
\begin{subequations}
\begin{align}
\label{eq:Cost nonequilibrium1}
q^{\mathrm{Cost}}_{ij}[\rho](y,z) &=  - \delta_{ij} \sum_{l=1}^d \int_{\mathbb{R}^d} \partial_l V(x) \left( k(x,y) \partial_{y_l}k(y,z) \right) \mathrm{d}\rho(x)
\\
\label{eq:Cost nonequilibrium2} 
		& +  \int_{\mathbb{R}^d} \left( \partial_i V(x) \partial_{y_j} k(x,y) k(y,z)   \right) \mathrm{d}\rho(x)
		\\
		\label{eq:Cost equilibrium}
	 & + \int_{\mathbb{R}^d} \partial_i \partial_j V(x) k(x,y) k(x,z) \, \mathrm{d}\rho(x)
		 + \int_{\mathbb{R}^d} \left( \partial_i V(x) \partial_{x_j} k(x,y) k(x,z) \right) \mathrm{d}\rho(x),
		\end{align}
	\end{subequations}
	for $i,j = 1,\ldots,d$.
\end{lemma}
\begin{proof}
	Proceeding as in the proof of Lemma \ref{lem:Hess Reg}, we obtain
	\begin{equation}
	\partial_t \mathrm{Cost}(\rho \vert \pi) = \int_{\mathbb{R}^d} V \partial_t \rho \, \mathrm{d}x = - \sum_{i=1}^d \int_{\mathbb{R}^d} V(\partial_i (\rho w_i))\, \mathrm{d}x = \sum_{i=1}^d \int_{\mathbb{R}^d} \partial_i V w_i \, \mathrm{d} \rho
	\end{equation}
	and 
\begin{subequations}
 \nonumber
\begin{align}
\partial_t^2 \mathrm{Cost}(\rho \vert \pi) & = \sum_{i=1}^d\int_{\mathbb{R}^d} \partial_i V \partial_t w_i \, \mathrm{d}\rho + \sum_{i=1}^d \int_{\mathbb{R}^d} \partial_i V w_i \partial_t \rho \, \mathrm{d}x \\
& = \sum_{i=1}^d\int_{\mathbb{R}^d} \partial_i V \partial_t w_i \, \mathrm{d}\rho - \sum_{i,j=1}^d \int_{\mathbb{R}^d} \partial_i V w_i \partial_j (\rho w_j) \, \mathrm{d}x \\
& = \sum_{i=1}^d\int_{\mathbb{R}^d} \partial_i V \partial_t w_i \, \mathrm{d}\rho + \sum_{i,j=1}^d \int_{\mathbb{R}^d} \partial_j (\partial_i V w_i) w_j \, \mathrm{d}\rho
\\
& = \sum_{i=1}^d\int_{\mathbb{R}^d} \partial_i V \partial_t w_i \, \mathrm{d}\rho + \sum_{i,j=1}^d\int_{\mathbb{R}^d} (\partial_i \partial_j V) w_i w_j \, \mathrm{d}\rho  + \sum_{i,j=1}^d \int_{\mathbb{R}^d} \partial_i V (\partial_j w_i) w_j \, \mathrm{d} \rho.
\end{align}
\end{subequations}
Inserting \eqref{eq:w} and \eqref{eq:wt} gives the announced result.
\end{proof}

We are now ready to conclude:

\begin{proof}\textbf{of Lemma \ref{lem:Hessian}}\;\;
It is enough to show that 
\begin{equation}
\nonumber
q_{ij}[\rho] = q_{ij}^{\mathrm{Reg}}[\rho] + q_{ij}^{\mathrm{Cost}}[\rho], \qquad \qquad i,j = 1,\ldots,d.
\end{equation}
A straightforward calculation shows that \eqref{eq:Reg nonequilibrium1} and \eqref{eq:Cost nonequilibrium1} add up to \eqref{eq:nonequilibrium1}, \eqref{eq:Reg nonequilirbrium2} and \eqref{eq:Cost nonequilibrium2} add up to \eqref{eq:nonequilibrium2}, and \eqref{eq:Reg equilibrium} and \eqref{eq:Cost equilibrium} add up to \eqref{eq:equilibrium}.	
\end{proof}

\section{Proofs for Section \ref{sec:equilibrium}}
\label{app:equilibrium}
\begin{proof}\textbf{of Lemma \ref{lem:funct inequalities}}\;\;
By a straightforward calculation, the first statement is equivalent to the inequality
\begin{equation}
\label{eq:BE pi}
\int_{\mathbb{R}^d} \left[\sum_{j=1}^d  \partial_j \left( e^{-V} T_{k,\pi} \partial_j \Psi\right) \right]^2 \, e^{V} \mathrm{d}x \ge \lambda \sum_{j=1}^d \int_{\mathbb{R}^d} \int_{\mathbb{R}^d} \partial_j \Psi(y) k(y,z) \partial_j \Psi(z) e^{-V(y)} e^{-V(z)} \, \mathrm{d}y \mathrm{d}z,
\end{equation}
for all $\Psi \in C_c^\infty(\mathbb{R}^d)$.		
To show the equivalence between \eqref{eq:BE pi} and the second statement, first notice that	\eqref{eq:BE pi} can be written in the form
	\begin{equation}
	\label{eq:L squared}
	\int_{\mathbb{R}^d} (\mathcal{L}\Psi)^2 \, \mathrm{d}\pi \ge \lambda \int_{\mathbb{R}^d} \Psi \mathcal{L}\Psi \, \mathrm{d}\pi, \quad  \Psi \in C_c^\infty(\mathbb{R}^d).
	\end{equation}	
	Next we argue that under Assumption \ref{ass:ispd}, the null space of $\mathcal{L}$ coincides with the constant functions. Indeed assume that $\phi \in C_b^\infty(\mathbb{R}^d) \cap \mathcal{D}(\mathcal{L})$ satisfies $\mathcal{L} \phi = 0$. Multiplying this equation by $\phi e^{-V}$ and integrating by parts leads to 
 \begin{equation}
 \nonumber
\sum_{i=1}^d \int_{\mathbb{R}^d} \int_{\mathbb{R}^d} \partial_i \phi(x) k(x,y) \partial_i \phi(y) e^{-V(x)} e^{-V(y)} \, \mathrm{d}x \mathrm{d}y = 0.
\end{equation} 
Since $k$ is positive definite, it follows that the summands in the above equation are each nonnegative and thus have to be zero individually. According to Assumption \ref{ass:ispd}, it follows that the measure $\partial_i \phi e^{-V} \, \mathrm{d}x$ vanishes for every $i \in \{1, \ldots d\}$, which is only possible if $\phi$ is constant. By a very similar argument (using again Assumption \ref{ass:ispd}) we see that the range of $\mathcal{L}$ is dense in $L_0^2(\pi)$.
	
	A straightforward application of the spectral theorem for (possibly unbounded) self-adjoint operators to \eqref{eq:L squared} shows that $\sigma(\mathcal{L}) \subset \{0\} \cup [\lambda, \infty)$. Note moreover that 
\begin{equation}
 \nonumber
\int_{\mathbb{R}^d} \mathcal{L}\phi \, \mathrm{d}\pi = 0 
\end{equation}
for all $\phi \in C_c^\infty(\mathbb{R}^d)$, and that $L_0^2(\pi)$ is the orthogonal complement of the constant functions in $L^2(\pi)$. Hence, $\mathcal{L}$ leaves $L_0^2(\pi)$ invariant, and the restriction satisfies $\sigma(\mathcal{L}\vert_{L_0^2(\pi)}) \subset [\lambda, \infty)$. Since $\mathcal{L}\vert_{L_0^2(\pi)}$ is therefore bounded from below and, as noted above, with dense range, it is invertible, and, in particular $\mathcal{L}^{-1/2}: L_0^2(\pi) \rightarrow L_0^2(\pi)$ is well-defined. The equivalence between \eqref{eq:BE pi} and the second statement now follows by letting $\Psi = \mathcal{L}^{-1/2} \phi$.  
\end{proof}

\begin{proof}\textbf{of Lemma \ref{lem:k smooth}}\;\;
For $\phi \in C_c^\infty(\mathbb{R}^d)$ we can write 
\begin{equation}
 \nonumber
(\mathcal{L}\phi) (x) =  \frac{1}{Z}\sum_{i=1}^d  \int_{\mathbb{R}^d} e^{V(x)} e^{V(y)} \partial_{x_i} \partial_{y_i} \left( e^{-V(x)} e^{-V(y)} k(x,y) \right) \phi(y) e^{-V(y)} \, \mathrm{d}y,
\end{equation}
using the regularity assumption on $k$. Defining the positive definite kernel
\begin{equation}
\nonumber
\tilde{k}(x,y) := \sum_{i=1}^d e^{V(x)} e^{V(y)} \partial_{x_i} \partial_{y_i} \left(  e^{-V(x)}e^{-V(y)} k(x,y)  \right),
\end{equation}
we see that $
\mathcal{L} = \mathcal{T}_{\tilde{k},\pi}$. 
A short calculation shows that the integrability condition \eqref{eq:integrability} is equivalent to 
	\begin{equation}
	\int_{\mathbb{R}^d}\tilde{k}(x,x) \, \mathrm{d}\pi(x) < \infty,
	\end{equation}
	and thus $\mathcal{L}$ is compact according to \citet[Theorem 4.27]{steinwart2008support}. By the spectral theorem for compact self-adjoint operators \citep[Section 8.3]{kreyszig1978introductory}, there exists an orthonormal basis 
	$(e_i)_{i \in \mathbb{N}}$ of $L^2(\pi)$ such that 
\begin{equation}
\label{eq:eigenvalues}
\mathcal{L} e_i = \mu_i e_i,
\end{equation}
 $\mu_i \ge 0$ and $\mu_i \rightarrow 0$.  Plugging \eqref{eq:eigenvalues} into \eqref{eq:Poincare} and using $\mu_i \rightarrow 0$ shows that necessarily $\lambda = 0$. 
\end{proof}
\begin{proof}\textbf{of Lemma \ref{lem:V k separate}}\;\;
	For $\Psi \in C_c^\infty(\mathbb{R})$, set $\phi = \mathcal{T}_{k,\pi} \Psi'$. Using \eqref{eq:L2 H}, we see that the right-hand side of \eqref{eq:local BE} coincides with $\lambda \langle \phi, \phi \rangle_{\mathcal{H}_k}$. For the left-hand side we calculate 
	\begin{subequations}
		\label{eq:V'' calculation}
		\begin{align}
		& \int_{\mathbb{R}} \left[(e^{-V} \phi)'\right]^2 e^V \, \mathrm{d}x  = \int_{\mathbb{R}} \left[ -V' \phi + \phi' \right]^2 e^{-V} \, \mathrm{d}x
		\\
		 & = \int_{\mathbb{R}} \left[ (V')^2 \phi^2 - 2 V' \phi \phi' + (\phi')^2\right] e^{-V} \, \mathrm{d}x 
	 = \int_{\mathbb{R}} \left[ (V'') \phi^2 + (\phi')^2 \right] e^{-V} \, \mathrm{d}x,
 		\end{align}
	\end{subequations}
where we have used that 
\begin{equation}
\nonumber
-2 \int_{\mathbb{R}} V' \phi \phi' e^{-V} \, \mathrm{d}x = - \int_{\mathbb{R}} V' (\phi^2)' e^{-V} \, \mathrm{d}x = \int_{\mathbb{R}} V'' \phi^2 e^{-V} \, \mathrm{d}x - \int_{\mathbb{R}} (V')^2 \phi^2 \, \mathrm{d}x.
\end{equation}
It is therefore clear that if \eqref{eq:BE 1d} holds for all $\phi \in \mathcal{H}_k$, then $\eqref{eq:local BE}$ holds for all $\Psi \in C_c^{\infty}(\mathbb{R})$. For the converse implication, note that boundedness of $V''$ implies that \eqref{eq:V'' calculation} is a continuous functional on $H^1(\pi)$. It thus remains to show that $
\left\{ \mathcal{T}_{k,\pi} \Psi': \,\, \Psi \in C_c^{\infty}(\mathbb{R}) 
\right\}$
is dense in $H^1(\pi)$. By Assumptions \ref{ass:ispd} and \ref{ass:V and pi}, $\mathcal{T}_{k,\pi} : L^2(\pi) \rightarrow \mathcal{H}_k$ is continuous with dense range, see \citet[Theorem 4.26ii) and Exercise 4.6]{steinwart2008support}. Since $\mathcal{H}_k$ is densely embedded in $H^1(\pi)$ by assumption, it suffices to argue that
\begin{equation}
\nonumber
\left\{ \Psi' : \quad \Psi \in C_c^\infty(\mathbb{R})\right\} = \left\{ \Psi \in C_c^\infty(\mathbb{R}) : \quad \int_{\mathbb{R}} \Psi \, \mathrm{d}x = 0 \right\}
\end{equation} 
is dense in $L^2(\pi)$. Indeed, for any $\phi \in L^2(\pi)$ and $\varepsilon > 0$ there exists $\Psi_1 \in C_c^\infty(\mathbb{R})$ such that $\Vert \phi - \Psi_1 \Vert_{L^2(\pi)} < \varepsilon/2$. Moreover, since $\pi$ is a probability measure, there exists $\Psi_2 \in C_c^\infty(\mathbb{R})$ such that $\int_{\mathbb{R}} (\Psi_1 + \Psi_2) \, \mathrm{d}x = 0$ and $\Vert \Psi_2 \Vert_{L^2(\pi)} < \varepsilon/2$. It now follows that $\Psi := \Psi_1 + \Psi_2$ satisfies $\Vert \phi - \Psi \Vert_{L^2(\pi)} < \varepsilon$, concluding the proof. 
\end{proof}
\begin{proof}\textbf{of Corollary \ref{cor:translation invariant}}\;\;
	We argue by contradiction. Assume that there exists $\lambda > 0$ such that \eqref{eq:BE 1d} holds for all $\phi \in \mathcal{H}_k$. For $x \in \mathbb{R}$, let us choose $\phi_x = k(x,\cdot) = h(x-\cdot) \in \mathcal{H}_k$. For the right-hand side of \eqref{eq:BE 1d} we then obtain
\begin{equation}
\label{eq:h0}
\lambda \langle \phi_x, \phi_x \rangle_{\mathcal{H}_k} = \lambda k(x,x) = \lambda h(0).
\end{equation}
Since $h$ and $h'$ are bounded, we have that 
\begin{equation}
\nonumber
\lim_{x \rightarrow \pm \infty} \left(\int_{\mathbb{R}}V''(y) h(x-y) \mathrm{d}\pi(y) + \int_\mathbb{R} (h'(x-y))^2 \, \mathrm{d}\pi(y)\right) = 0 
\end{equation}
by dominated convergence. This contradicts \eqref{eq:BE 1d} (or forces $\lambda = 0$), because \eqref{eq:h0} does not depend on $x \in \mathbb{R}$. 
\end{proof}
\begin{proof}\textbf{for Example \ref{ex:exp decay}}
Arguing as in the proof of Lemma \ref{lem:V k separate}, it is enough to show that
\begin{equation}
\label{eq:V k inequality}
\int_{\mathbb{R}} \left[ (V'') \phi^2 + (\phi')^2 \right] e^{-V} \, \mathrm{d}x \ge \lambda \langle \phi, \phi \rangle_{\mathcal{H}_k}
\end{equation}
for all $\phi \in \left\{ \mathcal{T}_{k,\pi} \Psi': \,\, \Psi \in C_c^\infty(\mathbb{R}) \right\}$.	
We show the stronger statement that \eqref{eq:V k inequality} holds for all $\phi \in \mathcal{H}_k$ (recall that $\Ran \mathcal{T}_{k,\pi} \subset \mathcal{H}_k$). Combining Theorem 1.7 and Corollary 2.5 from \citet{saitoh2016theory}, we see that
\begin{equation}
\nonumber
\mathcal{H}_k = \left\{ \pi^{-1/2}f: \quad f \in H^1(\mathbb{R}) \right\},
\end{equation} 
where $H^1(\mathbb{R})$ denotes the Sobolev space of order one, and, furthermore, 
\begin{equation}
\label{eq:Sobolev norm}
\langle \pi^{-1/2}f,\pi^{-1/2}f \rangle_{\mathcal{H}_k} = \Vert f \Vert_{H^1(\mathbb{R})}^2 = \int_{\mathbb{R}} \left[f^2 + (f')^2 \right] \mathrm{d}x.
\end{equation}   
For the left-hand side of \eqref{eq:V k inequality}, we calculate
\begin{subequations}
	\begin{align}
\int_{\mathbb{R}} \left[ (V'') (\pi^{-1/2} f)^2 + \left((\pi^{-1/2}f)'\right)^2 \right] e^{-V} \, \mathrm{d}x = \int_{\mathbb{R}} \left[ V'' f^2 + \left( \frac{V'}{2} f + f' \right)^2 \right] \mathrm{d}x \\
\label{eq:V'' V'}
= \int_{\mathbb{R}} \left[ V'' f^2 + \left( \frac{V'}{2} \right)^2 f^2 + V' f f'   + (f')^2  \right] \mathrm{d}x = \int_{\mathbb{R}} \left[ \left( \frac{V''}{2} + \left( \frac{V'}{2} \right)^2 \right) f^2 + (f')^2  \right] \mathrm{d}x,
\end{align}
\end{subequations}
using 
\begin{equation}
\label{eq:integration by parts}
\int_{\mathbb{R}} V' f f' \, \mathrm{d}x = \frac{1}{2} \int_{\mathbb{R}} V' (f^2)' \, \mathrm{d}x = - \frac{1}{2} \int_{\mathbb{R}} V'' f^2 \, \mathrm{d}x.
\end{equation}
In \eqref{eq:integration by parts} we have used the fact that by boundedness of $V''$, $f \in H^1(\mathbb{R})$ and L'H{\^o}pital's rule,
\begin{equation}
\nonumber
\lim_{x \rightarrow  \pm \infty} f^2 V' = \lim_{x \rightarrow  \pm \infty} 2ff' V'' = 0.
\end{equation}
From \eqref{eq:Sobolev norm} and \eqref{eq:V'' V'} it is clear that \eqref{eq:V k inequality} holds with $\lambda$ as given in \eqref{eq:Matern lambda}.
\end{proof}
\begin{proof}\textbf{of Lemma \ref{lem:dominance balls}}\;\;
	Following the proof of Lemma \ref{lem:V k separate}, it is straightforward to show that the Rayleigh coefficients are given by
\begin{equation}
\nonumber
\lambda^k_\Psi = \frac{\int_{\mathbb{R}} V'' \phi^2 \, \mathrm{d}\pi + \int_{\mathbb{R}} (\phi')^2 \, \mathrm{d}\pi}{\Vert \phi \Vert^2_{\mathcal{H}_k}}, 
\end{equation}
where $\phi = \mathcal{T}_{k,\pi} \Psi'$. The claim now follows by a density argument, similar to the one employed in the proof of Lemma \ref{lem:V k separate}.
\end{proof}
\begin{proof}\textbf{of Lemma \ref{lem:exp p kernel}}\;\;
	By a slight abuse of notation, we will denote $k_{p,\sigma}(x,y) = k_{p,\sigma}(r)$, with $r = \vert x- y\vert$, using the fact that $k_{p,\sigma}$ is radially symmetric. We compute the Fourier transform in spherical coordinates,
\begin{subequations}
\nonumber
\begin{align}
(\mathcal{F}k_{p,\sigma})(\xi) & = \int_{\mathbb{R}^d} \exp(-i x \cdot \xi) \exp \left( - \frac{\vert x \vert^p}{\sigma^p}\right) \mathrm{d}x \\
& = c_d \int_0^{2 \pi} \int_0^\infty \exp(-i r \vert \xi \vert \cos \theta) \exp \left( - \frac{r^p}{\sigma^p}\right) \mathrm{d}r \mathrm{d}\theta, 
 \end{align}
\end{subequations}
where $\theta$ is the angle between $\xi$ and $x$, and $c_d > 0$ is a dimension-dependent constant resulting from integration over the remaining angles. From \citet[Lemma 2.27]{koldobsky2005fourier} we have that 
\begin{equation}
\nonumber
\mathcal{A}_{p,\sigma} (\xi,\theta) :=  \int_0^\infty \exp(-i r \vert \xi \vert \cos \theta) \exp \left( - \frac{r^p}{\sigma^p}\right) \mathrm{d}r 
\end{equation}
is strictly positive for all $(\xi,\theta) \in \mathbb{R}^d \times [0,2 \pi]$. It therefore follows that $\mathcal{F}k_{p,\sigma}$ is strictly positive. Hence, by \citet[Theorem]{wendland2004scattered}, $k_{p,\sigma}$ is a positive definite kernel. The fact that it is also integrally strictly positive definite follows from \citet[Proposition 5]{sriperumbudur2011universality}.
From \citet[Lemma 2.28]{koldobsky2005fourier}, we have that there exist constants $C_1, C_2 > 0$ such that 
\begin{equation}
\nonumber
C_1 \vert \xi \vert^{-p-1} \le \mathcal{A}_{p,\sigma} (\xi,\theta) \le C_2 \vert \xi \vert^{-p-1}, \quad \vert \xi \vert > 1.
\end{equation}
It is then easy to see that $(\mathcal{F}k_{p,\sigma_p})/(\mathcal{F}k_{q,\sigma_q})$ is bounded if $p > q$ and unbounded if $q < p$, for all $\sigma_q, \sigma_p > 0$. The second claim of Lemma \ref{lem:exp p kernel} now follows from \citet[Proposition 3.1]{zhang2013inclusion}.  	According to the same result, in the case when $p > q$, we have 
\begin{equation}
\nonumber
\Vert \phi \Vert_{\mathcal{H}_{k_{q,\sigma_q}}} \le C \Vert \phi \Vert_{\mathcal{H}_{k_{p,\sigma_p}}}, \quad \phi \in \mathcal{H}_{k_{p,\sigma_p}},
\end{equation}
where 
\begin{equation}
\nonumber
C = \sqrt{\sup \frac{\mathcal{F}k_{p,\sigma_p}}{\mathcal{F}k_{q,\sigma_q}}}.
\end{equation}
Using 
\begin{equation}
\nonumber
(\mathcal{F}k_{p,L \sigma})(\xi) = \frac{1}{L^p}(\mathcal{F}k_{p, \sigma})(L^p \xi), \quad L > 0, 
\end{equation}
it is clear that $\sigma_p$ and $\sigma_q$ can be chosen in such a way that $C \le 1$, proving the third claim. 
\end{proof}

\bibliography{refs_final}

\begin{thebibliography}{99}
\providecommand{\natexlab}[1]{#1}
\providecommand{\url}[1]{\texttt{#1}}
\expandafter\ifx\csname urlstyle\endcsname\relax
  \providecommand{\doi}[1]{doi: #1}\else
  \providecommand{\doi}{doi: \begingroup \urlstyle{rm}\Url}\fi

\bibitem[Ambrogioni et~al.(2018)Ambrogioni, Guclu, Gucluturk, and van
  Gerven]{ambrogioni2018wasserstein}
L.~Ambrogioni, U.~Guclu, Y.~Gucluturk, and M.~van Gerven.
\newblock Wasserstein variational gradient descent: From semi-discrete optimal
  transport to ensemble variational inference.
\newblock \emph{arXiv:1811.02827}, 2018.

\bibitem[Ambrosio and Gigli(2013)]{ambrosio2013user}
L.~Ambrosio and N.~Gigli.
\newblock A user’s guide to optimal transport.
\newblock In \emph{Modelling and optimisation of flows on networks}, pages
  1--155. Springer, 2013.

\bibitem[Ambrosio et~al.(2008)Ambrosio, Gigli, and
  Savar{\'e}]{ambrosio2008gradient}
L.~Ambrosio, N.~Gigli, and G.~Savar{\'e}.
\newblock \emph{Gradient flows: in metric spaces and in the space of
  probability measures}.
\newblock Springer Science \& Business Media, 2008.

\bibitem[Arbel et~al.(2019)Arbel, Korba, Salim, and Gretton]{arbel2019maximum}
M.~Arbel, A.~Korba, A.~Salim, and A.~Gretton.
\newblock Maximum mean discrepancy gradient flow.
\newblock In \emph{Advances in Neural Information Processing Systems 32}, 2019.

\bibitem[Arnrich et~al.(2012)Arnrich, Mielke, Peletier, Savar{\'e}, and
  Veneroni]{arnrich2012passing}
S.~Arnrich, A.~Mielke, M.~A. Peletier, G.~Savar{\'e}, and M.~Veneroni.
\newblock Passing to the limit in a {W}asserstein gradient flow: from diffusion
  to reaction.
\newblock \emph{Calculus of Variations and Partial Differential Equations},
  44\penalty0 (3-4):\penalty0 419--454, 2012.

\bibitem[Bakry et~al.(2013)Bakry, Gentil, and Ledoux]{bakry2013analysis}
D.~Bakry, I.~Gentil, and M.~Ledoux.
\newblock \emph{Analysis and geometry of Markov diffusion operators}, volume
  348.
\newblock Springer Science \& Business Media, 2013.

\bibitem[Benamou and Brenier(2000)]{benamou2000computational}
J.-D. Benamou and Y.~Brenier.
\newblock A computational fluid mechanics solution to the {M}onge-{K}antorovich
  mass transfer problem.
\newblock \emph{Numerische Mathematik}, 84\penalty0 (3):\penalty0 375--393,
  2000.

\bibitem[Bigoni et~al.(2019)Bigoni, Zahm, Spantini, and
  Marzouk]{bigoni2019greedy}
D.~Bigoni, O.~Zahm, A.~Spantini, and Y.~Marzouk.
\newblock Greedy inference with layers of lazy maps.
\newblock \emph{arXiv:1906.00031}, 2019.

\bibitem[Bolley et~al.(2018)Bolley, Chafa{\"\i}, Fontbona,
  et~al.]{bolley2018dynamics}
F.~Bolley, D.~Chafa{\"\i}, J.~Fontbona, et~al.
\newblock Dynamics of a planar {C}oulomb gas.
\newblock \emph{The Annals of Applied Probability}, 28\penalty0 (5):\penalty0
  3152--3183, 2018.

\bibitem[Brasco(2012)]{brasco2012survey}
L.~Brasco.
\newblock A survey on dynamical transport distances.
\newblock \emph{Journal of Mathematical Sciences}, 181\penalty0 (6):\penalty0
  755--781, 2012.

\bibitem[Buttazzo et~al.(2009)Buttazzo, Jimenez, and
  Oudet]{buttazzo2009optimization}
G.~Buttazzo, C.~Jimenez, and E.~Oudet.
\newblock An optimization problem for mass transportation with congested
  dynamics.
\newblock \emph{SIAM Journal on Control and Optimization}, 48\penalty0
  (3):\penalty0 1961--1976, 2009.

\bibitem[Byrne and Hindmarsh(1975)]{byrne1975polyalgorithm}
G.~D. Byrne and A.~C. Hindmarsh.
\newblock A polyalgorithm for the numerical solution of ordinary differential
  equations.
\newblock \emph{ACM Transactions on Mathematical Software (TOMS)}, 1\penalty0
  (1):\penalty0 71--96, 1975.

\bibitem[Carmeli et~al.(2006)Carmeli, De~Vito, and Toigo]{carmeli2006vector}
C.~Carmeli, E.~De~Vito, and A.~Toigo.
\newblock Vector valued reproducing kernel {H}ilbert spaces of integrable
  functions and {M}ercer theorem.
\newblock \emph{Analysis and Applications}, 4\penalty0 (04):\penalty0 377--408,
  2006.

\bibitem[Carmona and Delarue(2018)]{carmona2018probabilistic}
R.~Carmona and F.~Delarue.
\newblock \emph{Probabilistic Theory of Mean Field Games with Applications
  I-II}.
\newblock Springer, 2018.

\bibitem[Carrillo et~al.(2010)Carrillo, Lisini, Savar{\'e}, and
  Slep{\v{c}}ev]{carrillo2010nonlinear}
J.~A. Carrillo, S.~Lisini, G.~Savar{\'e}, and D.~Slep{\v{c}}ev.
\newblock Nonlinear mobility continuity equations and generalized displacement
  convexity.
\newblock \emph{Journal of Functional Analysis}, 258\penalty0 (4):\penalty0
  1273--1309, 2010.

\bibitem[Chen and Zhang(2017)]{chen2017particle}
C.~Chen and R.~Zhang.
\newblock Particle optimization in {MCMC}.
\newblock \emph{arXiv:1711.10927}, 2017.

\bibitem[Chen et~al.(2018{\natexlab{a}})Chen, Zhang, Wang, Li, and
  Chen]{chen2018unified}
C.~Chen, R.~Zhang, W.~Wang, B.~Li, and L.~Chen.
\newblock A unified particle-optimization framework for scalable {B}ayesian
  sampling.
\newblock In \emph{Proceedings of the Thirty-Fourth Conference on Uncertainty
  in Artificial Intelligence, {UAI}}. {AUAI} Press, 2018{\natexlab{a}}.

\bibitem[Chen et~al.(2019)Chen, Wu, Chen, O'Leary-Roseberry, and
  Ghattas]{chen2019projected}
P.~Chen, K.~Wu, J.~Chen, T.~O'Leary-Roseberry, and O.~Ghattas.
\newblock Projected {S}tein variational {N}ewton: A fast and scalable
  {B}ayesian inference method in high dimensions.
\newblock In \emph{Advances in Neural Information Processing Systems 32}, 2019.

\bibitem[Chen et~al.(2018{\natexlab{b}})Chen, Mackey, Gorham, Briol, and
  Oates]{chen2018stein}
W.~Y. Chen, L.~Mackey, J.~Gorham, F.-X. Briol, and C.~Oates.
\newblock Stein points.
\newblock In \emph{International Conference on Machine Learning}, pages
  844--853. PMLR, 2018{\natexlab{b}}.

\bibitem[Daneri and Savar{\'e}(2008)]{daneri2008eulerian}
S.~Daneri and G.~Savar{\'e}.
\newblock Eulerian calculus for the displacement convexity in the {W}asserstein
  distance.
\newblock \emph{SIAM Journal on Mathematical Analysis}, 40\penalty0
  (3):\penalty0 1104--1122, 2008.

\bibitem[Detommaso et~al.(2018)Detommaso, Cui, Marzouk, Spantini, and
  Scheichl]{detommaso2018stein}
G.~Detommaso, T.~Cui, Y.~Marzouk, A.~Spantini, and R.~Scheichl.
\newblock A {S}tein variational {N}ewton method.
\newblock In \emph{Advances in Neural Information Processing Systems 31}, 2018.

\bibitem[Detommaso et~al.(2019)Detommaso, Hoitzing, Cui, and
  Alamir]{detommaso2019stein}
G.~Detommaso, H.~Hoitzing, T.~Cui, and A.~Alamir.
\newblock {S}tein variational online changepoint detection with applications to
  {H}awkes processes and neural networks.
\newblock \emph{arXiv:1901.07987}, 2019.

\bibitem[Dolbeault et~al.(2009)Dolbeault, Nazaret, and
  Savar{\'e}]{dolbeault2009new}
J.~Dolbeault, B.~Nazaret, and G.~Savar{\'e}.
\newblock A new class of transport distances between measures.
\newblock \emph{Calculus of Variations and Partial Differential Equations},
  34\penalty0 (2):\penalty0 193--231, 2009.

\bibitem[Duncan et~al.(2017)Duncan, N\"usken, and Pavliotis]{DNP2017}
A.~B. Duncan, N.~N\"usken, and G.~A. Pavliotis.
\newblock Using perturbed underdamped {L}angevin dynamics to efficiently sample
  from probability distributions.
\newblock \emph{J. Stat. Phys.}, 169\penalty0 (6):\penalty0 1098--1131, 2017.
\newblock ISSN 0022-4715.
\newblock URL \url{https://doi.org/10.1007/s10955-017-1906-8}.

\bibitem[Fasshauer(2007)]{fasshauer2007meshfree}
G.~E. Fasshauer.
\newblock \emph{Meshfree approximation methods with MATLAB}, volume~6.
\newblock World Scientific, 2007.

\bibitem[Figalli and Glaudo(2021)]{figalli2021invitation}
Alessio Figalli and Federico Glaudo.
\newblock \emph{An Invitation to Optimal Transport, Wasserstein Distances, and
  Gradient Flows}.
\newblock EMS Press, 2021.

\bibitem[Flamary and Courty(2017)]{flamary2017pot}
R.~Flamary and N.~Courty.
\newblock {POT} python optimal transport library, 2017.

\bibitem[Fukumizu et~al.(2009)Fukumizu, Gretton, Lanckriet, Sch{\"o}lkopf, and
  Sriperumbudur]{fukumizu2009kernel}
K.~Fukumizu, A.~Gretton, G.~R. Lanckriet, B.~Sch{\"o}lkopf, and B.~K.
  Sriperumbudur.
\newblock Kernel choice and classifiability for {RKHS} embeddings of
  probability distributions.
\newblock In \emph{Advances in neural information processing systems}, pages
  1750--1758, 2009.

\bibitem[Gallego and Insua(2018)]{gallego2018stochastic}
V.~Gallego and D.~R. Insua.
\newblock Stochastic gradient {MCMC} with repulsive forces.
\newblock \emph{arXiv:1812.00071}, 2018.

\bibitem[Gao et~al.(2019)Gao, Jiao, Wang, Wang, Yang, and Zhang]{yuan2019deep}
Y.~Gao, Y.~Jiao, Y.~Wang, Y.~Wang, C.~Yang, and S.~Zhang.
\newblock Deep generative learning via variational gradient flow.
\newblock In \emph{International Conference on Machine Learning}, pages
  2093--2101. PMLR, 2019.

\bibitem[Garbuno-Inigo et~al.(2020{\natexlab{a}})Garbuno-Inigo, N{\"u}sken, and
  Reich]{garbuno2019affine}
A.~Garbuno-Inigo, N.~N{\"u}sken, and S.~Reich.
\newblock Affine invariant interacting {L}angevin dynamics for {B}ayesian
  inference.
\newblock \emph{SIAM Journal on Applied Dynamical Systems}, 19\penalty0
  (3):\penalty0 1633--1658, 2020{\natexlab{a}}.

\bibitem[Garbuno-Inigo et~al.(2020{\natexlab{b}})Garbuno-Inigo, Hoffmann, Li,
  and Stuart]{garbuno2019gradient}
Alfredo Garbuno-Inigo, Franca Hoffmann, Wuchen Li, and Andrew~M Stuart.
\newblock Interacting langevin diffusions: Gradient structure and ensemble
  kalman sampler.
\newblock \emph{SIAM Journal on Applied Dynamical Systems}, 19\penalty0
  (1):\penalty0 412--441, 2020{\natexlab{b}}.

\bibitem[Gigli(2012)]{gigli2012second}
N.~Gigli.
\newblock \emph{Second Order Analysis on $(\mathcal{P}_2(M),W_2)$}.
\newblock American Mathematical Soc., 2012.

\bibitem[Iglesias et~al.(2013)Iglesias, Law, and Stuart]{iglesias2013ensemble}
M.~A. Iglesias, K.~J.~H. Law, and A.~M. Stuart.
\newblock Ensemble {K}alman methods for inverse problems.
\newblock \emph{Inverse Problems}, 29\penalty0 (4):\penalty0 045001, 2013.

\bibitem[Jordan et~al.(1998)Jordan, Kinderlehrer, and
  Otto]{jordan1998variational}
R.~Jordan, D.~Kinderlehrer, and F.~Otto.
\newblock The variational formulation of the {F}okker--{P}lanck equation.
\newblock \emph{SIAM journal on mathematical analysis}, 29\penalty0
  (1):\penalty0 1--17, 1998.

\bibitem[J{\"u}ngel(2016)]{jungel2016entropy}
Ansgar J{\"u}ngel.
\newblock \emph{Entropy methods for diffusive partial differential equations},
  volume 804.
\newblock Springer, 2016.

\bibitem[Khasminskii(2011)]{khasminskii2011stochastic}
R~Khasminskii.
\newblock \emph{Stochastic stability of differential equations}, volume~66.
\newblock Springer Science \& Business Media, 2011.

\bibitem[Kliemann(1987)]{kliemann1987recurrence}
W.~Kliemann.
\newblock Recurrence and invariant measures for degenerate diffusions.
\newblock \emph{The {A}nnals of {P}robability}, pages 690--707, 1987.

\bibitem[Koldobsky(2005)]{koldobsky2005fourier}
A.~Koldobsky.
\newblock \emph{Fourier analysis in convex geometry}.
\newblock Number 116. American Mathematical Soc., 2005.

\bibitem[Kreyszig(1978)]{kreyszig1978introductory}
E.~Kreyszig.
\newblock \emph{Introductory functional analysis with applications}, volume~1.
\newblock Wiley New York, 1978.

\bibitem[Lee(2006)]{lee2006riemannian}
J.~M. Lee.
\newblock \emph{Riemannian manifolds: an introduction to curvature}, volume
  176.
\newblock Springer Science \& Business Media, 2006.

\bibitem[Li et~al.(2020)Li, Li, Liu, Liu, and Lu]{li2019stochastic}
L.~Li, Y.~Li, J.-G. Liu, Z.~Liu, and J.~Lu.
\newblock A stochastic version of {S}tein variational gradient descent for
  efficient sampling.
\newblock \emph{Communications in Applied Mathematics and Computational
  Science}, 15\penalty0 (1):\penalty0 37--63, 2020.

\bibitem[Li(2019)]{li2019diffusion}
W.~Li.
\newblock Diffusion hypercontractivity via generalized density manifold.
\newblock \emph{arXiv preprint arXiv:1907.12546}, 2019.

\bibitem[Li(2021)]{li2021hessian}
W.~Li.
\newblock Hessian metric via transport information geometry.
\newblock \emph{Journal of Mathematical Physics}, 62\penalty0 (3):\penalty0
  033301, 2021.

\bibitem[Li and Mont{\'u}far(2018)]{li2018natural}
W.~Li and G.~Mont{\'u}far.
\newblock Natural gradient via optimal transport.
\newblock \emph{Information Geometry}, 1\penalty0 (2):\penalty0 181--214, 2018.

\bibitem[Li and Mont{\'u}far(2020)]{li2020ricci}
W.~Li and G.~Mont{\'u}far.
\newblock Ricci curvature for parametric statistics via optimal transport.
\newblock \emph{Information Geometry}, 3\penalty0 (1):\penalty0 89--117, 2020.

\bibitem[Liero and Mielke(2013)]{liero2013gradient}
M.~Liero and A.~Mielke.
\newblock Gradient structures and geodesic convexity for reaction--diffusion
  systems.
\newblock \emph{Philosophical Transactions of the Royal Society A:
  Mathematical, Physical and Engineering Sciences}, 371\penalty0
  (2005):\penalty0 20120346, 2013.

\bibitem[Liu and Zhu(2018)]{liu2018riemannian}
C.~Liu and J.~Zhu.
\newblock Riemannian {S}tein variational gradient descent for {B}ayesian
  inference.
\newblock In \emph{Thirty-Second AAAI Conference on Artificial Intelligence},
  2018.

\bibitem[Liu et~al.(2019)Liu, Zhuo, Cheng, Zhang, and
  Zhu]{liu2019understanding}
C.~Liu, J.~Zhuo, P.~Cheng, R.~Zhang, and J.~Zhu.
\newblock Understanding and accelerating particle-based variational inference.
\newblock In \emph{International Conference on Machine Learning}, pages
  4082--4092, 2019.

\bibitem[Liu(2017)]{liu2017stein}
Q.~Liu.
\newblock Stein variational gradient descent as gradient flow.
\newblock In \emph{Advances in neural information processing systems}, pages
  3115--3123, 2017.

\bibitem[Liu and Wang(2016)]{liu2016stein}
Q.~Liu and D.~Wang.
\newblock {S}tein variational gradient descent: a general purpose {B}ayesian
  inference algorithm.
\newblock In \emph{Advances In Neural Information Processing Systems}, pages
  2378--2386, 2016.

\bibitem[Liu and Wang(2018)]{liu2018stein}
Q.~Liu and D.~Wang.
\newblock {S}tein variational gradient descent as moment matching.
\newblock In \emph{Advances in Neural Information Processing Systems}, pages
  8868--8877, 2018.

\bibitem[Liutkus et~al.(2019)Liutkus, Simsekli, Majewski, Durmus, and
  St{\"o}ter]{csimcsekli2018sliced}
A.~Liutkus, U.~Simsekli, S.~Majewski, A.~Durmus, and F.-R. St{\"o}ter.
\newblock Sliced-{W}asserstein flows: Nonparametric generative modeling via
  optimal transport and diffusions.
\newblock In \emph{International Conference on Machine Learning}, pages
  4104--4113. PMLR, 2019.

\bibitem[Lott(2008)]{lott2008some}
J.~Lott.
\newblock Some geometric calculations on {W}asserstein space.
\newblock \emph{Communications in Mathematical Physics}, 277\penalty0
  (2):\penalty0 423--437, 2008.

\bibitem[Lu et~al.(2019{\natexlab{a}})Lu, Lu, and Nolen]{lu2019scaling}
J.~Lu, Y.~Lu, and J.~Nolen.
\newblock Scaling limit of the {S}tein variational gradient descent: the mean
  field regime.
\newblock \emph{SIAM Journal on Mathematical Analysis}, 51\penalty0
  (2):\penalty0 648--671, 2019{\natexlab{a}}.

\bibitem[Lu et~al.(2019{\natexlab{b}})Lu, Lu, and Nolen]{lu2019accelerating}
Y.~Lu, J.~Lu, and J.~Nolen.
\newblock Accelerating {L}angevin sampling with birth-death.
\newblock \emph{arXiv:1905.09863}, 2019{\natexlab{b}}.

\bibitem[Ma et~al.(2015)Ma, Chen, and Fox]{ma2015complete}
Y.-A. Ma, T.~Chen, and E.~Fox.
\newblock A complete recipe for stochastic gradient {MCMC}.
\newblock In \emph{Advances in Neural Information Processing Systems}, pages
  2899--2907, 2015.

\bibitem[Machlup and Onsager(1953)]{machlup1953fluctuations}
S.~Machlup and L.~Onsager.
\newblock Fluctuations and irreversible process. ii. systems with kinetic
  energy.
\newblock \emph{Physical Review}, 91\penalty0 (6):\penalty0 1512, 1953.

\bibitem[McCann(1997)]{mccann1997convexity}
R.~J. McCann.
\newblock A convexity principle for interacting gases.
\newblock \emph{Advances in mathematics}, 128\penalty0 (1):\penalty0 153--179,
  1997.

\bibitem[Meyn and Tweedie(1993)]{meyn1993stability}
S.~P. Meyn and R.~L. Tweedie.
\newblock Stability of {M}arkovian processes iii: {F}oster--{L}yapunov criteria
  for continuous-time processes.
\newblock \emph{Advances in Applied Probability}, 25\penalty0 (3):\penalty0
  518--548, 1993.

\bibitem[Micchelli and Pontil(2005)]{micchelli2005learning}
C.~A. Micchelli and M.~Pontil.
\newblock On learning vector-valued functions.
\newblock \emph{Neural computation}, 17\penalty0 (1):\penalty0 177--204, 2005.

\bibitem[Mielke(2011)]{mielke2011gradient}
A.~Mielke.
\newblock A gradient structure for reaction--diffusion systems and for
  energy-drift-diffusion systems.
\newblock \emph{Nonlinearity}, 24\penalty0 (4):\penalty0 1329, 2011.

\bibitem[Mielke(2013)]{mielke2013thermomechanical}
A.~Mielke.
\newblock Thermomechanical modeling of energy-reaction-diffusion systems,
  including bulk-interface interactions.
\newblock \emph{Discr. Cont. Dynam. Systems Ser. S}, 6\penalty0 (2):\penalty0
  479--499, 2013.

\bibitem[Mielke et~al.(2014)Mielke, Peletier, and Renger]{mielke2014relation}
A.~Mielke, M.~A. Peletier, and D.~R.~M. Renger.
\newblock On the relation between gradient flows and the large-deviation
  principle, with applications to {M}arkov chains and diffusion.
\newblock \emph{Potential Analysis}, 41\penalty0 (4):\penalty0 1293--1327,
  2014.

\bibitem[Mielke et~al.(2016)Mielke, Renger, and
  Peletier]{mielke2016generalization}
A.~Mielke, D.~R.~M. Renger, and M.~A. Peletier.
\newblock A generalization of {O}nsager’s reciprocity relations to gradient
  flows with nonlinear mobility.
\newblock \emph{Journal of Non-Equilibrium Thermodynamics}, 41\penalty0
  (2):\penalty0 141--149, 2016.

\bibitem[Mroueh et~al.(2018)Mroueh, Li, Sercu, Raj, and
  Cheng]{mroueh2017sobolev}
Y.~Mroueh, C.-L. Li, T.~Sercu, A.~Raj, and Y.~Cheng.
\newblock Sobolev {GAN}.
\newblock In \emph{6th International Conference on Learning Representations,
  {ICLR}}, 2018.

\bibitem[Mroueh et~al.(2019)Mroueh, Sercu, and Raj]{mroueh2019sobolev}
Y.~Mroueh, T.~Sercu, and A.~Raj.
\newblock Sobolev descent.
\newblock In \emph{The 22nd International Conference on Artificial Intelligence
  and Statistics}, pages 2976--2985, 2019.

\bibitem[{N}{\"u}sken and {P}avliotis(2019)]{nusken2019constructing}
N.~{N}{\"u}sken and G.~A. {P}avliotis.
\newblock Constructing sampling schemes via coupling: {M}arkov semigroups and
  optimal transport.
\newblock \emph{SIAM/ASA Journal on Uncertainty Quantification}, 7\penalty0
  (1):\penalty0 324--382, 2019.

\bibitem[N{\"u}sken and Reich(2019)]{nusken2019note}
N.~N{\"u}sken and S~Reich.
\newblock Note on interacting {L}angevin diffusions: Gradient structure and
  ensemble {K}alman sampler by {G}arbuno-{I}nigo, {H}offmann, {L}i and
  {S}tuart.
\newblock \emph{arXiv:1908.10890}, 2019.

\bibitem[N{\"u}sken and Renger(2021)]{nusken2021stein}
N.~N{\"u}sken and D.~R. Renger.
\newblock Stein variational gradient descent: many-particle and long-time
  asymptotics.
\newblock \emph{arXiv preprint arXiv:2102.12956}, 2021.

\bibitem[{\"O}ttinger(2005)]{ottinger2005beyond}
H.~C. {\"O}ttinger.
\newblock \emph{Beyond equilibrium thermodynamics}.
\newblock John Wiley \& Sons, 2005.

\bibitem[Otto(1998)]{otto1998dynamics}
F.~Otto.
\newblock Dynamics of labyrinthine pattern formation in magnetic fluids: A
  mean-field theory.
\newblock \emph{Archive for Rational Mechanics and Analysis}, 141\penalty0
  (1):\penalty0 63--103, 1998.

\bibitem[Otto(2001)]{otto2001geometry}
F.~Otto.
\newblock The geometry of dissipative evolution equations: the porous medium
  equation.
\newblock \emph{Communications in Partial Differential Equations}, 26\penalty0
  (1-2):\penalty0 101--174, 2001.

\bibitem[Otto and Villani(2000)]{otto2000generalization}
F.~Otto and C.~Villani.
\newblock Generalization of an inequality by {T}alagrand and links with the
  logarithmic {S}obolev inequality.
\newblock \emph{Journal of Functional Analysis}, 173\penalty0 (2):\penalty0
  361--400, 2000.

\bibitem[Otto and Westdickenberg(2005)]{otto2005eulerian}
F.~Otto and M.~Westdickenberg.
\newblock {E}ulerian calculus for the contraction in the {W}asserstein
  distance.
\newblock \emph{SIAM journal on mathematical analysis}, 37\penalty0
  (4):\penalty0 1227--1255, 2005.

\bibitem[Pathiraja and Reich(2019)]{pathiraja2019discrete}
S.~Pathiraja and S.~Reich.
\newblock Discrete gradients for computational {B}ayesian inference.
\newblock \emph{Journal of Computational Dynamics}, 6\penalty0 (2):\penalty0
  385--400, 2019.

\bibitem[Pavliotis(2014)]{pavliotis2014stochastic}
G.~A. Pavliotis.
\newblock \emph{{Stochastic processes and applications: Diffusion Processes,
  the Fokker-Planck and Langevin Equations}}, volume~60.
\newblock Springer, 2014.

\bibitem[Peletier(2014)]{peletier2014variational}
M.~A. Peletier.
\newblock Variational modelling: Energies, gradient flows, and large
  deviations.
\newblock \emph{arXiv:1402.1990}, 2014.

\bibitem[Pulido and van Leeuwen(2018)]{pulido2018kernel}
M.~Pulido and P.~J. van Leeuwen.
\newblock Kernel embedding of maps for sequential {B}ayesian inference: the
  variational mapping particle filter.
\newblock \emph{arXiv:1805.11380}, 2018.

\bibitem[Reed et~al.(1972)Reed, Simon, Simon, and Simon]{reed1972methods}
Michael Reed, Barry Simon, Barry Simon, and Barry Simon.
\newblock \emph{Methods of modern mathematical physics}, volume~1.
\newblock Elsevier, 1972.

\bibitem[Reich and Weissmann(2021)]{reich2019fokker}
S.~Reich and S.~Weissmann.
\newblock {F}okker--{P}lanck particle systems for {B}ayesian inference:
  Computational approaches.
\newblock \emph{SIAM/ASA Journal on Uncertainty Quantification}, 9\penalty0
  (2):\penalty0 446--482, 2021.

\bibitem[Robert and Casella(2013)]{robert2013monte}
C.~Robert and G.~Casella.
\newblock \emph{Monte Carlo statistical methods}.
\newblock Springer Science \& Business Media, 2013.

\bibitem[Roberts et~al.(1996)Roberts, Tweedie, et~al.]{roberts1996exponential}
G.~O. Roberts, R.~L. Tweedie, et~al.
\newblock Exponential convergence of {L}angevin distributions and their
  discrete approximations.
\newblock \emph{Bernoulli}, 2\penalty0 (4):\penalty0 341--363, 1996.

\bibitem[Rockafellar(1970)]{rockafellar1970convex}
R.~T. Rockafellar.
\newblock \emph{Convex analysis}, volume~28.
\newblock Princeton university press, 1970.

\bibitem[Saitoh and Sawano(2016)]{saitoh2016theory}
S.~Saitoh and Y.~Sawano.
\newblock \emph{Theory of reproducing kernels and applications}.
\newblock Springer, 2016.

\bibitem[Sriperumbudur et~al.(2010)Sriperumbudur, Gretton, Fukumizu,
  Sch{\"o}lkopf, and Lanckriet]{sriperumbudur2010hilbert}
B.~K. Sriperumbudur, A.~Gretton, K.~Fukumizu, B.~Sch{\"o}lkopf, and G.~R.~G.
  Lanckriet.
\newblock Hilbert space embeddings and metrics on probability measures.
\newblock \emph{Journal of Machine Learning Research}, 11\penalty0
  (Apr):\penalty0 1517--1561, 2010.

\bibitem[Sriperumbudur et~al.(2011)Sriperumbudur, Fukumizu, and
  Lanckriet]{sriperumbudur2011universality}
B.~K. Sriperumbudur, K.~Fukumizu, and G.~R.~G. Lanckriet.
\newblock Universality, characteristic kernels and {RKHS} embedding of
  measures.
\newblock \emph{Journal of Machine Learning Research}, 12\penalty0
  (Jul):\penalty0 2389--2410, 2011.

\bibitem[Steinwart and Christmann(2008)]{steinwart2008support}
I.~Steinwart and A.~Christmann.
\newblock \emph{Support vector machines}.
\newblock Springer Science \& Business Media, 2008.

\bibitem[Villani(2003{\natexlab{a}})]{V2003}
C.~Villani.
\newblock \emph{Topics in optimal transportation}, volume~58 of \emph{Graduate
  Studies in Mathematics}.
\newblock American Mathematical Society, Providence, RI, 2003{\natexlab{a}}.
\newblock ISBN 0-8218-3312-X.
\newblock URL \url{https://doi.org/10.1007/b12016}.

\bibitem[Villani(2003{\natexlab{b}})]{villani2003optimal}
C.~Villani.
\newblock Optimal transportation, dissipative {PDE’s} and functional
  inequalities.
\newblock In \emph{Optimal transportation and applications}, pages 53--89.
  Springer, 2003{\natexlab{b}}.

\bibitem[Villani(2009)]{V2009}
C.~Villani.
\newblock \emph{Optimal transport}, volume 338 of \emph{Grundlehren der
  Mathematischen Wissenschaften [Fundamental Principles of Mathematical
  Sciences]}.
\newblock Springer-Verlag, Berlin, 2009.
\newblock ISBN 978-3-540-71049-3.
\newblock URL \url{https://doi.org/10.1007/978-3-540-71050-9}.
\newblock Old and new.

\bibitem[Wang et~al.(2019{\natexlab{a}})Wang, Tang, Bajaj, and
  Liu]{wang2019stein}
D.~Wang, Z.~Tang, C.~Bajaj, and Q.~Liu.
\newblock Stein variational gradient descent with matrix-valued kernels.
\newblock In \emph{Advances in Neural Information Processing Systems}, pages
  7834--7844, 2019{\natexlab{a}}.

\bibitem[Wang and Li(2020)]{wang2020information}
Y.~Wang and W.~Li.
\newblock Information {N}ewton's flow: second-order optimization method in
  probability space.
\newblock \emph{arXiv preprint arXiv:2001.04341}, 2020.

\bibitem[Wang et~al.(2019{\natexlab{b}})Wang, Ren, Zhu, and
  Zhang]{wang2019function}
Z.~Wang, T.~Ren, J.~Zhu, and B.~Zhang.
\newblock Function space particle optimization for {B}ayesian neural networks.
\newblock In \emph{7th International Conference on Learning Representations,
  {ICLR}}, 2019{\natexlab{b}}.

\bibitem[Wendland(2004)]{wendland2004scattered}
H.~Wendland.
\newblock \emph{Scattered data approximation}, volume~17.
\newblock Cambridge university press, 2004.

\bibitem[Zhang and Zhao(2013)]{zhang2013inclusion}
H.~Zhang and L.~Zhao.
\newblock On the inclusion relation of reproducing kernel {H}ilbert spaces.
\newblock \emph{Analysis and Applications}, 11\penalty0 (02):\penalty0 1350014,
  2013.

\bibitem[Zhang et~al.(2018)Zhang, Zhang, and Chen]{zhang2018towards}
J.~Zhang, R.~Zhang, and C.~Chen.
\newblock Towards more theoretically-grounded sampling for deep learning.
\newblock \emph{openreview.net}, 2018.

\bibitem[Zhang et~al.(2020)Zhang, Zhang, Carin, and Chen]{zhang2018stochastic}
J.~Zhang, R.~Zhang, L.~Carin, and C.~Chen.
\newblock Stochastic particle-optimization sampling and the non-asymptotic
  convergence theory.
\newblock In \emph{International Conference on Artificial Intelligence and
  Statistics}, pages 1877--1887. PMLR, 2020.

\bibitem[Zhuo et~al.(2018)Zhuo, Liu, Shi, Zhu, Chen, and
  Zhang]{zhuo2017message}
J.~Zhuo, C.~Liu, J.~Shi, J.~Zhu, N.~Chen, and B.~Zhang.
\newblock Message passing {S}tein variational gradient descent.
\newblock In \emph{International Conference on Machine Learning}, pages
  6018--6027. PMLR, 2018.

\end{thebibliography}

\end{document}